\documentclass[journal]{IEEEtran}

\usepackage{cite}
\usepackage{amsmath,amssymb,amsfonts,amsthm}
\usepackage{graphicx}
\usepackage{booktabs}
\usepackage{duckuments}
\usepackage[ruled,vlined,linesnumbered,noend]{algorithm2e}
\usepackage{xcolor}
\usepackage{url}
\usepackage[font=footnotesize,labelsep=period]{caption}
\usepackage{subcaption}
\usepackage{comment} 
\usepackage{tikz}
\usepackage[dvipsnames,x11names,table]{xcolor}

\usepackage[hidelinks]{hyperref}

\newcommand{\Dep}{\mathrm{Dep}}
\newtheorem{theorem}{Theorem}
\newtheorem{lemma}{Lemma}
\newtheorem{proposition}{Proposition}
\newtheorem{corollary}{Corollary}
\DontPrintSemicolon
\SetKwInOut{Input}{Input}
\SetKwInOut{Output}{Output}
\SetKw{Return}{return}
\SetKw{Continue}{continue}
\SetKw{Break}{break}

\newcommand{\Guide}{\textsc{Guide}}
\newcommand{\Hit}{\textsc{Hit}}
\newcommand{\Target}{\textsc{Target}}
\newcommand{\Cand}{\textsc{ViewCand}}
\newcommand{\Conn}{\textsc{Connect}}
\newcommand{\Prev}{\textsc{Prev}}
\newcommand{\Traj}{\textsc{Traj}}
\newcommand{\Obs}{\textsc{Observe}}

\newcommand{\RuntimePlanner}{\textsc{RuntimePlanner}}
\newcommand{\Child}{\textsc{Child}}

\newcommand{\PropagateNoGood}{\mathrm{PropagateExclusion}}
\newcommand{\Fail}{\textsc{Fail}}
\newcommand{\Exh}{\textsc{Exh}}
\newcommand{\Refut}{\textsc{Refut}}

\newcommand{\PlannerName}{\mbox{SCOPE}}
\newcommand{\PlannerNameSpaced}{\mbox{SCOPE }}

\begin{document}

\title{SCOPE: Field-of-View-Aware Path Planning in Unknown Space via Safety-Volume Certification}

\author{Anonymous Authors}

\author{Junbin Yuan, Muqing Cao, Yunwoo Lee, Brady Moon, and Sebastian Scherer
\thanks{This work was supported by Shimizu Corporation.}%
\thanks{Junbin Yuan and Muqing Cao contributed equally to this work.}
\thanks{Junbin Yuan is with the Department of Mechanical Engineering at Carnegie Mellon University, Pittsburgh, PA 15213 USA
(junbiny@andrew.cmu.edu).}%
\thanks{Muqing Cao, Yunwoo Lee, and Sebastian Scherer are with the Robotics Institute, School of Computer Science, Carnegie Mellon University, Pittsburgh, PA 15213 USA.}%
\thanks{Brady Moon is with the Department of Mechanical Engineering, Brigham Young University, Provo, UT 84602 USA.}
}


\IEEEaftertitletext{\vspace{-1\baselineskip}}

\maketitle

\begin{abstract}
Safe navigation with a body-mounted limited-field-of-view
sensor requires the complete robot-inflated volume of an
intended motion to be observed and verified free before
execution. We formulate this requirement as online
safety-volume certification in an unknown voxel map and
construct a certified graph whose vertices correspond exactly
to positions with fully known-free safety volumes. Based on
this representation, we propose \PlannerNameSpaced(Safety Certification through Observation Planning and Execution), 
a planning framework that decouples optimistic goal-directed guidance from certified execution.
\PlannerNameSpaced converts the first uncertified point along an optimistic route
into an explicit observation obligation, resolves it through target-centric viewpoint search, and recursively clears intermediate obligations when useful viewpoints are not yet certified-reachable. A certified preview mechanism and an observation-aware trajectory optimization backend enable smooth execution.
We prove conditional completeness: under ideal monotone
sensing and exhaustive finite-domain graph search, \PlannerNameSpaced reaches the goal
whenever a finite feasible sequence of certified sensing actions exists
within its planning primitives.
Across 100 randomized tasks in five unknown 3D environments,
\PlannerNameSpaced reaches every goal while maintaining near-zero entry into
non-certified inflated space, and an ablation shows that the certified preview mechanism reduces mean mission time by 27\%.
Finally, we validate the complete system through real-robot demonstrations in four scenarios.
Project website: \href{https://yuanjunbin.github.io/scope-planner/}{yuanjunbin.github.io/scope-planner}.
\end{abstract}


\section{Introduction}
\label{sec:intro}
Autonomous micro aerial vehicles are increasingly expected to operate in cluttered and partially unknown three-dimensional environments, such as industrial facilities, collapsed buildings, tunnels, atriums, and multi-level indoor structures. In these settings, safe navigation is not only a problem of generating dynamically feasible and collision-free motion with respect to the currently reconstructed map. A vehicle equipped with a body-mounted depth sensor must also ensure that the space it is about to occupy has been observed with sufficient range and field of view (FOV). This requirement becomes particularly important when the desired motion is not aligned with the camera's nominal viewing direction.

Most quadrotor planning systems implicitly exploit a convenient property of horizontal flight: yaw can be controlled to approximately align the sensor with the direction of motion. Under this ``look-where-you-fly'' assumption, observing the forward frustum often provides enough information for local obstacle avoidance. However, this assumption breaks down in fully three-dimensional navigation. When the robot must climb through a floor opening, pass under overhangs, or move through vertically separated structures, 
yaw control alone cannot align a
forward-facing sensor with the vertical safety volume required by the intended motion. 
In addition, safety cannot be certified by observing only the trajectory
centerline or a sparse set of future waypoints: the entire robot-inflated volume
swept by the executed motion must be observed and verified free before
traversal.

\begin{figure}[t]
    \centering
    \begin{tikzpicture}
        \node[anchor=south west, inner sep=0] (image) at (0,0) {
            \includegraphics[trim={1300px 250px 1100px 250px},clip,width=\columnwidth]{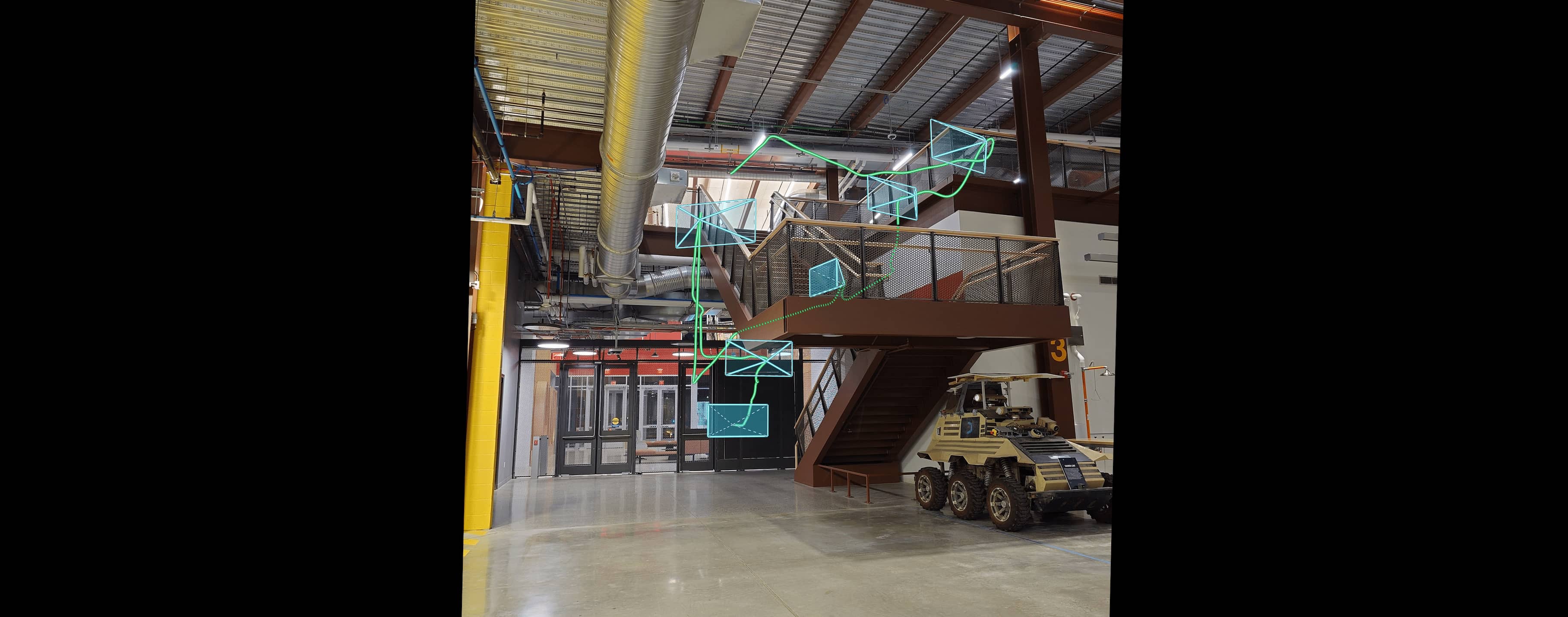}
        };
        \begin{scope}[x={(image.south east)},y={(image.north west)}]
            \fill[ForestGreen, opacity=0.9] (0.32, 0.21) circle (3pt);
            \node[
                anchor=west,
                font=\scriptsize,
                text=Snow2
            ] at ([yshift=6.5pt,xshift=-10pt]0.32, 0.21) {Start};
            \fill[Firebrick3, opacity=0.9] (0.32, 0.835) circle (3pt);
            \node[
                anchor=west,
                font=\scriptsize,
                text=Snow2
            ] at ([yshift=6.5pt,xshift=-15pt]0.32, 0.835) {Goal};
        \end{scope}
    \end{tikzpicture}
    \caption{
    Real-world demonstration of \PlannerNameSpaced Planner in a previously unknown multi-level
    indoor scenario. The UAV incrementally clears the space required for ascent
    using the viewpoints illustrated in cyan, and safely reaches the upper
    level along the green trajectory.
    }
    \label{fig:intro-figure}
    \vspace{-1.0em}                 
\end{figure}
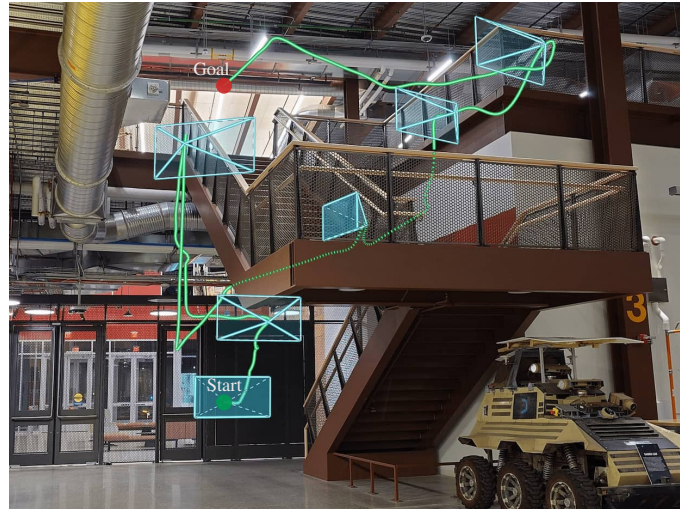


Existing methods address important aspects of navigation in unknown
environments, including local replanning with respect to the map reconstructed
from onboard sensing, goal-directed navigation with incomplete maps, and
perception-aware or FOV-constrained trajectory generation. However, these
methods typically reason about collision avoidance in the current map, progress
toward a goal, or visibility of selected waypoints and targets. These objectives
are not equivalent to certifying the full robot-inflated volume required by the
next motion segment to be executed. As a result, a planner may either produce a
path that appears feasible while leaving safety-critical portions of the volume
to be traversed unobserved, or fail to make progress because it lacks a specific
observation target that connects active perception to the blocked goal-directed
motion. This paper instead asks the question central to safe limited-FOV
navigation: before executing a motion segment, has the entire inflated safety
volume needed by that segment been observed and certified free?


This paper proposes \PlannerName, a field-of-view-aware planning framework for
goal-directed quadrotor navigation in unknown three-dimensional environments.
The central idea is to decouple task progress from safety certification.
\PlannerNameSpaced maintains two complementary graph views of the map: an
optimistic graph that provides goal-directed guidance through regions not yet
certified for execution, and a certified graph that restricts execution to
known-free voxels whose robot-inflated safety volumes have already been
verified. Rather than treating perception as a generic exploration objective,
\PlannerNameSpaced derives sensing objectives from the certification boundary
encountered by goal-directed progress. The planner then resolves the resulting
safety-volume certification deficit through targeted observation before the
corresponding motion is allowed to execute.

This formulation enables behavior that differs from previous FOV-aware planners.
The vehicle may move sideways, backward, or to an intermediate vantage point in
order to observe and certify the safety volume required for future progress;
Fig.~\ref{fig:intro-figure} shows an example in which the vehicle clears the
ascent volume from a sequence of viewpoints before reaching the upper level. If
a useful viewpoint is not currently reachable through certified space,
\PlannerNameSpaced recursively treats that viewpoint as a subgoal and clears the
safety volumes needed to reach it. A preview mechanism further allows the robot
to continue moving smoothly through ordinary clearing actions instead of
repeatedly stopping at each sensing pose. The resulting planner therefore
separates where the robot should move from what it must certify, while
maintaining a hard certification boundary for executed motion.
We establish a
conditional-completeness guarantee for the recursive search. Under a bounded workspace, ideal monotone sensing, and exhaustive finite-domain graph
search, if there exists a finite feasible sequence of certified sensing
actions that can be generated and connected by the planner, then exhaustive
\PlannerName-search reaches the goal after finitely many sensing and replanning
episodes. 


The contributions of this paper are:
\begin{enumerate}
    \item We formulate limited-FOV goal-directed navigation as an online
    safety-volume certification problem, requiring every executed motion
    segment to remain inside robot-inflated space that has already been
    observed and certified free.

    \item We propose \PlannerName-search, which decouples optimistic goal-directed
    guidance from certified execution by converting the first uncertified
    hitpoint into an explicit safety-volume clearing obligation. Target-centric
    viewpoint search, recursive viewpoint subgoals, and context-local exclusion
    memory enable the planner to resolve observation obligations whose useful
    viewpoints are not initially certified-reachable.

    \item We develop a certified preview mechanism and an observation-aware
    trajectory backend that realize the discrete search output as smooth,
    dynamically feasible motion while retaining explicit observation
    constraints.

    \item We establish  conditional completeness of our search algorithm: the exhaustive planner reaches the goal
    whenever a finite feasible sequence of certified sensing actions exists within its candidate-generation and certified-connection primitives. 
\end{enumerate}

We validate the complete system on goal-directed navigation tasks in challenging
3D scenes with vertical passages, occlusions, and multi-level structures, and
compare it against representative perception-aware and target-reach planners.


\section{Related Work}

Related work relevant to our problem spans several complementary roles of
perception in robot planning: enabling goal-directed navigation as an
environment is revealed, selecting viewpoints to acquire useful observations,
shaping aerial motion through perception-aware objectives, and constraining
safe motion under limited sensor FOV. These directions address different
parts of the interaction between motion and sensing. Our problem lies at
their intersection: the robot must make goal-directed progress in an unknown
3D environment while actively acquiring the observations needed to certify
safe execution. We review these lines in turn, from broader navigation and
viewpoint-planning methods to approaches most directly concerned with
limited-FOV safety and certification.

\subsection{Goal-Directed Navigation with Incomplete Maps}

Goal-directed navigation with incomplete environment knowledge has a long history in robotics. 
Classical sensor-based navigation, including Bug-type methods, DistBug, and
TangentBug~\cite{lumelsky1986dynamic,kamon1997distbug,kamon1998tangentbug},
together with incremental graph-search methods such as D* and D*
Lite~\cite{stentz1995dstar,koenig2002dstar}, established the paradigm of
goal-directed motion with online environment revelation. Their sensing
models, however, generally abstract perception as local occupancy or range
information rather than explicitly reasoning about limited sensor FOV and
pre-execution certification of the robot's swept volume. 

Modern aerial planners combine online path search
and trajectory optimization for fast navigation in unknown cluttered
environments~\cite{zhou2019fastplanner,zhou2021ego,zhou2020topotraj,ren2022bubble}.
Within this broader line of work, FASTER maintains safe fallback trajectories
while planning through unknown space~\cite{tordesillas2021faster}, FAR Planner
dynamically updates a global visibility graph as the environment is
revealed~\cite{yang2021far}, and SUPER directly plans high-speed trajectories
from onboard LiDAR point clouds~\cite{ren2025super}.
These methods provide strong mechanisms for replanning, collision avoidance,
and goal progress, but do not explicitly require the complete body-inflated
volume of the next motion to be observed and certified before execution.

Conceptually closest to our safety objective are the ``look before you
sweep'' formulation and VAMP~\cite{goretkin2018look,goretkin2022vamp},
which require workspace regions to be observed before the robot sweeps
through them. Our setting differs in considering online multirotor navigation
in unknown 3D voxel maps, where motion and sensing directions can be strongly
misaligned, particularly during vertical maneuvers, and where body-inflated
volume certification must be maintained through closed-loop replanning and
trajectory execution rather than for a planar geometric path alone.

\subsection{Viewpoint Selection for Exploration and Visibility}

A broader line of work explicitly reasons about where the robot should observe
and how visibility should be acquired or maintained.
Exploration and active-view planners select sensing targets using frontiers,
next-best-view criteria, or information gain~\cite{yamauchi1997frontier,
connolly1985nbv,bircher2016nbvp,harutyunyan2025mapexrl,baek2025pipe}.
Hierarchical frameworks such as FUEL, TARE, and GBPlanner further combine
local exploration decisions with global planning to scale exploration to
large or complex environments~\cite{zhou2021fuel,cao2021tare,dang2019gbplanner}.
Their objective is typically map coverage or information acquisition rather
than certifying the specific safety volume required for the next goal-directed
motion. Regions with high frontier or information gain need not coincide with
the unresolved space that actually blocks goal-directed progress.

Visibility-aware motion planning has also coupled robot motion with the need
to acquire or maintain visibility of regions or targets, as in watchman-route,
pursuit-evasion, and autonomous-observer problems~\cite{chin1988watchman,
guibas1997pursuit,lavalle1997visibility,latombe1997visibility}.
These formulations typically consider known or planar environments and
region- or target-visibility objectives, whereas our planner selects
observations to certify the safety volume required for goal-directed motion
in an unknown 3D environment.

\subsection{Perception-Aware Aerial Planning}

A distinct line of work treats perception not merely as an input to planning,
but as an explicit objective or constraint on aerial motion.
Existing approaches optimize vehicle motion and sensing direction to maintain
visibility of targets or obstacles, improve visual-inertial estimation, or
reduce visual degradation during tracking~\cite{falanga2018pampc,
bartolomei2020semantic,wang2021visibility,tordesillas2022panther}.
Beyond these objectives, perception-aware planning has also been applied to
gaze-guided teleoperation, safe-and-visible inspection corridors, yaw
optimization under FOV constraints, and trajectory optimization with obstacle
threat and sensing urgency~\cite{wang2022gpa,liu2022starconvex,
wu2024globalyaw,zhang2025spot}.

These methods actively shape motion to improve perception, but their sensing
objectives differ from ours: they typically concern target, obstacle, or
landmark visibility, estimation quality, inspection, or trajectory-level
sensing utility. They do not generally require the body-inflated volume of each executed
segment to have been previously observed and certified free.
This distinction is important in vertical and multi-level environments,
where yaw alone cannot generally expose the required safety volume to a
forward-looking sensor, particularly for upward or downward motion.
Our planner instead decouples motion from observation direction and explicitly
plans observations to certify the safety volume required for execution.

\subsection{Limited-FOV Safety Planning and Certification}

A more directly related line of work considers how limited onboard sensing
constrains safe aerial motion.
Limited-FOV aerial planners have addressed restricted onboard perception
through local sensing histories, motion-primitive reuse, safe stopping
policies, and visibility-constrained search or trajectory optimization~
\cite{liu2016limited,lopez2017aggressive3d,lopez2017limitedfov,
florence2018nanomap,nieuwenhuisen2019sensorvisibility}.
Together, these methods establish mechanisms for safe motion under limited
sensing, ranging from local collision avoidance and reuse of previously
observed information to constraining planned motion to remain observable
within the sensor FOV.

More closely related are CPA-Planner~\cite{yu2022cpa},
multi-FOV-constrained planning~\cite{wang2024multifov} based on fast marching tree (FMT*), and
OmniPlanner~\cite{zacharia2026omni}.
CPA-Planner requires unknown portions of a planned trajectory to be observed
in advance and from a safe distance, while multi-FOV-constrained planning
explicitly accounts for multiple sensor FOVs during path and trajectory
generation. OmniPlanner combines persistent global planning with
exploration-based information acquisition to support progress toward an
unreached goal.
Despite these advances, their observation requirements are tied primarily to
trajectory-level observability, FOV-admissible motion, or general exploration
objectives rather than an explicit certification obligation over the
body-inflated transit volume whose unresolved portion blocks goal-directed
progress.

Our method, \PlannerName, instead separates optimistic goal-directed progress,
active observation, and certified execution. It turns the first uncertified
point along an optimistic route into a safety-volume observation obligation
and plans viewpoints to resolve that obligation before execution, allowing
motion and sensing directions to differ in fully three-dimensional structures
such as vertical transitions and narrow openings.

\section{Problem Formulation}
\label{sec:problem}
We consider goal-directed navigation for a multirotor equipped with a sensor rigidly mounted to its body. The robot can yaw to steer the viewing direction,
but it cannot freely roll or pitch its body to align the sensor with an
arbitrary three-dimensional direction of motion. Therefore, when the desired
motion has a strong vertical component, a look-where-you-fly strategy breaks
down: the absence of obstacles in the current forward frustum does not certify
the safety volume required by the intended motion. The robot must first obtain
observations that verify the corresponding volume is free of obstacles before traversing it.

Let the workspace be represented by a voxel graph
\begin{equation*}
    G=(V,E),
\end{equation*}
where $V$ is the set of voxel states and $E$ contains an edge between each pair of $26$-neighborhood
voxels. The workspace is bounded, so $V$ is finite. At planning time, the map belief partitions $V$ into
\begin{equation*}
    V = F \cup U \cup O,
\end{equation*}
where $F$ is the set of known free voxels, $U$ is the set of unknown voxels,
and $O$ is the set of known occupied voxels.
We define the free-space frontier as
\begin{equation*}
    \Phi = \{\mathbf{v}\in F \mid \mathcal{N}_{26}(\mathbf{v})\cap U\neq\emptyset\},
\end{equation*}
where $\mathcal{N}_{26}$ denotes the set of 26-neighbors of $\textbf{v}$; hence $\Phi$ is the
boundary between known free space and unknown space.
Fig.~\ref{fig:guidance_map} illustrates the voxel types.

We identify each voxel with its integer grid coordinate in $\mathbb{Z}^3$.
Let $\delta$ be the voxel resolution. 
Let $r_{xy}$ and $r_z$ denote the
horizontal safety half-width and vertical safety half-height, respectively. The robot's safety radii are rounded to integer: 
\[
n_{xy}=\left\lceil \frac{r_{xy}}{\delta}\right\rceil,\qquad
n_z=\left\lceil \frac{r_z}{\delta}\right\rceil .
\]
The robot safety kernel is the axis-aligned box voxel set
\[
\mathcal{R}
=
\{(i,j,k)\in\mathbb{Z}^3
\mid
|i|\le n_{xy},\ |j|\le n_{xy},\ |k|\le n_z
\},
\]
i.e., an anisotropically scaled Chebyshev ball. The box contains the volume
swept by the robot body under all yaw angles, so the kernel does not need to
rotate with yaw.
Let $\mathcal{R}_f$ be defined in the same way as $\mathcal{R}$, but with smaller safety margins:
\[
n_{xy}^f=\max(0,n_{xy}-1),\qquad
n_z^f=\max(0,n_z-1).
\]
Let $\mathcal{I}_{\mathcal{R}}(A)=A\oplus\mathcal{R}$ denote voxel dilation by
this kernel. The planner maintains inflated obstacle and frontier sets: 
\[
O^+ = \mathcal{I}_{\mathcal{R}}(O),\qquad
\Phi^+ = \mathcal{I}_{\mathcal{R}_f}(\Phi).
\]

Frontier inflation uses a one-voxel-thinner kernel because each frontier
voxel is itself known free. For the box kernels used here, let
\[
    B_\infty(1)
    =
    \left\{
        \mathbf b\in\mathbb Z^3
        \mid
        \|\mathbf b\|_\infty\leq 1
    \right\}
    =
    N_{26}(\mathbf 0)\cup\{\mathbf 0\}.
\]
When \(n_{xy},n_z\geq 1\), the two kernels satisfy
\begin{equation}
    \mathcal R_f\oplus B_\infty(1)=\mathcal R.
    \label{eq:kernel_identity}
\end{equation}
Thus, among known-free voxels outside occupied inflation,
\(\Phi^+\) identifies precisely those whose inflated safety volume still
contains unknown space.

\begin{lemma}[Frontier-inflation characterization]
\label{lem:cert-semantics}
For any
\(\mathbf v\in F\setminus O^+\),
\[
    \mathbf v\in\Phi^+
    \quad\Longleftrightarrow\quad
    \mathcal I_{\mathcal R}(\{\mathbf v\})\cap U
    \neq\emptyset.
\]
\end{lemma}

\begin{proof}
Suppose first that
\(\mathcal I_{\mathcal R}(\{\mathbf v\})\cap U\neq\emptyset\).
Choose an unknown voxel \(\mathbf u\) in this set with minimum
\(\ell_1\)-distance to \(\mathbf v\), and move one 26-neighbor step from
\(\mathbf u\) toward \(\mathbf v\) to obtain \(\mathbf w\).
Coordinatewise,
    $\mathbf w
    \in
    \mathcal I_{\mathcal R_f}(\{\mathbf v\}$).
The minimality of \(\mathbf u\) excludes \(\mathbf w\in U\), while
\(\mathbf v\notin O^+\) excludes
\(\mathbf w\in O\). Hence \(\mathbf w\in F\). Since
\(\mathbf w\) has the unknown neighbor \(\mathbf u\),
\(\mathbf w\in\Phi\), and therefore
\(\mathbf v\in\Phi^+\).

Conversely, if \(\mathbf v\in\Phi^+\), then some
\(\boldsymbol\phi\in\Phi\) lies within \(\mathcal R_f\) of
\(\mathbf v\). Because \(\boldsymbol\phi\) has an unknown
26-neighbor \(\mathbf u\), Eq.~\eqref{eq:kernel_identity} gives $\mathbf u
    \in
    \mathcal I_{\mathcal R}(\{\mathbf v\})$,
and hence
\(\mathcal I_{\mathcal R}(\{\mathbf v\})\cap U\neq\emptyset\).
\end{proof}

\begin{corollary}[Certified-set characterization]
\label{cor:gcert-semantics}
\[
    F\setminus(O^+\cup\Phi^+)
    =
    \left\{
        \mathbf v\in F
        \mid
        \mathcal I_{\mathcal R}(\{\mathbf v\})\subseteq F
    \right\}.
\]
\end{corollary}

\begin{proof}
For any \(\mathbf v\in F\), symmetry of \(\mathcal R\) and
Lemma~\ref{lem:cert-semantics} give
\[
\begin{aligned}
    \mathbf v\notin O^+\cup\Phi^+
    &\Longleftrightarrow
    \mathcal I_{\mathcal R}(\{\mathbf v\})\cap O=\emptyset
    \ \text{and}\
    \mathcal I_{\mathcal R}(\{\mathbf v\})\cap U=\emptyset\\
    &\Longleftrightarrow
    \mathcal I_{\mathcal R}(\{\mathbf v\})\subseteq F.
\end{aligned}
\]
\end{proof}

\begin{figure}[t]
    \centering
    \includegraphics[trim={0.0cm 4.05cm 0.0cm 0.0cm},clip,width=0.8\columnwidth]{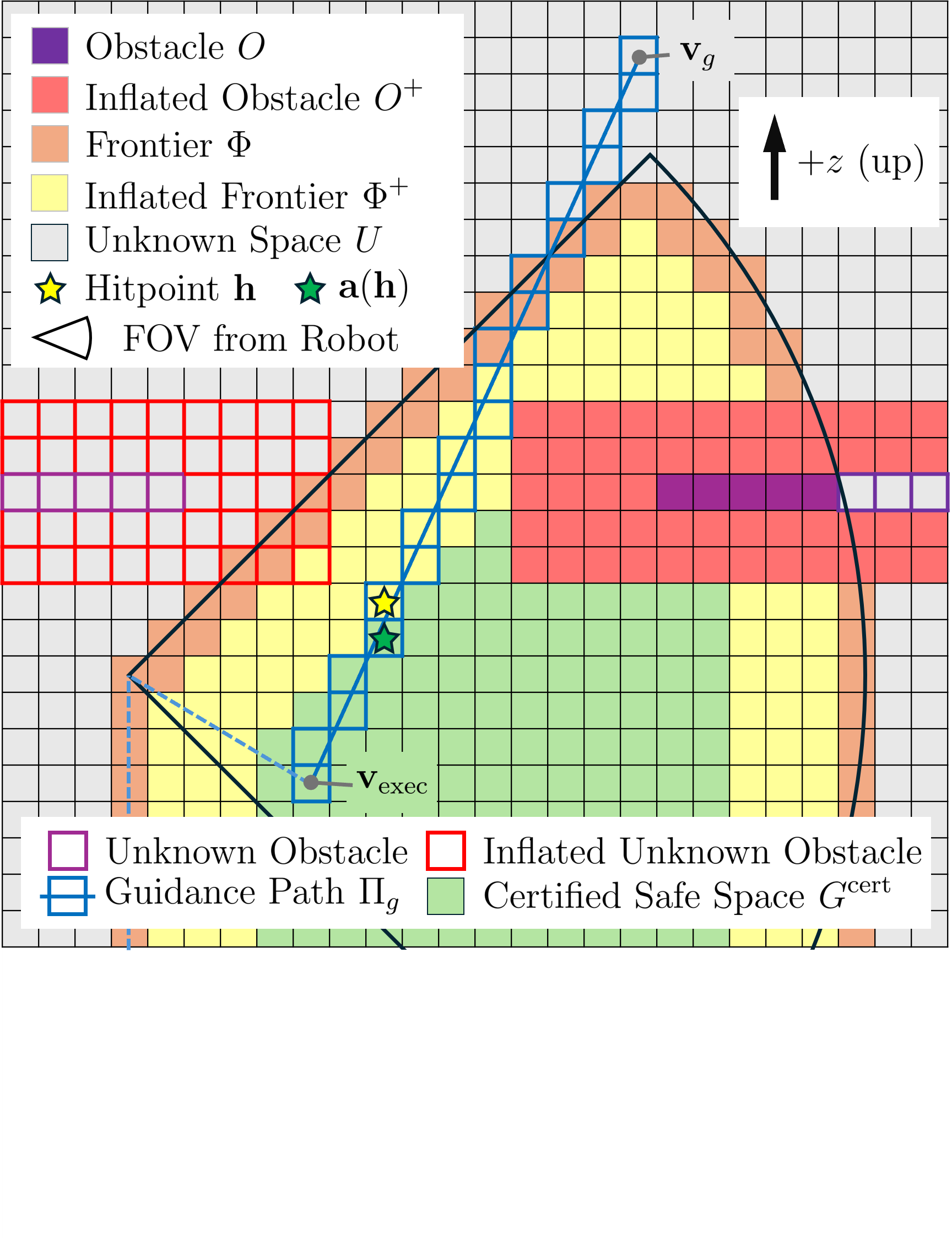}
    \caption{Illustration of map elements and guidance tracing in an example voxel
    map with inflation $n_{xy}=4$ and $n_z=2$. 
    The figure is 2-D illustrative rather than a geometrically exact 3-D
    cross-section of the inflated volume.
    The black sector denotes the onboard sensor field of view,
    and for illustration the robot position is the same as
    the sensor origin. 
    }
    \label{fig:guidance_map}
    \vspace{-0.3cm}
\end{figure}

By
Lemma~\ref{lem:cert-semantics}, the safety volume of every voxel in
$\Phi^+$ intersects unknown space; such voxels are not known to be occupied,
but they cannot be traversed until additional sensing resolves the
corresponding unknown region.

These sets induce two graph views of the map. The \emph{optimistic graph} is
\begin{equation}
    G^{\mathrm{opt}} = G[V \setminus O^{+}],
\end{equation}
which blocks occupied inflation but allows frontier-inflated and unknown
voxels. Paths in this graph may route through unknown or frontier-adjacent voxels
whose surrounding safety volume has not yet been certified.
The \emph{certified graph} is
\begin{equation}
    G^{\mathrm{cert}} = G[F \setminus (O^{+}\cup \Phi^{+})],
\end{equation}
which, by Corollary~\ref{cor:gcert-semantics}, contains exactly the known-free
voxels whose entire inflated safety volume is known free. A path contained in
$G^{\mathrm{cert}}$ requires no additional sensing for safety certification.

Corollary~\ref{cor:gcert-semantics} extends from voxel centers to continuous
positions. The cells of the voxels in
$\mathcal{I}_{\mathcal{R}}(\{\mathbf{v}\})$ jointly cover the box with
half-widths $(n_{xy}+\tfrac{1}{2})\delta$ and half-height
$(n_z+\tfrac{1}{2})\delta$ around $\mathrm{center}(\mathbf{v})$, while a
robot centered anywhere in the cell of $\mathbf{v}$ sweeps at most
$\tfrac{\delta}{2}+r_{xy}$ horizontally and $\tfrac{\delta}{2}+r_z$
vertically from $\mathrm{center}(\mathbf{v})$; since $r_{xy}\le n_{xy}\delta$
and $r_z\le n_z\delta$ by the ceiling definition, the swept safety volume
stays inside known-free cells. We call the union of the cells of the voxels
of $G^{\mathrm{cert}}$ the \emph{continuous certified region}: at any robot
position inside it, the entire safety volume lies in known-free space.

Let $\mathbf{v}_{\mathrm{exec}}\in G^{\mathrm{cert}}$ denote the \emph{execution start
voxel}, the execution start may be the current robot voxel or a safe lookahead
voxel on the active trajectory.
We say that a voxel $\mathbf{v}$ is \emph{certified-reachable} from
$\mathbf{v}_{\mathrm{exec}}$ if there exists a path from $\mathbf{v}_{\mathrm{exec}}$ to $\mathbf{v}$
contained entirely in $G^{\mathrm{cert}}$.

At planning cycle $t$, the robot has pose $(\mathbf{x}_t,\psi_t)$ and map
belief $(F_t,U_t,O_t)$. The goal is specified as
$(\mathbf{x}_g,\psi_g)$. 
The planner operates online: at each cycle it produces a finite
geometric motion segment whose executed portion lies in
$G_t^{\mathrm{cert}}$, obtains new depth observations, updates the map belief,
and replans.

\section{Path-Searching with FOV Awareness}
\label{sec:method}

\begin{figure}[t]
    \centering
    \includegraphics[trim={0.0cm 0.0cm 0.0cm 0.0cm},clip,width=\columnwidth]{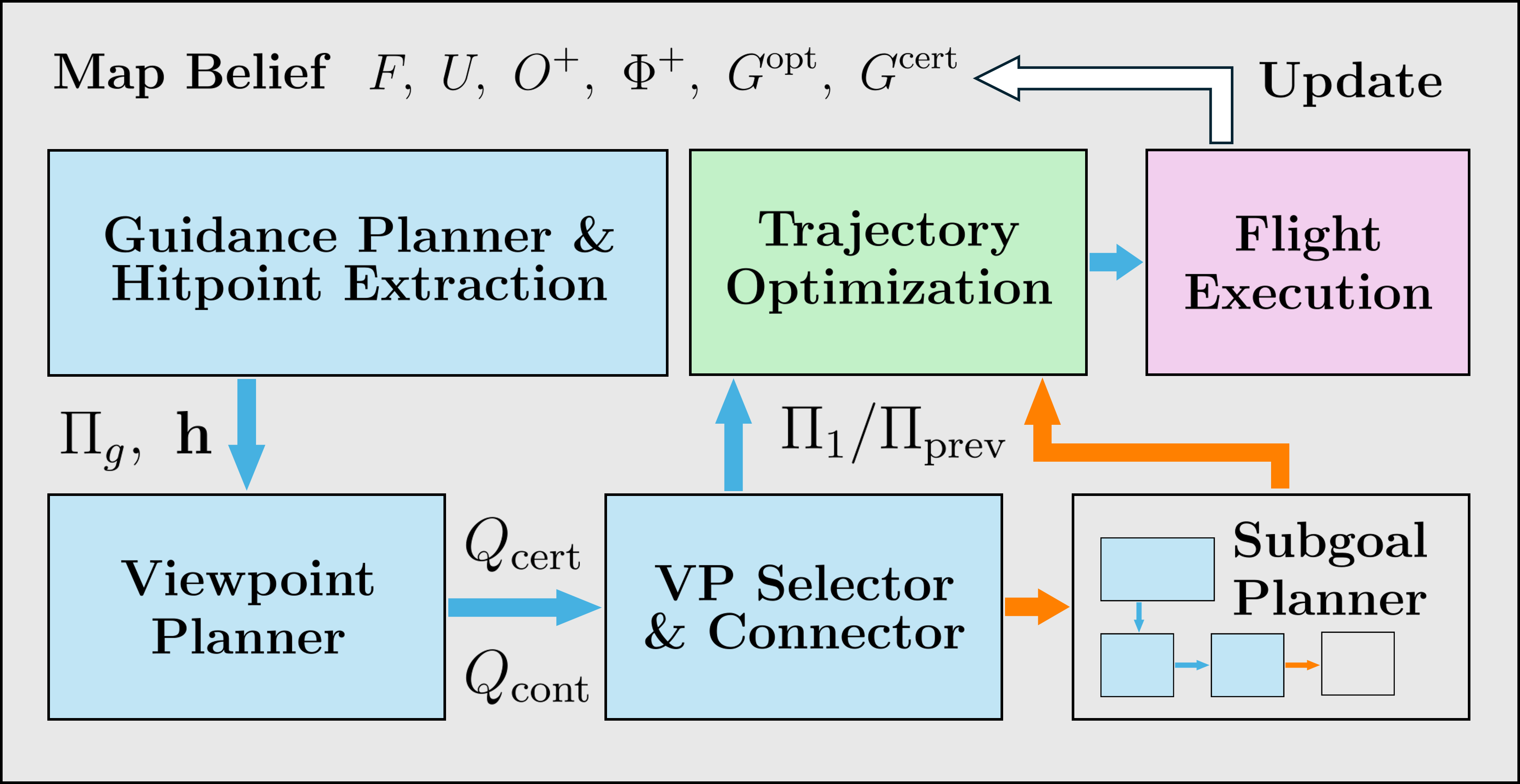}
    \caption{Overview of the search planner.}
    \label{fig:method_overview}
    \vspace{-1.0em}
\end{figure}

Fig.~\ref{fig:method_overview} overviews the search planner. Each planning
cycle traces an optimistic guidance path toward the goal until its first
uncertified voxel, the \emph{hitpoint}, whose inflated safety volume becomes
an observation obligation (Sec.~\ref{sec:global-guidance}); the planner then
generates candidate sensing poses expected to resolve this obligation
(Sec.~\ref{sec:target-centric}) and connects the best certified-reachable
candidate to the execution start (Sec.~\ref{sec:connect}). When every useful
candidate is reachable only through uncertified space, one is pursued
recursively as a subgoal, with a memory of failed attempts preventing
repeated selection (Sec.~\ref{sec:subgoal}).
Sec.~\ref{sec:planning-loop} assembles these stages into the runtime loop
(Alg.~\ref{alg:runtime_pipeline}), and Sec.~\ref{sec:receding-preview} adds a
certified preview layer for smooth execution.

\subsection{Optimistic Guidance and Observation Obligations}
\label{sec:global-guidance}
Given a certified-reachable guidance start voxel $\mathbf{v}_{\mathrm{exec}}$ and a guidance
target voxel $\mathbf{v}_g\in G^{\mathrm{opt}}$, the guidance module returns
an optimistic path in $G^{\mathrm{opt}}$,
\begin{equation}
    \Pi_g = (\mathbf{v}_0, \mathbf{v}_1, \ldots, \mathbf{v}_N), \quad \mathbf{v}_i \in G^{\mathrm{opt}}
\end{equation}
where $\mathbf{v}_0=\mathbf{v}_{\mathrm{exec}}$ and $\mathbf{v}_N=\mathbf{v}_g$. Since
$G^{\mathrm{opt}}$ allows unknown and frontier-inflated voxels, $\Pi_g$ is not
necessarily directly executable. Instead, it is a proposal for task progress
that must undergo local safety certification before execution.

The planner ray-marches $\Pi_g$ from $\mathbf{v}_{\mathrm{exec}}$ toward $\mathbf{v}_g$. The
first voxel on this path that does not belong to $G^{\mathrm{cert}}$ is defined
as the \emph{hitpoint}, denoted by $\mathbf{h}$. The hitpoint marks the first place along the optimistic route
whose surrounding safety volume must be observed before the robot can safely move farther.
We write $\mathbf{a}(\mathbf{h})$ for the voxel immediately preceding
$\mathbf{h}$ along $\Pi_g$; by construction, the front section of $\Pi_g$ up
to $\mathbf{a}(\mathbf{h})$ lies in $G^{\mathrm{cert}}$. See Fig.~\ref{fig:guidance_map} for illustration.

Associated with each hitpoint is an \emph{observation obligation}
\begin{equation}
    \Omega(\mathbf{h}) =
    \mathcal{I}_{\mathcal{R}}(\{\mathbf{h}\})\cap U,
\end{equation}
namely the unknown voxels inside the inflated safety volume of
$\mathbf{h}$. Sensing resolves voxels in $\Omega(\mathbf{h})$ by moving them
from $U$ to either $F$ or $O$. After the resulting map update
$(F',U',O')$, we say that the hitpoint is \emph{cleared} if the entire inflated
safety volume is known free,
\begin{equation}
    \mathrm{Clear}(\mathbf{h}) \Longleftrightarrow
    \mathcal{I}_{\mathcal{R}}(\{\mathbf{h}\}) \subseteq F'.
\end{equation}
By Corollary~\ref{cor:gcert-semantics}, $\mathrm{Clear}(\mathbf{h})$ holds
exactly when $\mathbf{h}$ belongs to the certified graph induced by the
updated map.
If any voxel in this inflated safety volume is found occupied, i.e., 
$\mathcal{I}_{\mathcal{R}}(\{\mathbf{h}\})\cap O'\neq\emptyset$,
then $\mathbf{h}\in O^{+\prime}$ under the updated map, so the current
optimistic guidance path is invalidated and must be replanned.

The guidance module is not tied to a particular graph-search algorithm; it only
needs to produce a route in $G^{\mathrm{opt}}$ from the certified start to the
current target.
Our implementation uses
voxel-space A* from $\mathbf{v}_{\mathrm{exec}}$ to $\mathbf{v}_g$ with a 26-connected
grid-distance heuristic. 
In addition to $O^+$, the guidance planner may be given an extra blocked set
$B$, to be introduced in Section~\ref{sec:subgoal}.

\subsection{Target-Centric Observation Candidate Search}
\label{sec:target-centric}

At a hitpoint $\mathbf{h}$, the target-centric planner searches for sensing
poses that can observe an \emph{observation target}
$\tau\subseteq\mathcal{I}_{\mathcal R}(\{\mathbf h\})$.
A candidate sensing pose is
$q_{\mathrm{vp}}=(\mathbf v_{\mathrm{vp}},\psi_{\mathrm{vp}})$, where
$\mathbf v_{\mathrm{vp}}$ is the viewpoint voxel and $\psi_{\mathrm{vp}}$
is the body yaw. Each target $\tau$ is associated with a representative
point $\mathbf p_\tau$. At a candidate voxel, the yaw is chosen so that the
sensor faces $\mathbf p_\tau$, accounting for the fixed body-to-sensor
extrinsics. The resulting target vector in the sensor frame is
\begin{equation}
    \mathbf{d}_s =
    [d_{s,x},d_{s,y},d_{s,z}]^\top
    =
    R_{ws}^{\top}(\mathbf{p}_\tau-\mathbf{x}_s),
\end{equation}
where $\mathbf x_s$ and $R_{ws}$ are the induced sensor position and
orientation. 
We use two plan types:
\begin{itemize}
    \item \textbf{Volume-clearing plan} (Fig.~\ref{fig:obs_vp_selected}):
    $\tau=\mathcal{I}_{\mathcal{R}}(\{\mathbf{h}\})$ targets the full inflated
    safety volume of the hitpoint, with
    $\mathbf{p}_\tau=\mathrm{center}(\mathbf{h})$ as the representative target
    point. This is the first sensing attempt for a newly encountered hitpoint.

    \item \textbf{Unknown-observation plan}
    (Fig.~\ref{fig:obs_vp_unknown}):
    $\tau=\{\mathbf{u}\}$ with $\mathbf{u}\in\Omega(\mathbf{h})$ targets a
    single unresolved voxel, with
    $\mathbf{p}_\tau=\mathrm{center}(\mathbf{u})$.
    If an executed volume-clearing plan leaves the hitpoint uncleared, the
    planner switches to this mode and resolves the remaining obligation
    through targeted observations of unresolved voxels.
\end{itemize}


\begin{figure}[t]
    \centering
    \includegraphics[trim={3.54cm 5.55cm 0.0cm 2.5cm},clip,width=0.7\columnwidth]{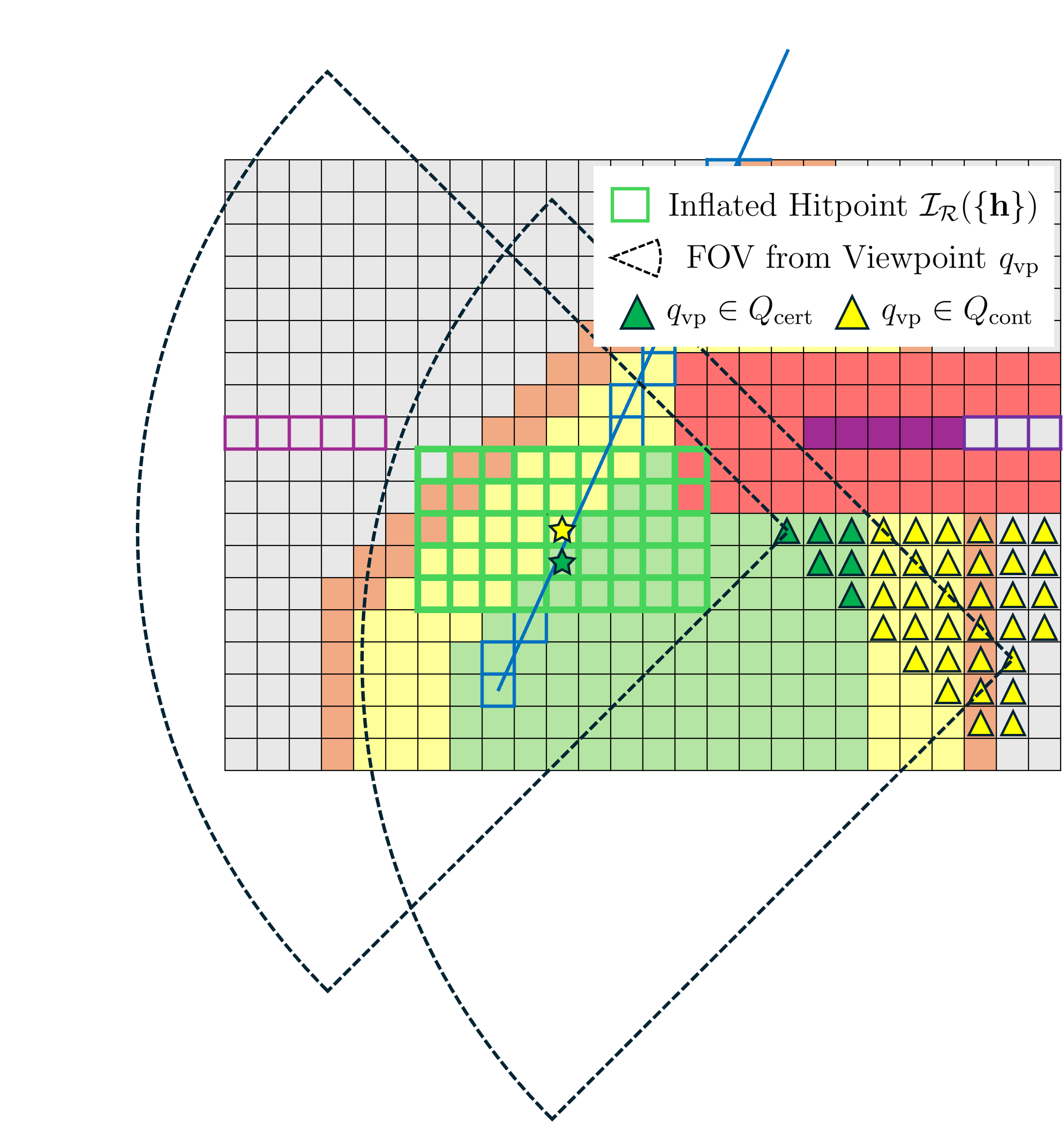}
    \caption{Volume-clearing candidate search.
    The light-green outlined voxels indicate the target
    $\tau=\mathcal{I}_{\mathcal R}(\{\mathbf h\})$.
    Starting from $\mathbf a(\mathbf h)$, the planner searches for candidate
    viewpoints from which the entire inflated safety volume can be observed.
    Map colors follow Fig.~\ref{fig:guidance_map}.}
    \label{fig:clear_vp}
    \vspace{-0.3cm}
\end{figure}

\begin{figure}[t]
    \centering
    \begin{subfigure}[t]{0.49\columnwidth}
        \centering
        \includegraphics[trim={3.0cm 6.55cm 3.525cm 2.01cm},clip,width=\linewidth]{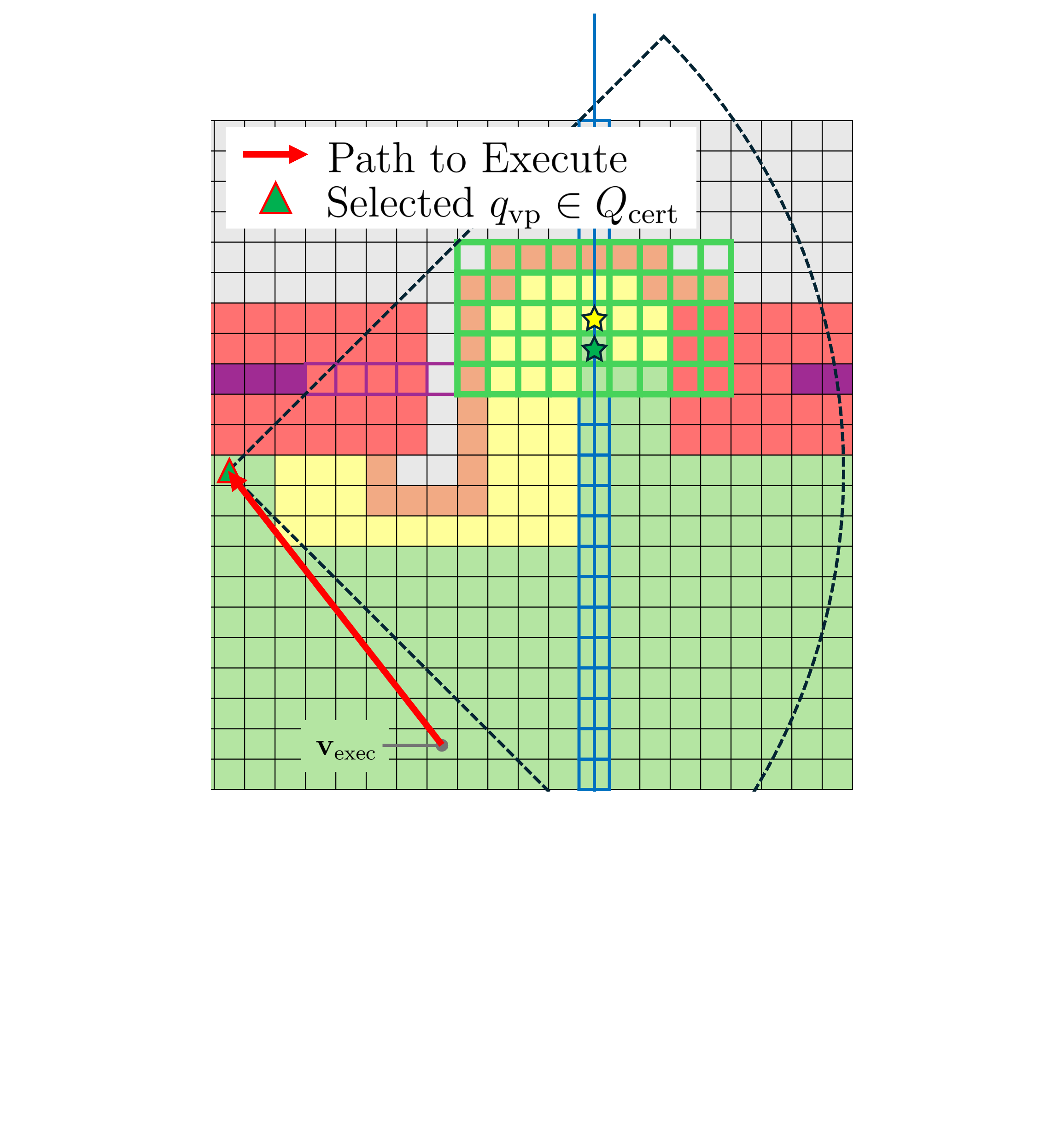}
        \caption{}
        \label{fig:obs_vp_selected}
    \end{subfigure}
    \begin{subfigure}[t]{0.49\columnwidth}
        \centering
        \includegraphics[trim={3.0cm 6.54cm 3.525cm 2.01cm},clip,width=\linewidth]{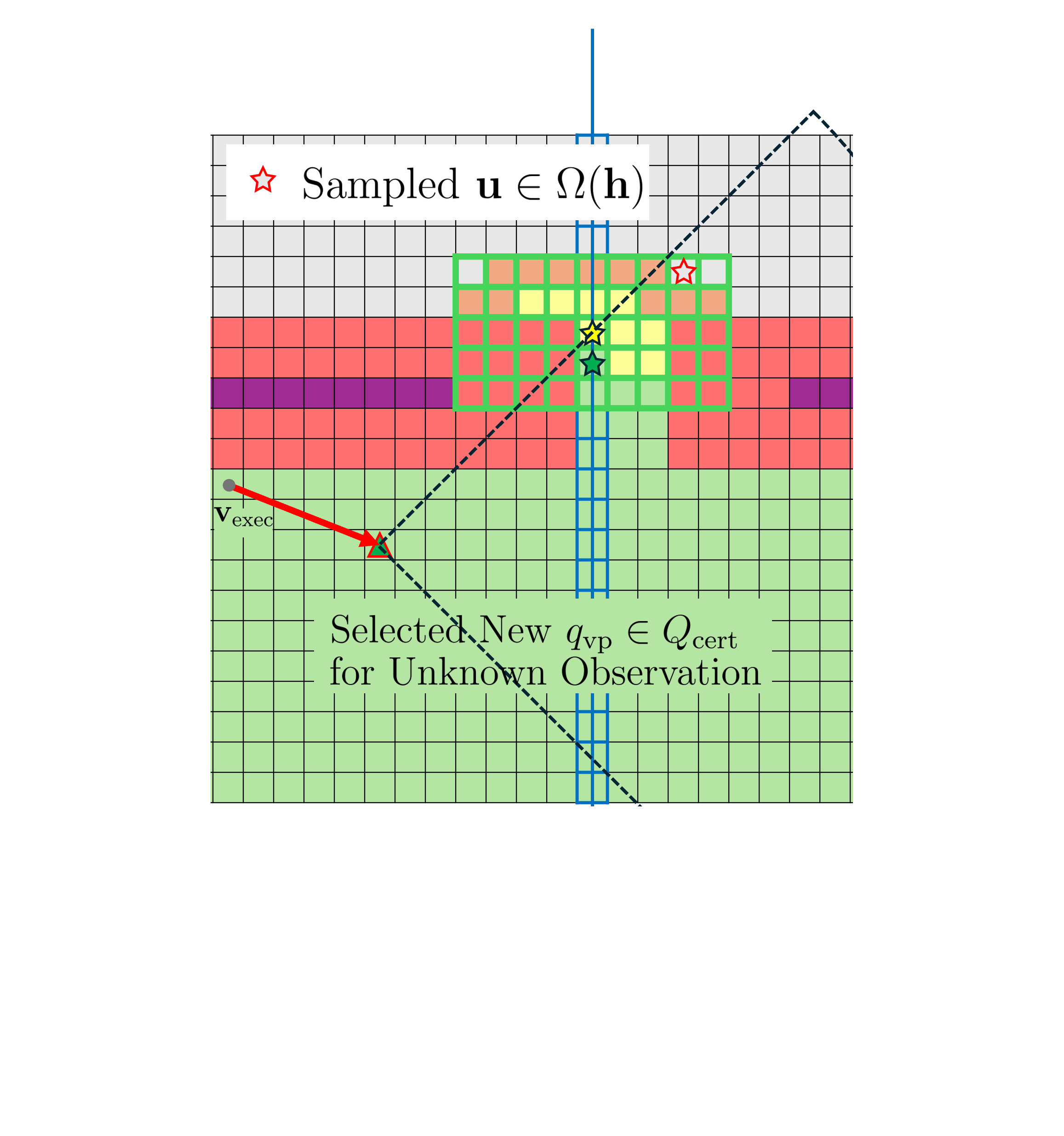}
        \caption{}
        \label{fig:obs_vp_unknown}
    \end{subfigure}

    \caption{Transition from volume clearing to unknown observation.
    (\ref{fig:obs_vp_selected}) A viewpoint $q_{\mathrm{vp}}$ is selected
    to attempt clearing the inflated safety volume of $\mathbf h$. A previously unknown obstacle occludes part
    of the target, leaving voxels in $\Omega(\mathbf h)$ unresolved.
    (\ref{fig:obs_vp_unknown}) The planner targets one remaining unknown voxel
    and searches for a new unknown-observation viewpoint.}
    \label{fig:unknown_observation_planning}
    \vspace{-0.3cm}
\end{figure}

\begin{figure}[t]
    \centering
    \begin{subfigure}[t]{0.49\columnwidth}
        \centering
        \includegraphics[trim={4.0cm 3.0cm 3.0cm 3.5cm},clip,width=\linewidth]{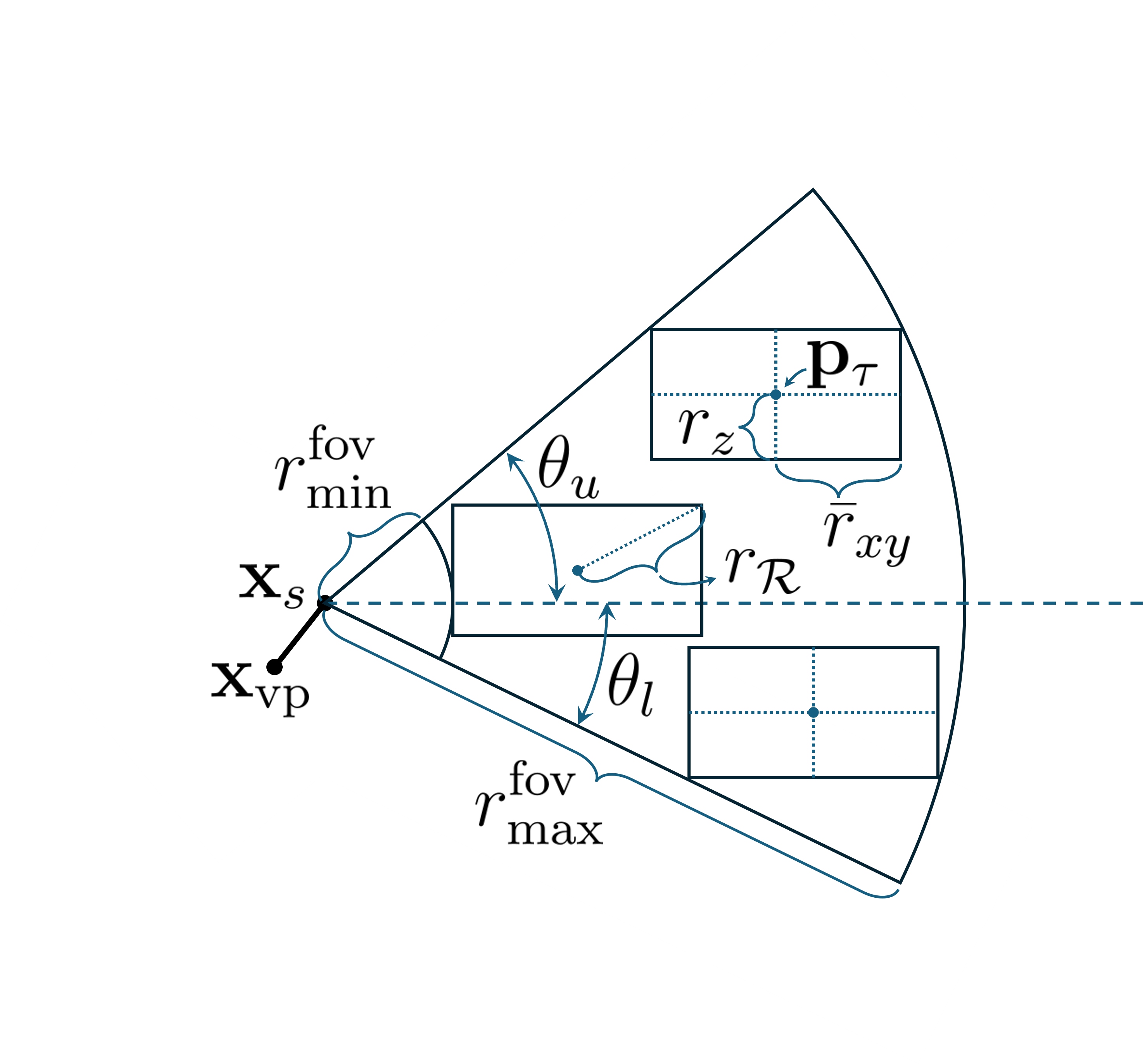}
        \caption{}
        \label{fig:fov_1}
    \end{subfigure}
    \hfill
    \begin{subfigure}[t]{0.49\columnwidth}
        \centering
        \includegraphics[trim={4.0cm 3.0cm 3.0cm 3.5cm},clip,width=\linewidth]{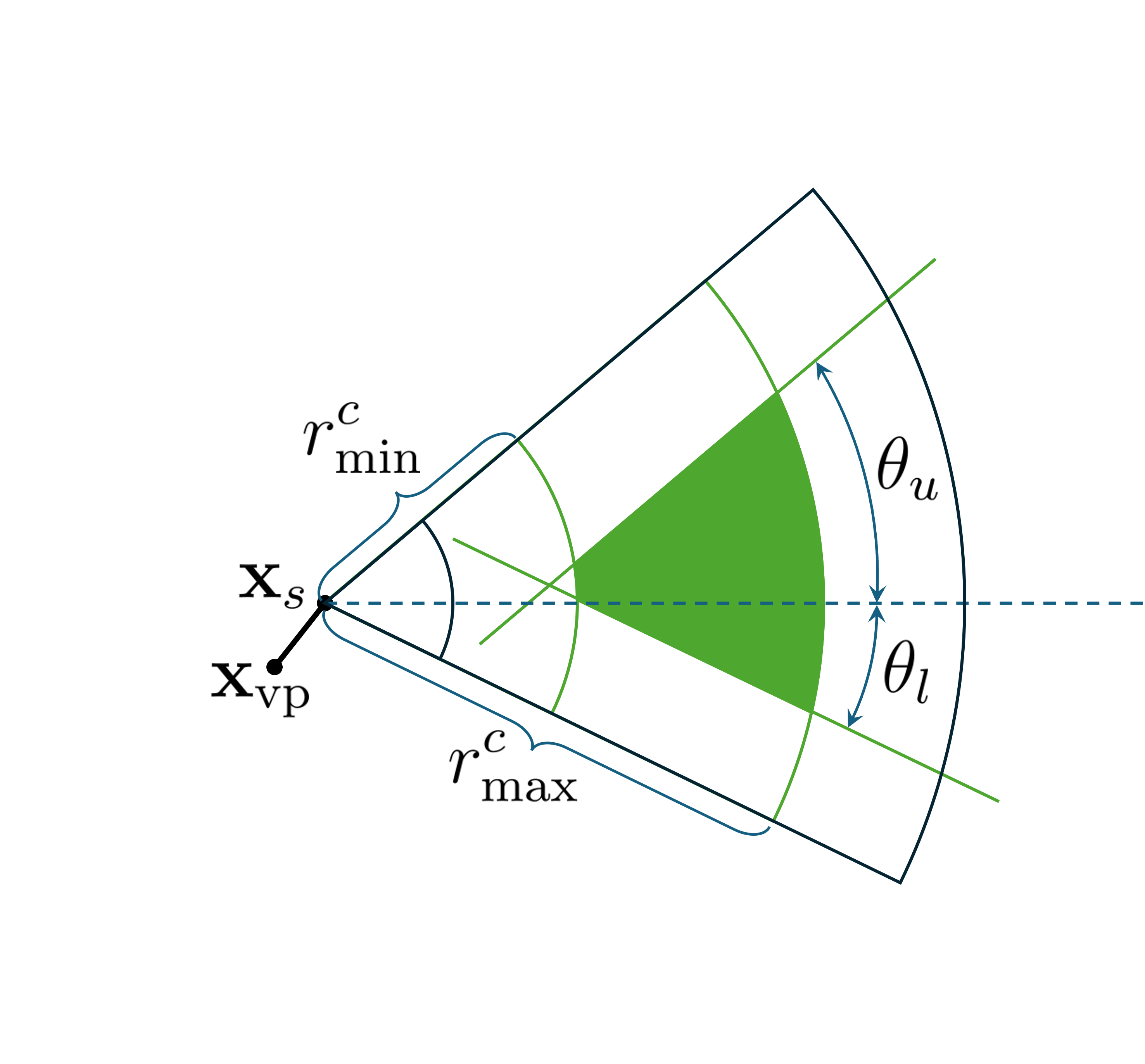}
        \caption{}
        \label{fig:fov_2}
    \end{subfigure}

    \caption{FOV-aware geometry for volume clearing.
    (\ref{fig:fov_1}) The inflated safety volume of $\mathbf h$ must fit within
    the usable sensing range and vertical FOV of the target-facing sensor pose.
    (\ref{fig:fov_2}) These constraints define a feasible region for the
    hitpoint center, shown in green. The heuristic $h_c$ in
    Eq.~\eqref{eq:clear_heuristic} measures the geometric violation of these
    constraints and guides candidate enumeration toward this region.}
    \label{fig:fov_search}
    \vspace{-0.3cm}
\end{figure}

The candidate search is guided by the geometric violation of the
corresponding observation requirement.

For a volume-clearing plan, the
inflated safety volume must fit within the usable field of view.
Let $\bar r_{xy}=\sqrt{2}\,r_{xy}$ denote the box kernel's largest horizontal width and
$r_{\mathcal R}=\sqrt{2r_{xy}^2+r_z^2}$ the circumradius of the box kernel.
Accounting for this extent gives the effective clearing range
\begin{equation}
    r_{\max}^c
    =
    \max\{r_{\max}^{\mathrm{fov}}-r_{\mathcal R},0\},
    \qquad
    r_{\min}^c
    =
    r_{\min}^{\mathrm{fov}}+r_{\mathcal R}.
\end{equation}
The apparent upper and lower vertical angles of the inflated volume are
\begin{equation}
\begin{aligned}
    \theta_c^+
    &=
    \operatorname{atan2}
    (d_{s,z}+r_z,\ d_{s,x}-\bar r_{xy}),\\
    \theta_c^-
    &=
    \operatorname{atan2}
    (d_{s,z}-r_z,\ d_{s,x}-\bar r_{xy}).
\end{aligned}
\end{equation}
With $\theta_u$ and $\theta_l$ denoting the upper and lower vertical FOV
bounds, we define four geometric residuals
\begin{equation}
\begin{aligned}
    \eta^c_{\max} &= \|\mathbf{d}_s\|-r_{\max}^c,
    &
    \eta^c_u &= \|\mathbf{d}_s\|\sin(\theta_c^+-\theta_u),\\
    \eta^c_{\min} &= r_{\min}^c-\|\mathbf{d}_s\|,
    &
    \eta^c_l &= \|\mathbf{d}_s\|\sin(\theta_l-\theta_c^-).
\end{aligned}
\end{equation}
Here $\eta^c_{\max},\eta^c_{\min}$ penalize range violations, while
$\eta^c_u,\eta^c_l$ penalize violations of the vertical FOV bounds.
The volume-clearing heuristic is
\begin{equation}
\label{eq:clear_heuristic}
    h_c(q_{\mathrm{vp}};\mathbf h)
    =
    \frac{1}{\delta}
    \max\{\eta^c_{\max},\eta^c_{\min},\eta^c_u,\eta^c_l,0\}.
\end{equation}

For an unknown-observation plan, no inflated-volume margin is required.
Using the point-target vertical angle
$\theta_o=\operatorname{atan2}(d_{s,z},d_{s,x})$, we similarly define
\begin{equation}
\begin{aligned}
    \eta^o_{\max} &= \|\mathbf{d}_s\|-r_{\max}^{\mathrm{fov}},
    &
    \eta^o_{u} &= \|\mathbf{d}_s\|\sin(\theta_o-\theta_u),\\
    \eta^o_{\min} &= r_{\min}^{\mathrm{fov}}-\|\mathbf{d}_s\|,
    &
    \eta^o_{l} &= \|\mathbf{d}_s\|\sin(\theta_l-\theta_o),
\end{aligned}
\end{equation}
and the unknown-observation heuristic is
\begin{equation}
    h_o(q_{\mathrm{vp}};\mathbf{u})
    =
    \frac{1}{\delta}
    \max\{\eta^o_{\max},\eta^o_{\min},\eta^o_u,\eta^o_l,0\}.
\end{equation}

We write $q_{\mathrm{vp}}\models\tau$ when the observation heuristic
corresponding to the current plan type is zero and the viewing ray to
$\mathbf p_\tau$ is not blocked by any currently known occupied voxel.

Starting from $\mathbf a(\mathbf h)$, which is certified-reachable from
$\mathbf v_{\mathrm{exec}}$ by construction, the candidate search performs a
26-connected heuristic best-first expansion over $G^{\mathrm{opt}}$, using
the accumulated path cost and the corresponding observation heuristic to
guide candidate enumeration. The search additionally records whether the
path from $\mathbf a(\mathbf h)$ has remained entirely in
$G^{\mathrm{cert}}$; once a path enters uncertified space, it and its
descendants are labeled contaminated. A pose satisfying
$q_{\mathrm{vp}}\models\tau$ is placed in $Q_{\mathrm{cert}}(\tau)$ or
$Q_{\mathrm{cont}}(\tau)$ according to this path label (Fig.~\ref{fig:clear_vp}). Since 
$\mathbf a(\mathbf h)$ is certified-reachable, every pose in
$Q_{\mathrm{cert}}(\tau)$ is also certified-reachable from
$\mathbf v_{\mathrm{exec}}$.

\subsection{Certified Candidate Selection via Multi-Source Connect}
\label{sec:connect}

The target-centric search of Sec.~\ref{sec:target-centric}
expands states and evaluates candidate poses based on a local search cost originating from $\mathbf{a}(\mathbf{h})$.
Consequently, the order in which candidates are generated does not reflect
the true certified connection cost from the robot's actual execution start $\mathbf{v}_{\mathrm{exec}}$.
To find the most efficient executable path to any valid observation pose,
we treat the certified viewpoint candidate set
\begin{equation}
    Q_{\mathrm{cert}}(\tau)
    =
    \{q_{\mathrm{vp},1},\ldots,q_{\mathrm{vp},K}\},
\end{equation}
where
$q_{\mathrm{vp},i}=(\mathbf{v}_{\mathrm{vp},i},\psi_{\mathrm{vp},i})$, as a
goal set on the certified graph.
Instead of running one connect search per candidate, we run a reverse
multi-source A* on $G^{\mathrm{cert}}$. All certified viewpoint voxels are
inserted as zero-cost sources, and the execution start voxel $\mathbf{v}_{\mathrm{exec}}$ is
used as the single goal. Since $G^{\mathrm{cert}}$ is undirected with symmetric
edge costs, the first time $\mathbf{v}_{\mathrm{exec}}$ is popped yields the path with the minimum
certified cost from $\mathbf{v}_{\mathrm{exec}}$ to any candidate in
$Q_{\mathrm{cert}}(\tau)$. 

In our implementation, the reverse search uses the same 26-connected
grid-distance heuristic as in the optimistic guidance, evaluated from
$\mathbf{v}$ to the execution start voxel $\mathbf{v}_{\mathrm{exec}}$.
By construction, every candidate pose in $Q_{\mathrm{cert}}(\tau)$
is certified-reachable from the execution start $\mathbf{v}_{\mathrm{exec}}$.
So, the reverse multi-source A* search is theoretically guaranteed to succeed and return a valid connection path.

\subsection{Recursive Subgoals and Exclusion Memory}
\label{sec:subgoal}
\subsubsection{Recursive viewpoint subgoals}
Section \ref{sec:connect} describes how the planner finds a path with
$Q_{\mathrm{cert}}$. We now consider the case where no certified viewpoint can
be found but contaminated viewpoints exist, that is
$Q_{\mathrm{cert}}=\emptyset$ and $Q_{\mathrm{cont}}\neq\emptyset$.
In this case, the planner cannot directly find a path in $G^{\mathrm{cert}}$,
and it starts recursive subgoal planning phase.

Each recursive search level is represented by a \emph{recursive context} $S$,
which stores the active hitpoint, the target being resolved, together with the
exclusion memory to be introduced below. The root context corresponds to the
original planning task.

To observe the target $\tau$ for hitpoint $\mathbf{h}$
in context $S$, the planner selects the lowest-cost contaminated viewpoint
$q_{\mathrm{vp}}\in Q_{\mathrm{cont}}$ 
, and sets it as a recursive subgoal.
The subgoal is resolved by the same optimistic-guidance and hitpoint-clearing procedure: the
guidance planner searches optimistically toward $\mathbf{v}_{\mathrm{vp}}$, identifies
the hitpoint on that subgoal guidance path, and attempts to clear it.
If all required hitpoints are cleared, the contaminated viewpoint becomes
certified-reachable and can then be used to observe the original target.

\begin{figure}[t]
    \centering
    \begin{subfigure}[t]{0.49\columnwidth}
        \centering
        \includegraphics[trim={5.6cm 6.0cm 3.525cm 3.04cm},clip,width=\linewidth]{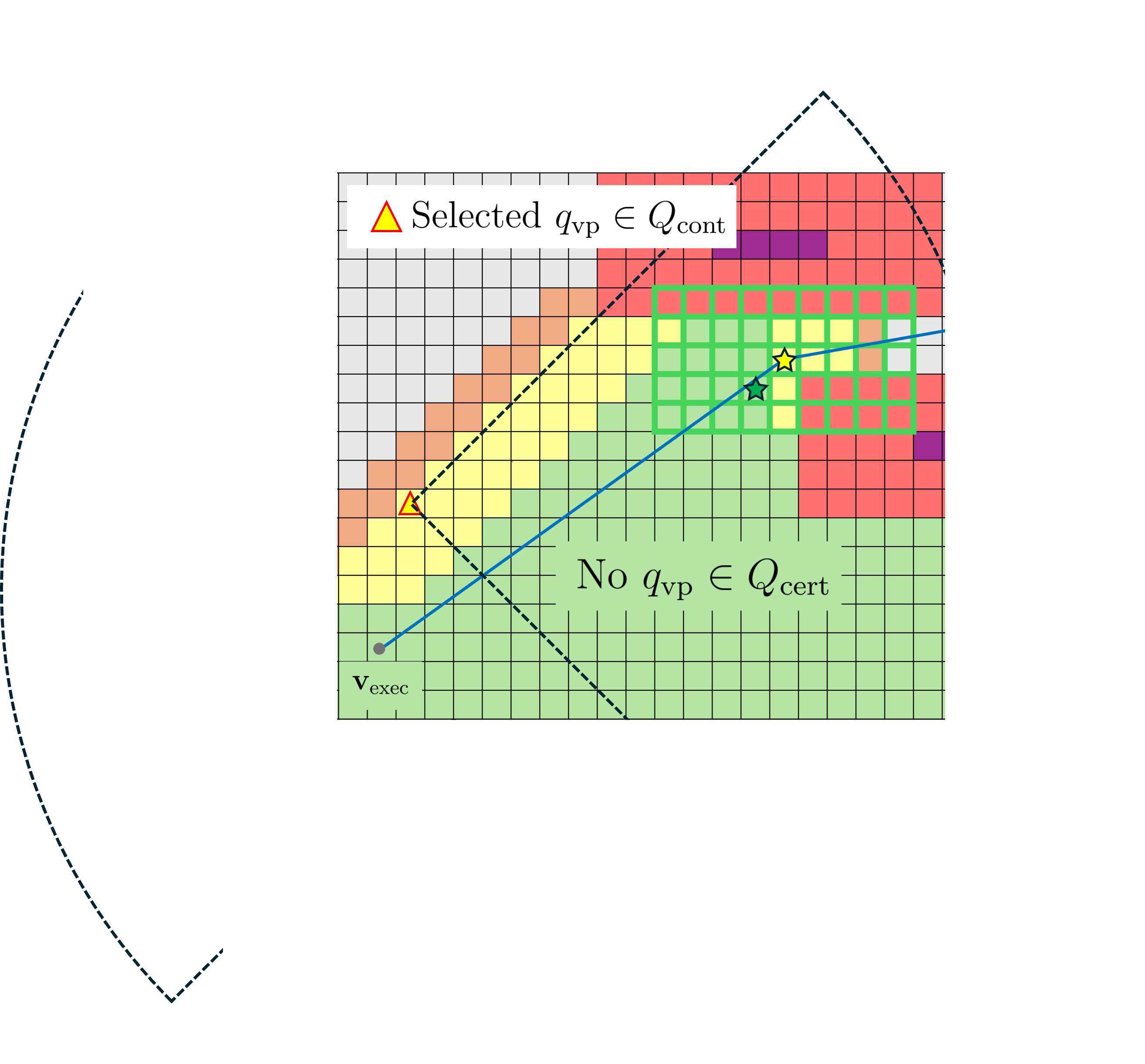}
        \caption{}
        \label{fig:subgoal_1}
    \end{subfigure}
    \begin{subfigure}[t]{0.49\columnwidth}
        \centering
        \includegraphics[trim={5.6cm 6.0cm 3.525cm 3.04cm},clip,width=\linewidth]{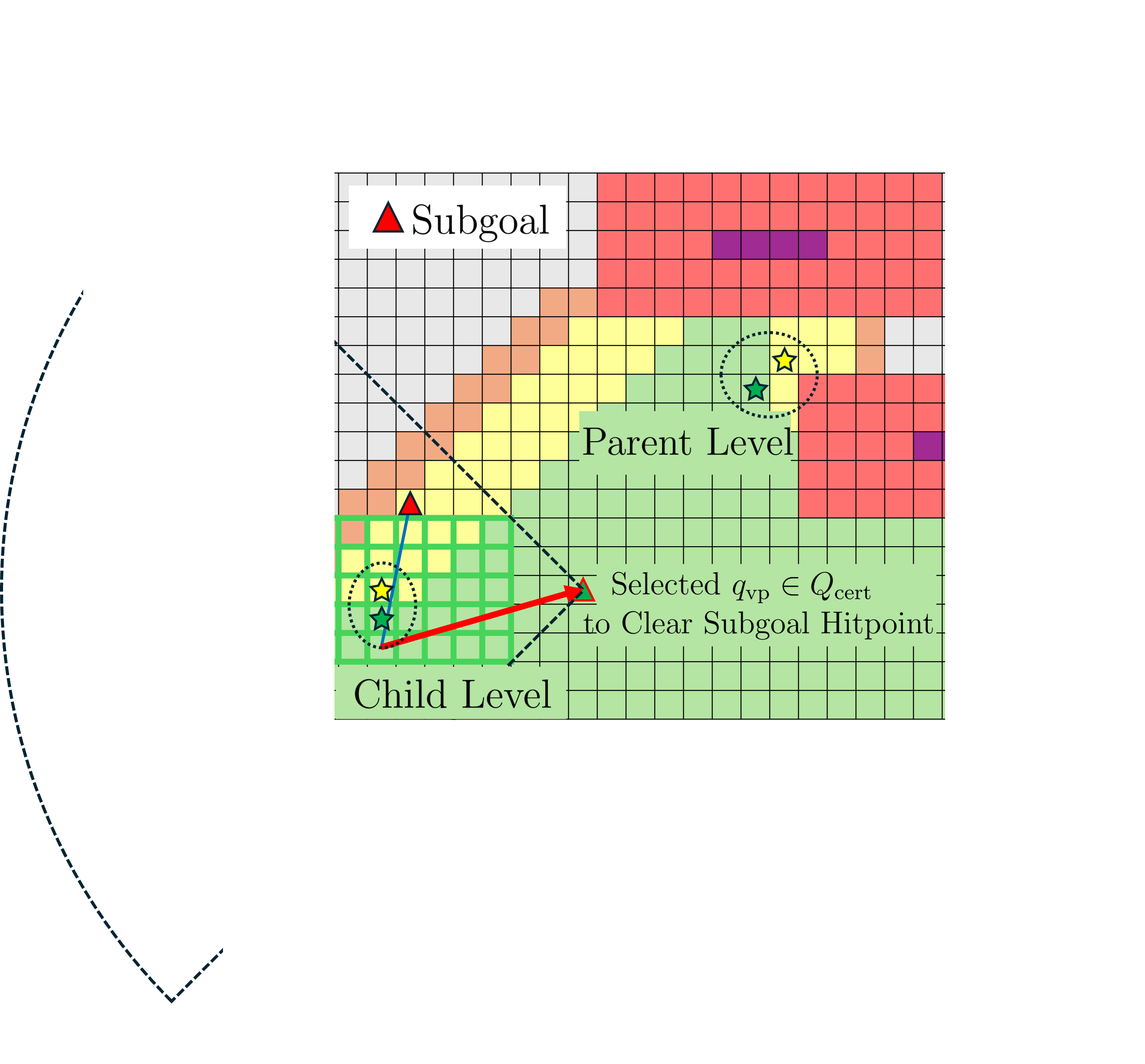}
        \caption{}
        \label{fig:subgoal_3}
    \end{subfigure}

    \caption{An example of subgoal generation.
    ~(\ref{fig:subgoal_1}) There is no certified viewpoint that can clear the current hitpoint,
    so a contaminated viewpoint is selected and targeted as a subgoal.
    ~(\ref{fig:subgoal_3}) To help reach the subgoal, the planner searches for a subgoal guidance path
    that has its own hitpoints to be cleared.
     It then searches for viewpoints that help with
    clearing them, similar to the primary planning pipeline.}
    \label{fig:subgoal_generation}
    \vspace{-0.3cm}
\end{figure}
\subsubsection{Context-local exclusion memory (Fig.~\ref{fig:recursive_subgoal_generation})}
Each context $S$ carries three local exclusion sets which are used in the search.
First, $H_{\mathrm{anc}}(S)$ is the ancestor-hitpoint set. It contains the hitpoint voxels
already encountered along the recursive chain from the root context to $S$.
These hitpoints are regarded as blocked in descendant guidance searches so that the recursion
cannot cycle back to an unresolved hitpoint already on its ancestry.

Second, $Q_{\mathrm{rej}}(S)$ stores the voxels of the rejected viewpoints generated in
context $S$. The voxel $\mathbf{v}_{\mathrm{vp}}$ of a viewpoint $q_{\mathrm{vp}}$
is rejected when its subgoal planning fails under the current context: no
guidance path to $q_{\mathrm{vp}}$ remains after every hitpoint encountered
on the successive subgoal guidance paths has been exhausted, as defined
below.
$Q_{\mathrm{rej}}(S)$ acts as the failed-set for replanning in context $S$:
once a viewpoint is rejected, the planner will not select the viewpoints at the same
voxel again in this context.

Third, $H_{\mathrm{exh}}(S)$ stores hitpoints that are exhausted in
context $S$. A hitpoint $\mathbf{h}$ is exhausted only after the following
two-stage test fails. First, the planner attempts to clear the volume around
$\mathbf{h}$. If there is no viewpoint in $Q_{\mathrm{cert}}$ for this clearing target,
and every viewpoint in $Q_{\mathrm{cont}}$ is rejected in the child context,
written $S_+$, and added to $Q_{\mathrm{rej}}(S_+)$,
the planner samples an unknown voxel $\mathbf{u}\in\Omega(\mathbf{h})$ and falls back to the
unknown-observation planning. If this unknown-observation planning also yields no viewpoint in
$Q_{\mathrm{cert}}$, and every viewpoint in $Q_{\mathrm{cont}}$ is rejected as subgoal
in the child context and added to $Q_{\mathrm{rej}}(S_+)$,
the hitpoint is inserted into $H_{\mathrm{exh}}(S)$.
Sampling a single unknown voxel suffices for this test: clearing
$\mathbf{h}$ requires resolving every voxel in $\Omega(\mathbf{h})$, so one
unknown voxel that cannot be observed by any admissible candidate under the
current context already certifies that $\mathbf{h}$ cannot be cleared in this
context.
Exhausted hitpoints in $H_{\mathrm{exh}}(S)$ are the hitpoints that cannot be
cleared in the context, so future guidance searches under the context should
avoid them, which means they are regarded as blocked like $H_{\mathrm{anc}}(S)$.

Using the blocked-set notation from the optimistic guidance~\ref{sec:global-guidance}
, the context-local blocked set is
\begin{equation}
    B(S)
    =
    H_{\mathrm{anc}}(S)
    \cup H_{\mathrm{exh}}(S).
\end{equation}
A voxel $\mathbf{v}$ is treated as blocked in context $S$ if
$\mathbf{v}\in O^+ \cup B(S)$.
\subsubsection{Reuse and invalidation of exclusions}
In a child context $S_+$, the exclusion set is no smaller and
$H_{\mathrm{anc}}(S_+)=H_{\mathrm{anc}}(S)\cup\{\mathbf{h}\}$ is larger, so any hitpoint that is
exhausted under $B(S)$ remains exhausted under $B(S_+)$, and any
viewpoint voxel whose subgoal fails under $B(S)$ keeps failing under
$B(S_+)$ (Lemma~\ref{lem:nogood-monotone}). 
Therefore, a child context can safely
inherit $H_{\mathrm{exh}}(S)$ and $Q_{\mathrm{rej}}(S)$. Conversely, the
memory accumulated by a failed child is a family of failure certificates
whose derivations block the parent hitpoint $\mathbf{h}$, so it cannot prune
anything in $S$ while $\mathbf{h}$ is still open; it may, however, be reused
directly by sibling children created for the same hitpoint. Once
$\mathbf{h}$ is exhausted and inserted into $H_{\mathrm{exh}}(S)$, the parent
exclusion set covers the dependencies of every child certificate, and the
entire $H_{\mathrm{exh}}(S_+)$ and $Q_{\mathrm{rej}}(S_+)$ are propagated
back to the parent context (Corollary~\ref{cor:nogood-reuse}).

Before expanding the deepest active context, the planner performs a top-down
necessity check along the current context chain. Each non-root context is
justified by the parent hitpoint that created its subgoal. If a map
update has already cleared the parent-context hitpoint, the corresponding child
context and all of its descendants are no longer needed.
The planner prunes that subtree and resumes planning from the parent context.


\begin{figure}[t]
    \centering



    \begin{subfigure}[t]{0.48\columnwidth}
        \centering
        \includegraphics[trim={3.5cm 8.0cm 0.0cm 8.0cm},clip,width=\linewidth]{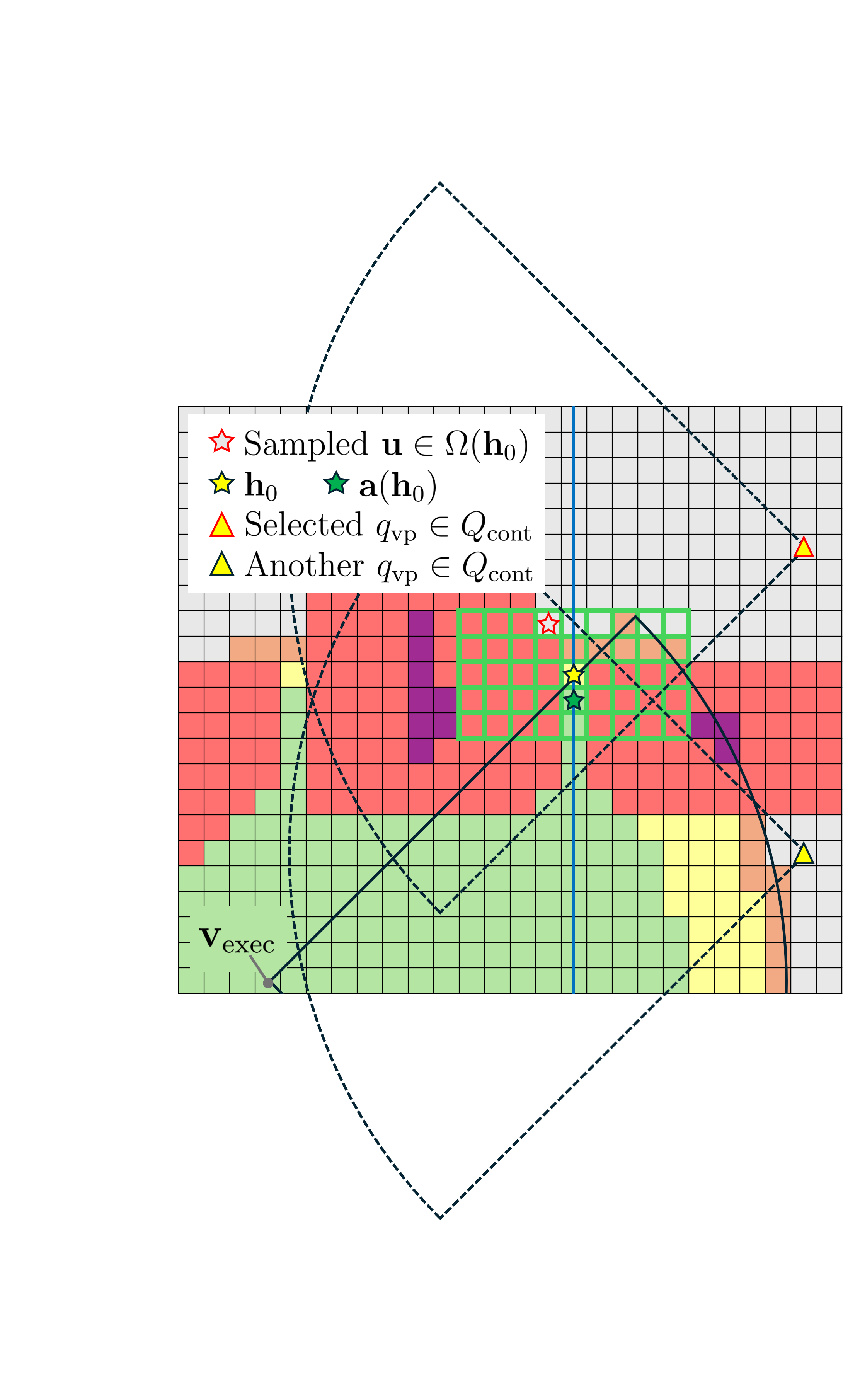}
        \caption{}
        \label{fig:recursive_1}
    \end{subfigure}
    \hfill
    \begin{subfigure}[t]{0.48\columnwidth}
        \centering
        \includegraphics[trim={3.5cm 8.0cm 0.0cm 8.0cm},clip,width=\linewidth]{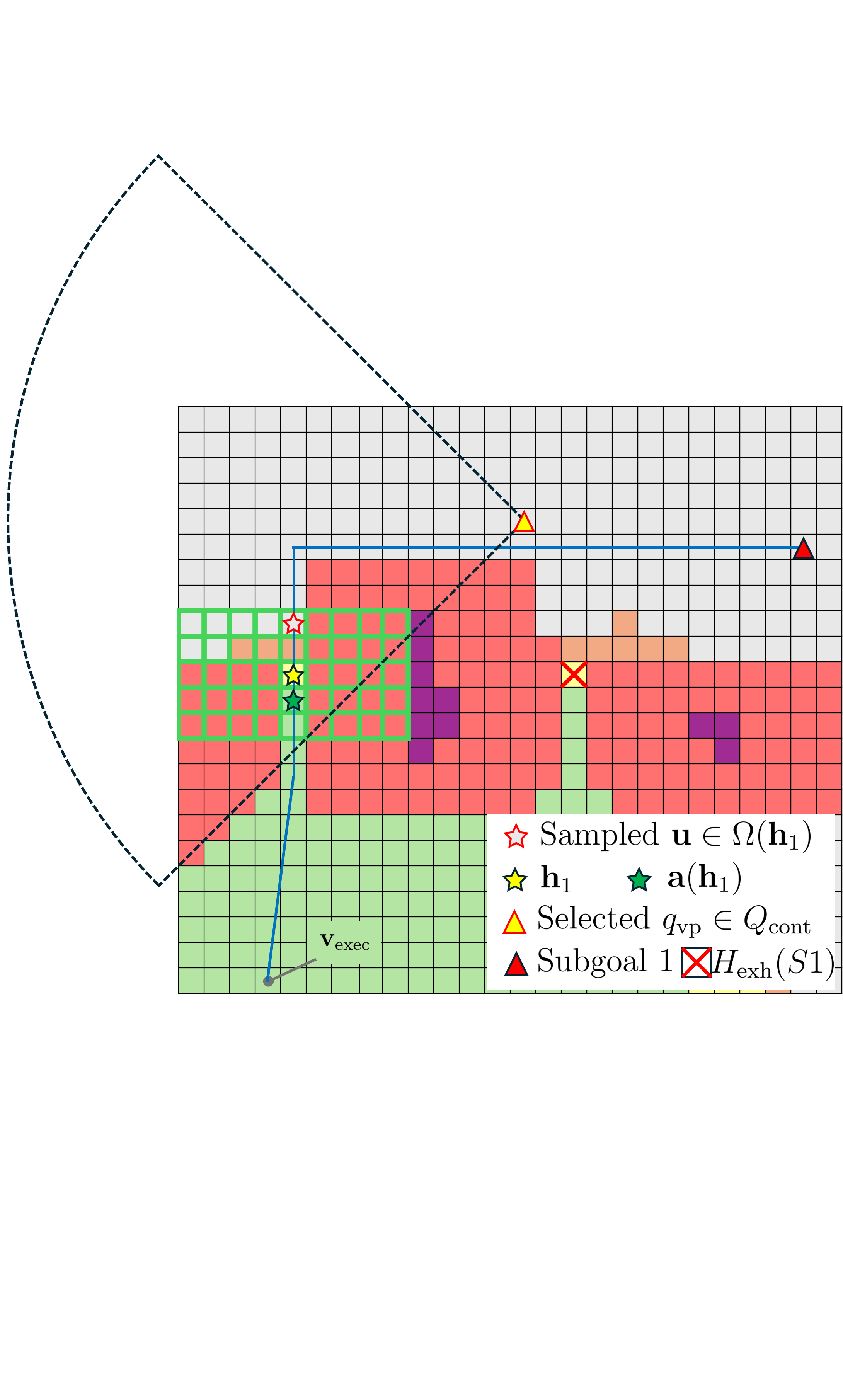}
        \caption{}
        \label{fig:recursive_2}
    \end{subfigure}
    
    \vspace{0.2cm}
    
    \begin{subfigure}[t]{0.48\columnwidth}
        \centering
        \includegraphics[trim={3.5cm 8.0cm 0.0cm 8.0cm},clip,width=\linewidth]{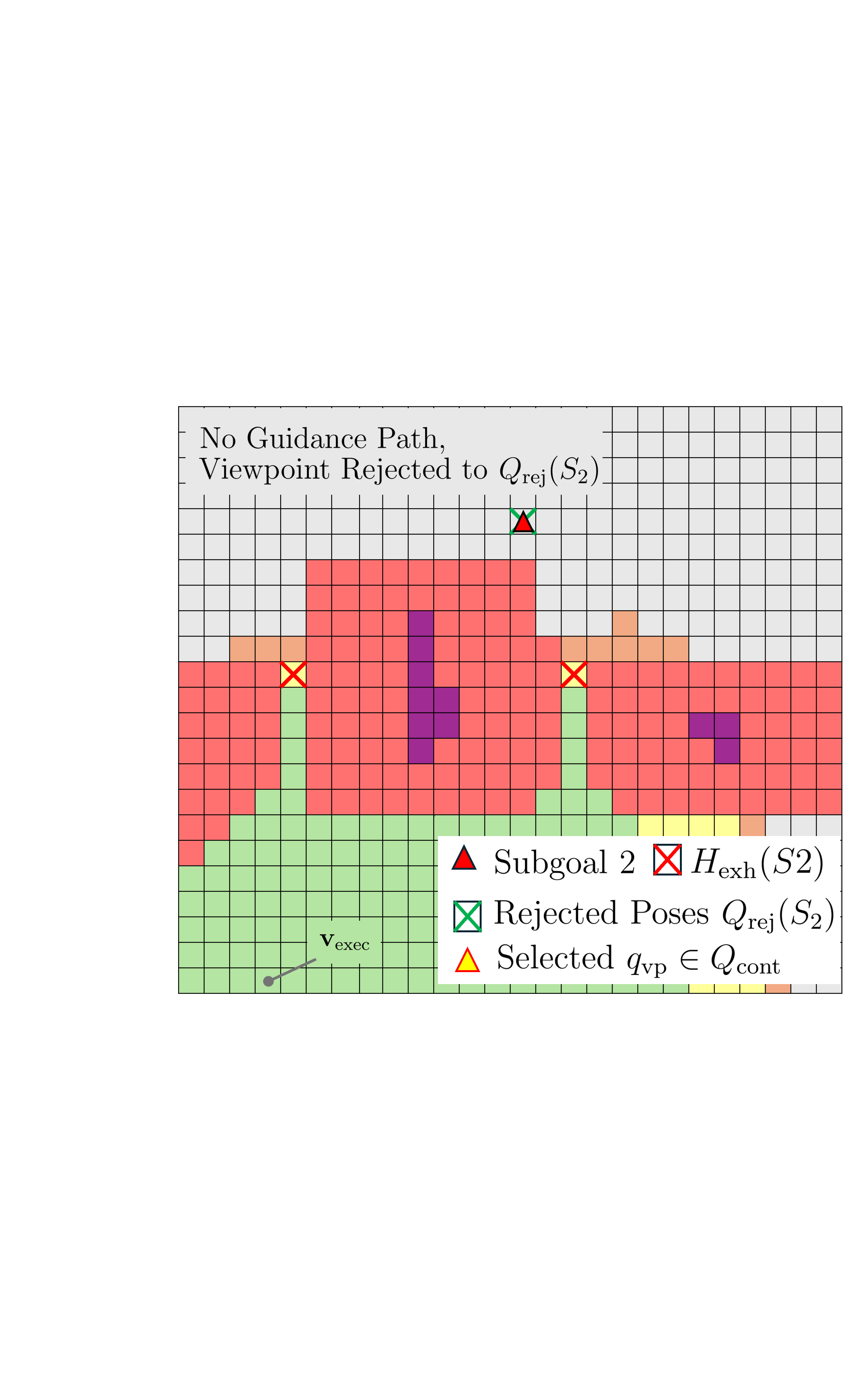}
        \caption{}
        \label{fig:recursive_3}
    \end{subfigure}
    \hfill
    \begin{subfigure}[t]{0.48\columnwidth}
        \centering
        \includegraphics[trim={3.5cm 8.0cm 0.0cm 8.0cm},clip,width=\linewidth]{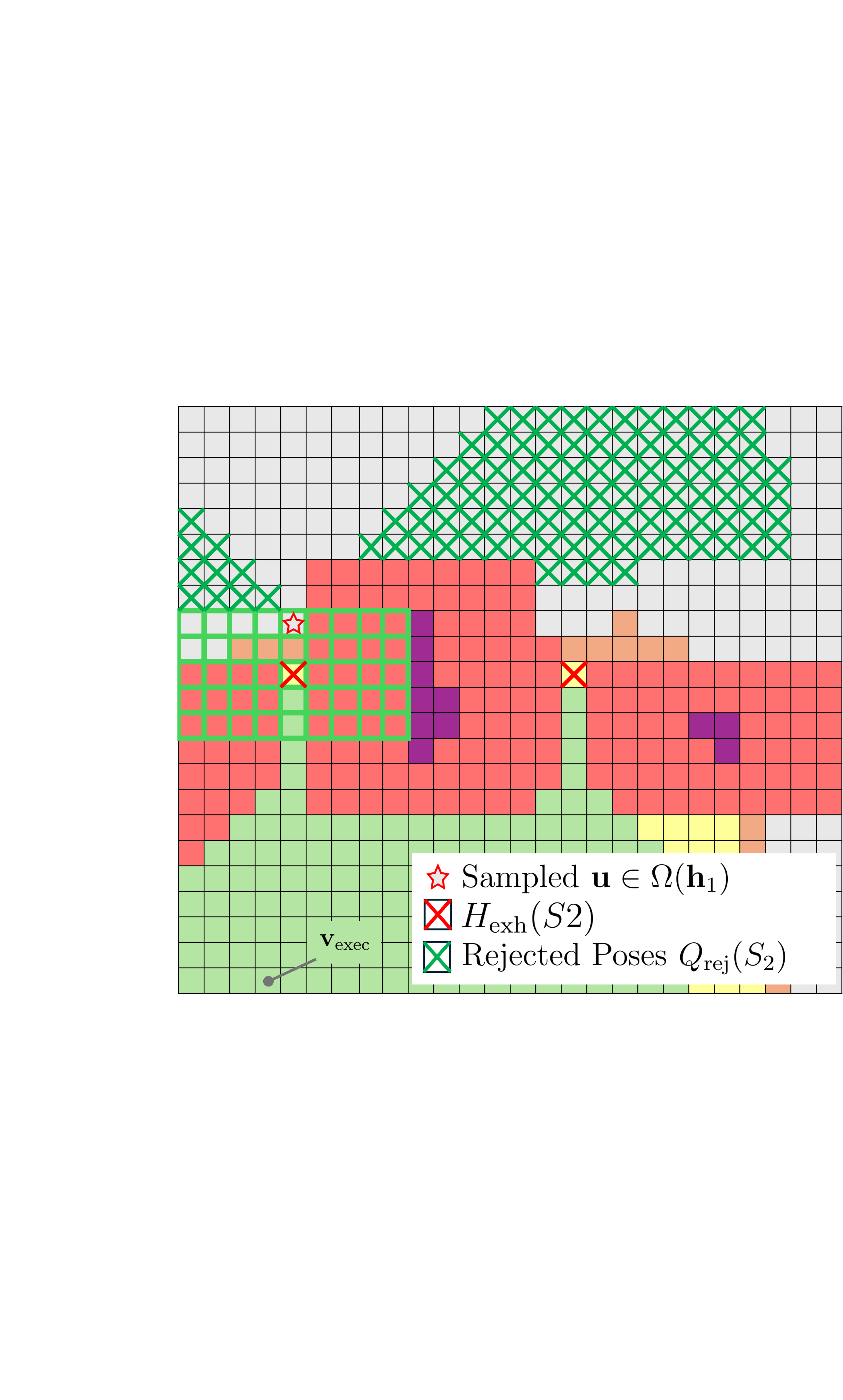}
        \caption{}
        \label{fig:recursive_4}
    \end{subfigure}

    \vspace{0.2cm}
    
    \begin{subfigure}[t]{0.48\columnwidth}
        \centering
        \includegraphics[trim={3.5cm 8.0cm 0.0cm 8.0cm},clip,width=\linewidth]{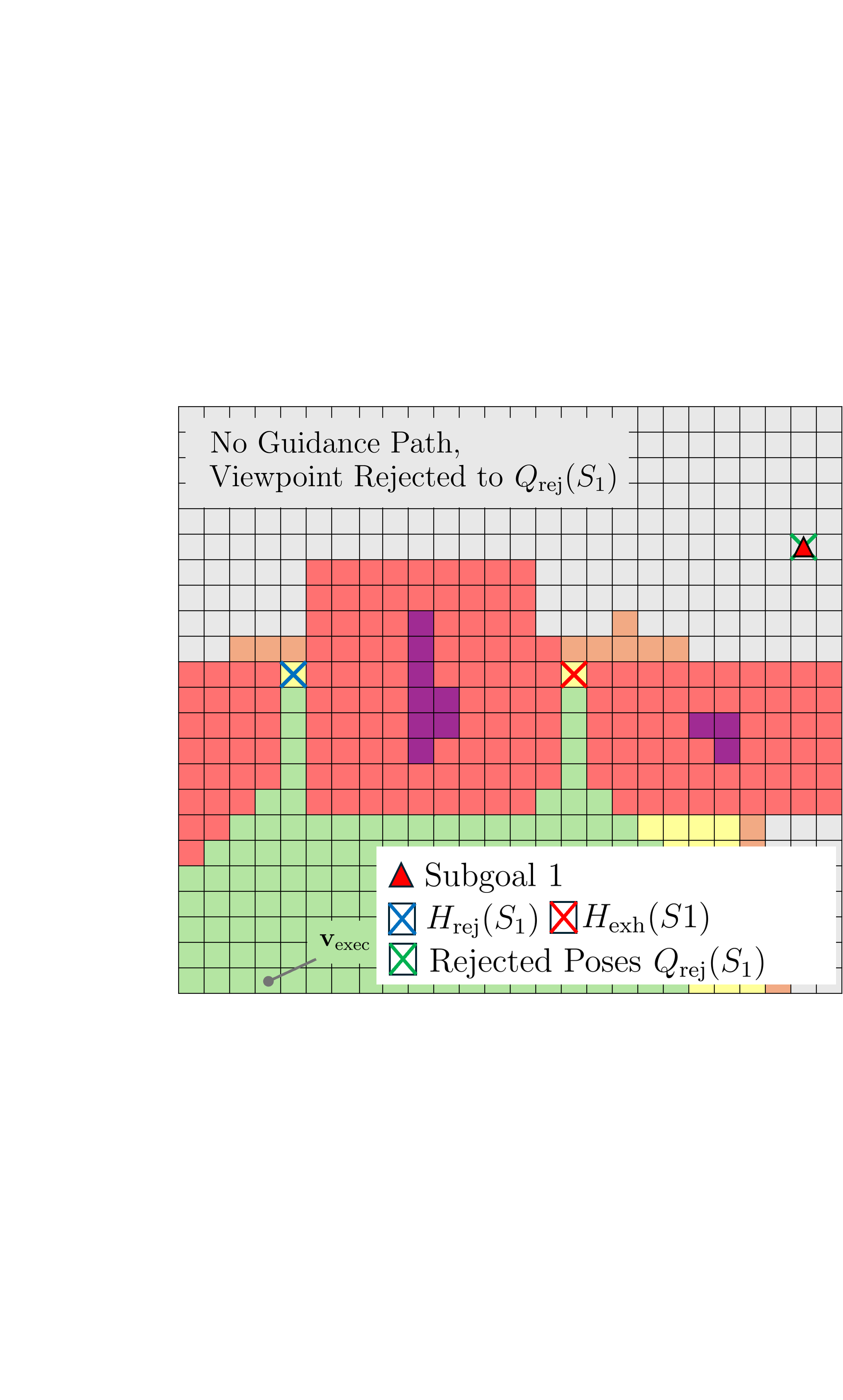}
        \caption{}
        \label{fig:recursive_5}
    \end{subfigure}
    \hfill
    \begin{subfigure}[t]{0.48\columnwidth}
        \centering
        \includegraphics[trim={3.5cm 8.0cm 0.0cm 8.0cm},clip,width=\linewidth]{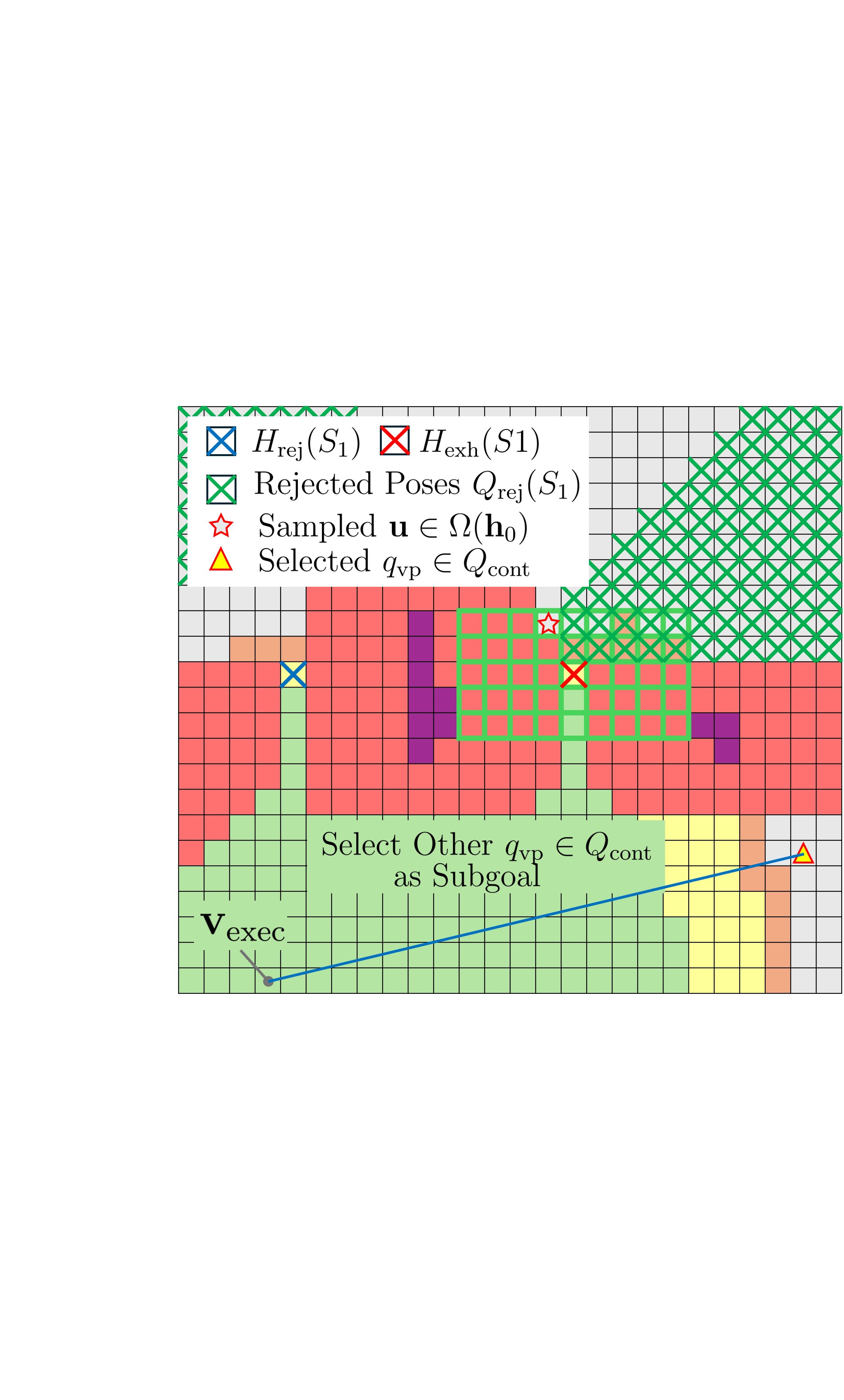}
        \caption{}
        \label{fig:recursive_6}
    \end{subfigure}

    \caption{An example of how the exclusion structure helps reject undesirable selection of viewpoints.
    Consider the case after a failed volume clearing attempt: The planner samples an unknown voxel in
    $\Omega(\mathbf{h}_0)$ and searches for viewpoints that can observe it.
    Without any certified candidate, there are multiple contaminated candidates that
    can possibly be the next subgoal, with the first level context $S_1$~(\ref{fig:recursive_1}).
    The planner selects the upper one, turns it into a subgoal (Subgoal 1),
    and generates a guidance path but with the $\mathbf{h}_0$ blocked by adding it to $H_{\mathrm{anc}}(S_1)$,
    and this guidance path has a new hitpoint $\mathbf{h}_1$.
    With no possible volume clearing viewpoint, the planner again samples an unknown voxel
    in $\Omega(\mathbf{h}_1)$ and searches for viewpoints that can observe it.
    Without any certified candidate, the planner now has a second level context $S_2$
    and selects a contaminated viewpoint as the next subgoal (Subgoal 2)~(\ref{fig:recursive_2}).
    The planner then attempts to generate a guidance path to Subgoal 2, similarly as before but with
    another parent hitpoint $\mathbf{h}_1$ added to $H_{\mathrm{anc}}(S_2)$. However, with the current $H_{\mathrm{anc}}(S_2)$ blocking,
    no guidance path can be found, and therefore Subgoal 2 as a candidate viewpoint is
    rejected and added to $Q_\mathrm{rej}(S_2)$~(\ref{fig:recursive_3}). Similarly, all other
    candidates are rejected, meaning that the sampled target in $\Omega(\mathbf{h}_1)$
    cannot be observed by any admissible candidate under the current context. Thus, the corresponding
    $\mathbf{h}_1$ is not clearable in this context and is added to $H_\mathrm{exh}(S_1)$~(\ref{fig:recursive_4}).
    Back to the parent context $S_1$, now with the block of $H_{\mathrm{anc}}(S_1)$ and $H_\mathrm{exh}(S_1)$,
    there is no guidance available to the upper viewpoint selected before~(\ref{fig:recursive_5}).
    In this way, the planner is enabled to deny infeasible viewpoint selections, and,
    as characterized by the completeness analysis in Sec.~\ref{sec:theory}, a feasible
    viewpoint within the generated candidate family is eventually selected
    if one exists~(\ref{fig:recursive_6}).}
    \label{fig:recursive_subgoal_generation}
    \vspace{-0.3cm}
\end{figure}

\subsection{Complete Recursive Search Procedure}
\label{sec:planning-loop}

The previous subsections define the three ingredients of the search planner: target-centric viewpoint candidates, certified connection to
executable viewpoints, and recursive subgoals for contaminated viewpoints with
context-local rejection memory. We now collect these pieces into the compact
planning loop shared by the search planner and the runtime planner (Alg. \ref{alg:runtime_pipeline}).
We use the following shorthand:
$\Guide$, $\Cand$, and $\Conn$ denote the optimistic guidance, view candidates generation, and multi-source connection subroutines
described in~\ref{sec:global-guidance}, \ref{sec:target-centric}, and \ref{sec:connect},
respectively. The operators $\Hit$ and
$\Target$ denote the ray-marching hitpoint extraction and observation-target
selection defined in Sections~\ref{sec:global-guidance} and
\ref{sec:target-centric}. 
In the case that no voxel on
$\Pi_g$ violates $G^{\mathrm{cert}}$, $\Hit(\Pi_g)$ returns none, and the target is certified-reachable.

Line \ref{ln:return} is the search planner return: only by finding and connecting to a certified
sensing pose, the search planner returns a certified path.
The shorthand $\Traj(\Pi_1;q_1,\mathbf{h})$ denotes a certified sensing action:
the robot follows the certified path $\Pi_1$ to the sensing pose $q_1$ and
executes the current sensing target associated with the active hitpoint
$\mathbf{h}$.
Lines \ref{ln:prev-start}--\ref{ln:prev-end} are runtime-only preview logic
to be introduced in Sec.~\ref{sec:receding-preview}: when a certified next
viewpoint is available, the runtime may instead return the two-step
trajectory $\Traj(\Pi_1\circ\Pi_{12})$.

\begin{algorithm}[t]
\caption{Runtime Planner}
\label{alg:runtime_pipeline}
\Input{Context $S$; execution start $\mathbf{v}_{\mathrm{exec}}$; guidance target $\mathbf{v}_{g}$}
\Output{Certified trajectory/action, terminal path, or \texttt{FAIL}}
\BlankLine
\While{true}{
    $\Pi_g \gets \Guide(\mathbf{v}_{\mathrm{exec}},\mathbf{v}_{g};B(S))$\;
    \If{$\Pi_g=\texttt{FAIL}$}{
        \Return \texttt{FAIL}\;
    }

    $\mathbf{h}\gets\Hit(\Pi_g)$\;
    \If{$\mathbf{h}=\texttt{NONE}$}{
        \Return $\Traj(\Pi_g)$\;
    }

    $\tau\gets\Target(\mathbf{h};S)$\;

    \While{true}{
        $(Q_{\mathrm{cert}},Q_{\mathrm{cont}})
        \gets \Cand(\mathbf{h},\tau;Q_{\mathrm{rej}}(S))$\;

        \If{$Q_{\mathrm{cert}}\neq\emptyset$}{
            $(\Pi_1,q_1)\gets\Conn(\mathbf{v}_{\mathrm{exec}},Q_{\mathrm{cert}})$\;

            \If{$\tau$ is a volume-clearing target\label{ln:prev-start}}{
                $(\Pi_{12},q_2,\mathbf{g}_2)\gets
                \Prev(q_1,\mathbf{g}_1,\Pi_g)$\;
                \If{$\Pi_{12}\neq\texttt{NONE}$}{
                    \Return $\Traj(\Pi_1\circ\Pi_{12})$\;\label{ln:prev-end}
                }
            }

            \Return $\Traj(\Pi_1;q_1,\mathbf{h})$\;\label{ln:return}
        }

        \ForEach{$q=(\mathbf{v}_q,\psi_q)\in Q_{\mathrm{cont}}$\label{ln:recurse-start}}{
            $S_+\gets\Child(S,q,\mathbf{h},\tau)$\;
            $r\gets\RuntimePlanner(S_+,\mathbf{v}_{\mathrm{exec}},\mathbf{v}_q)$\;

            \If{$r\neq\texttt{FAIL}$}{
                \Return $r$\;
            }

            $Q_{\mathrm{rej}}(S)\gets Q_{\mathrm{rej}}(S)\cup\{\mathbf{v}_q\}$\;
            $\PropagateNoGood(S_+,S)$\;\label{ln:recurse-end}
        }

        \If{$\tau$ is a volume-clearing target}{
            $\mathbf{u}\gets\mathrm{sample}(\Omega(\mathbf{h}))$\;
            $\tau\gets\Obs(\mathbf{u})$\;
            \Continue\;
        }

        $H_{\mathrm{exh}}(S)\gets H_{\mathrm{exh}}(S)\cup\{\mathbf{h}\}$\;
        \Break\;
    }
}
\end{algorithm}

When no certified viewpoints are available, the planners enter the recursive call (line \ref{ln:recurse-start}-\ref{ln:recurse-end}).
The constructor $\Child(S,q,\mathbf{h},\tau)$ creates the child
context for subgoal $q$: it sets $H_{\mathrm{anc}}(S_+)=H_{\mathrm{anc}}(S)\cup\{\mathbf{h}\}$ and inherits
the rejection memory of $S$ (Sec.~\ref{sec:subgoal}).
The operation $\PropagateNoGood(S_+,S)$ propagates the context-dependent
exclusion sets from the child to the parent whenever their dependency hitpoints
are contained in the updated parent exclusion set
(Corollary~\ref{cor:nogood-reuse}).
The recursive calls use the context-local
exclusion memory defined in Sec.~\ref{sec:subgoal}; finite-runtime timeouts are
treated as incomplete planning ticks rather than rejection or exhaustion.

The transition from volume clearing to unknown observation in
Sec.~\ref{sec:target-centric} is triggered by two distinct events. At
planning time, Alg.~\ref{alg:runtime_pipeline} falls back to an
unknown-observation target when the clearing target has no viewpoint in
$Q_{\mathrm{cert}}$ and every viewpoint in $Q_{\mathrm{cont}}$ has been
rejected as a subgoal. At execution time, if an executed volume-clearing
action leaves the hitpoint uncleared, the context switches the observation
target of its active hitpoint to $\Obs(\mathbf{u})$ with
$\mathbf{u}\in\Omega(\mathbf{h})$, and volume clearing for $\mathbf{h}$ is
not retried until new observations change the map belief. This
\emph{execution-time fallback rule} is part of the analyzed search planner in
Appendix~\ref{app:completeness}: it guarantees that the planned sensing actions
induce a map update (Lemma~\ref{lem:finite-exec}). The preview runtime of
Sec.~\ref{sec:receding-preview} relaxes only the trigger timing while the
hitpoint continues to advance; it does not modify the recursive context state.

\subsection{Preview Viewpoints for Smooth Execution}
\label{sec:receding-preview}

The core search planner returns a certified sensing action
$\Traj(\Pi_1;q_1,\mathbf{h})$ as soon as it finds a reachable viewpoint
$q_1$. Treating every such viewpoint as a terminal stop can produce
stop--check--replan motion. This conservatism is useful for an
unknown-observation target, which may require a precise view of one voxel, but
is often unnecessary for volume clearing. We therefore apply a certified
preview layer only to volume-clearing actions so that the robot can continue
through the first viewpoint while preparing the next observation.

Let $q_1$ be the current \emph{anchor viewpoint} and let $\mathbf{g}_1$ denote
the guidance point it is expected to clear; initially,
$\mathbf{g}_1=\mathbf{h}$ (Fig. \ref{fig:certified_preview}). The preview routine looks farther along the current
guidance path $\Pi_g$, selects a lookahead point $\mathbf{v}_{\mathrm{adv}}$,
and runs a certified-only version of the target-centric viewpoint search. Each
candidate is evaluated by the farthest point on $\Pi_g$ that it can clear, and
the planner selects a connectable viewpoint $q_2$ that maximizes this forward
progress. We denote the result by
$(\Pi_{12},q_2,\mathbf{g}_2)=\Prev(q_1,\mathbf{g}_1,\Pi_g)$, where
$\Pi_{12}\subset G^{\mathrm{cert}}$ connects $q_1$ to $q_2$. If no
preview viewpoint or connection is found within the runtime budget, the planner
simply executes the original single-viewpoint action.

When preview succeeds, the certified path
$\Pi_{\mathrm{prev}}=\Pi_1\circ\Pi_{12}$ contains $q_1$ as an intermediate
observation anchor and terminates at the preview viewpoint $q_2$. The trajectory
optimizer of Sec.~\ref{sec:visibility-cell-anchor} preserves the observation
requirement associated with $q_1$ while allowing nonzero passage through its
local visibility region; $q_2$ supplies the terminal viewpoint for the current
trajectory.

The construction is receding. When the robot reaches the current anchor,
$q_2$ is promoted to the next anchor and $\mathbf{g}_2$ becomes its expected
clear point; the planner then searches for another preview beyond it. 


\begin{figure}[t]
    \centering
    \begin{subfigure}[t]{0.49\columnwidth}
        \centering
        \includegraphics[trim={3.53cm 8.08cm 5.05cm 6.55cm},clip,width=\linewidth]{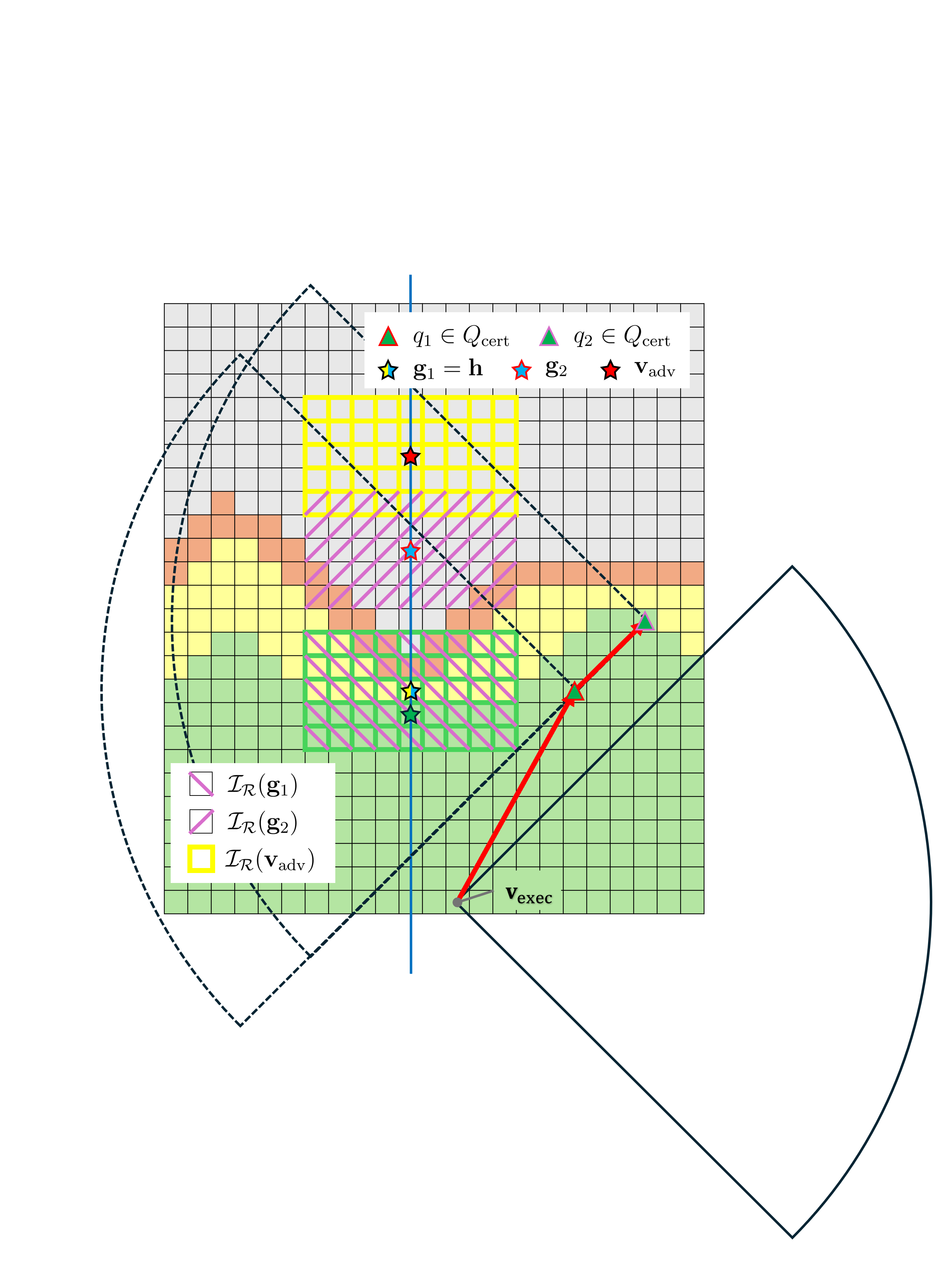}
        \caption{Certified preview.}
        \label{fig:preview_1}
    \end{subfigure}
    \hfill
    \begin{subfigure}[t]{0.49\columnwidth}
        \centering
        \includegraphics[trim={3.53cm 8.08cm 5.05cm 6.55cm},clip,width=\linewidth]{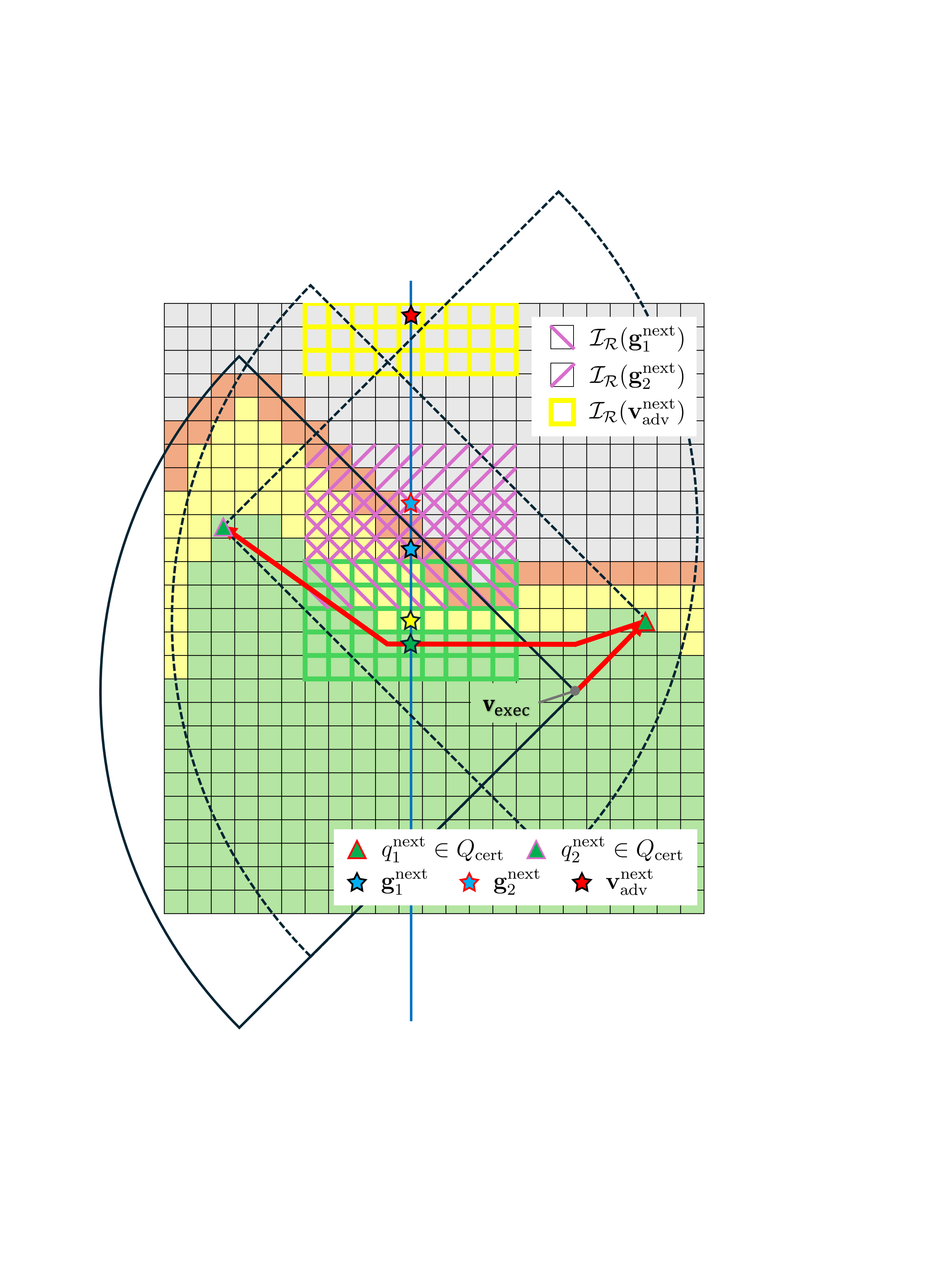}
        \caption{Receding promotion.}
        \label{fig:preview_2}
    \end{subfigure}

    \caption{Certified preview routine for smooth execution.
    From the current anchor $q_1$ with expected clear point
    $\mathbf{g}_1$, the preview routine searches for a farther
    certified viewpoint that tries to clear $\mathbf{v}_\mathrm{adv}$
    , and end up with $q_2$ that is expected to clear $\mathbf{g}_2$
    as the farthest point~(\ref{fig:preview_1}).
    After passing the anchor $q_1$, $q_2$ is promoted to new anchor
    $q_1^{\mathrm{next}}$, and $\mathbf{g}_2$ becomes
    $\mathbf{g}_1^{\mathrm{next}}$. Since the hitpoint is pushed forward,
    the planner then target on the next $\mathbf{v}_\mathrm{adv}^\mathrm{next}$ pushed from
    $\mathbf{g}_1^{\mathrm{next}}$ instead of sticking to the
    updated $\mathbf{h}$ with lower progress~(\ref{fig:preview_2}).}
    \label{fig:certified_preview}
    \vspace{-0.3cm}
\end{figure}

\section{Trajectory Optimization with Observation Constraints}
\label{sec:traj_opt}

The search planner produces a discrete certified route composed of grid-based path segments.
Although this path is sufficient for topological reasoning about safety and visibility,
it is not directly executable by a quadrotor since the path is
piecewise-linear, nonsmooth, and lacks timing and dynamic feasibility.
We therefore adopt a trajectory optimization backend to convert the discrete path to a smooth,
time-parameterized trajectory within safe flight corridors.

\subsection{GCOPTER Backbone}
We adopt GCOPTER~\cite{WANG2022GCOPTER} as the basic framework for trajectory optimization.
Given a discrete certified voxel path
\begin{equation}
    \Pi = (\mathbf{v}_0,\mathbf{v}_1,\ldots,\mathbf{v}_N),
\end{equation}
GCOPTER first constructs a polyhedral safe flight corridor (SFC) using 
Fast Iterative Region Inflation (FIRI)~\cite{FIRI} 
as a convex abstraction of the local free
configuration space. The SFC is represented as sequentially connected convex polytopes
\begin{equation}
    \mathcal{C}
    =
    (\mathcal{P}_1,\mathcal{P}_2,\ldots,\mathcal{P}_{M}),
    \quad
    \mathcal{P}_m
    =
    \{\mathbf{x}\in\mathbb{R}^3
    \mid
    \mathbf{A}_m\mathbf{x}\leq \mathbf{b}_m\}.
\end{equation}
In our implementation, the SFC is constructed against $G^{\mathrm{cert}}$
around the searched path. The realized polytopes may still
graze non-certified cells near voxel boundaries; Sec.~\ref{sec:theory}
delimits what this \emph{corridor-realization slack} means for the
certification guarantee, and the benchmark of Sec.~\ref{sec:experiments}
measures it.

Following the safety corridor extraction, the GCOPTER pipeline assigns
trajectory pieces to the safe flight corridor cells and optimizes a
spatial-temporal parameterization of the trajectory using MINCO basis. 
Let
$\mathbf z_i=(\mathbf x_i,\psi_i)$ denote a position--yaw control knot and let
$T_i$ be the corresponding segment duration. For a route terminating at a
goal or sensing viewpoint, the initial and terminal boundary conditions are
\begin{equation}
\begin{aligned}
    \mathbf x^{[2]}(0)
    &=
    \mathbf x^{[2]}_{\mathrm{exec}},
    &
    \mathbf x^{[2]}(T_\Sigma)
    &=
    \mathbf x^{[2]}_f,\\
    \psi_0
    &=
    \psi_{\mathrm{exec}},
    &
    \psi_{M_t}
    &=
    \psi_f,
\end{aligned}
\label{eq:terminal_boundary}
\end{equation}
where
$\mathbf x^{[2]}=(\mathbf x,\dot{\mathbf x},\ddot{\mathbf x})$,
$T_\Sigma=\sum_i T_i$, and the terminal velocity and acceleration are set to
zero. For a sensing viewpoint
$q_{\mathrm{vp}}=(\mathbf v_{\mathrm{vp}},\psi_{\mathrm{vp}})$, we use
\[
    \mathbf x_f
    =
    \operatorname{center}(\mathbf v_{\mathrm{vp}}),
    \qquad
    \psi_f
    =
    \psi_{\mathrm{vp}}.
\]

We use the standard trajectory optimization objective $J_{\mathrm{traj}}$, which balances trajectory
smoothness, control effort, and duration. Additionally, we include a duration-aware yaw
regularization. With
$\Delta\psi_i=\psi_{i+1}-\psi_i$ for unwrapped yaw knots, we use
\begin{equation}
\begin{aligned}
    J(\mathbf z,T)
    &=
    J_{\mathrm{traj}}(\mathbf z,T)
    +
    J_{\mathrm{yaw}}(\mathbf z,T),\\
    J_{\mathrm{yaw}}
    &=
    w_{\dot\psi,\mathrm{reg}}
    \sum_i
    \frac{\Delta\psi_i^2}{T_i}\\
    &\quad+
    w_{\dot\psi,\mathrm{lim}}
    \sum_i
    T_i
    L_\mu
    \left(
        \frac{\Delta\psi_i^2}{T_i^2}
        -
        \dot\psi_{\max}^2
    \right),
\end{aligned}
\label{eq:terminal_viewpoint_obj}
\end{equation}
where $L_\mu$ is a differentiable approximation of the positive-part
function. The first term smooths the yaw profile, while the second discourages
yaw-rate violations. Because $T_i$ is optimized jointly with the trajectory,
the optimizer may allocate more time to segments requiring larger rotations.

\subsection{Visibility-Polytope Anchor Optimization}
\label{sec:visibility-cell-anchor}

As described in Sec.~\ref{sec:receding-preview}, the preview planner may produce
a composite path that contains an intermediate observation anchor $q_1$ and
continues to a preview viewpoint $q_2$. Compared with the simple start-to-terminal-viewpoint case,
the optimizer must satisfy not only the terminal boundary condition
$(\mathbf{x}_f,\psi_f)$ associated with $q_2$, but also an intermediate observation requirement
associated with the anchor viewpoint $q_1$.
Forcing the trajectory through the
exact position of $q_1$ is unnecessarily restrictive because $q_1$ is only
one witness pose satisfying the observation requirement. Conversely, a soft
point-attraction cost may move the optimized knot to a pose from which the
target is no longer observable. We therefore replace the point anchor by a
local, sample-validated visibility polytope.
The optimizer is allowed to choose where to pass through this polytope, while the
polytope is validated to preserve the observation requirement.

\subsubsection{Visibility-polytope extraction}

Let $\tau$ be the observation target associated with
$q_1=(\mathbf v_1,\psi_1)$. Within a local voxel region, we form the
6-connected component containing $\mathbf v_1$ of certified voxels whose
target-facing poses satisfy the observation relation:
\begin{equation}
    V_{\mathrm{vis}}(\tau)
    =
    \operatorname{Comp}^{6}_{\mathbf v_1}
    \left(
        \left\{
            \mathbf v\in G^{\mathrm{cert}}
            \mid
            q(\mathbf v;\tau)\models\tau
        \right\}
    \right).
    \label{eq:visibility_voxel_cluster}
\end{equation}
We erode this component by a fixed number of voxel layers to provide a margin
against discretization and boundary-visibility errors. An interior voxel with
maximum graph distance from the eroded boundary is selected as the seed,
and the neighboring excluded voxels are converted to obstacle samples. 
We use FIRI to inflate a convex polytope from the seed: 
\begin{equation}
    \mathcal P_{\mathrm{vis}}
    =
    \left\{
        \mathbf x\in\mathbb R^3
        \mid
        \mathbf A_{\mathrm{vis}}\mathbf x
        \leq
        \mathbf b_{\mathrm{vis}}
    \right\}.
    \label{eq:visibility_polytope}
\end{equation}

Finally, we validate the resulting polytope before using it in trajectory
optimization.
The polytope must exceed a prescribed Chebyshev-radius threshold,
and the seed, Chebyshev center, vertices, and vertex--seed midpoints must
satisfy the same observation test.  This validation step bridges the gap between the voxel-center visibility tests used during region growing and the continuous positions that may be selected by
the optimizer inside $\mathcal{P}_{\mathrm{vis}}$. 

\subsubsection{Trajectory optimization through the visibility polytope}

For the preview trajectory, we construct motion-corridor sequences
$\mathcal C_1$ and $\mathcal C_{12}$ around $\Pi_1$ and $\Pi_{12}$,
respectively, and insert the visibility polytope between them:
\begin{equation}
    \mathcal C_{\mathrm{vis}}
    =
    \left(
        \mathcal C_1,
        \mathcal P_{\mathrm{vis}},
        \mathcal C_{12}
    \right).
    \label{eq:visibility_corridor_seq}
\end{equation}
A connector polytope is added when necessary to maintain overlap between
adjacent corridor elements. The visibility polytope is kept as a designated
corridor element and assigned at least two polynomial pieces, which creates an
internal observation knot $k_{\mathrm{vis}}$. This knot is assigned to
$\mathcal P_{\mathrm{vis}}$, with deviation penalized by the same corridor
cost, so the optimizer gains local geometric freedom without being tied to
the original witness position.

Let
$\mathbf c_{\mathrm{vis}}$
be the FIRI seed and
$\mathbf p_\tau$
be the representative target point. The target-facing yaw reference is
\begin{equation}
    \psi_{\mathrm{vis}}
    =
    \operatorname{atan2}
    \left(
        p_{\tau,y}-c_{\mathrm{vis},y},
        p_{\tau,x}-c_{\mathrm{vis},x}
    \right)
    -
    \psi_{\mathrm{off}},
\end{equation}
and the observation-knot alignment cost is
\begin{equation}
    J_{\psi,\mathrm{vis}}
    =
    w_{\psi,\mathrm{vis}}
    \left(
        \psi_{k_{\mathrm{vis}}}
        -
        \psi_{\mathrm{vis}}
    \right)^2.
    \label{eq:visibility_yaw_cost}
\end{equation}
The preview objective is therefore
\begin{equation}
    J_{\mathrm{vis}}
    =
    J_{\mathrm{gco}}
    +
    J_{\mathrm{yaw}}
    +
    J_{\psi,\mathrm{vis}}.
    \label{eq:visibility_polytope_obj}
\end{equation}
The terminal state is induced by $q_2$, while the internal knot provides the
intermediate observation opportunity associated with $q_1$.

If the visibility polytope cannot be constructed, cannot be connected to the
adjacent motion corridors, or leads to an unsuccessful optimization, we use a
seam-stitched fallback. The entry and preview corridors are concatenated, and
the searched anchor is encouraged using high-weight position and yaw seam
penalties.

\section{Theoretical Properties}
\label{sec:theory}

This section first analyzes the recursive search planner and then the safety
of executing its certified output. The first result is a conditional
completeness theorem: under an idealized model, recursive viewpoint subgoals
and exclusion memory do not eliminate a finite feasible sequence of sensing
actions satisfying Assumption~(A1). The second result states the geometric condition under which an
executed trajectory inherits the safety-volume certification of the discrete
search output.


\subsection{Conditional Completeness of the Search Planner}
\label{sec:theory-completeness}

The recursive planner may reject contaminated viewpoints and mark hitpoints as
exhausted. The central theoretical question is therefore whether these
exclusions can accidentally remove a viewpoint or hitpoint needed by an
otherwise feasible planning sequence. 

A \emph{fixed-map episode} is an interval during which the planner-relevant
map belief, inflation sets, visibility tests, and candidate predicates do not
change. The workspace is bounded, so every graph searched during an episode
is finite. 
The completeness result uses the following two assumptions.

(A1) \textbf{Existence of a finite feasible sequence of sensing
actions.}
There exists a finite sequence of certified sensing actions that resolves
the hitpoints required to reach the goal. Each action consists of a certified
path to a sensing pose and the observation executed at that pose. 
After executing the
sequence and incorporating the resulting map updates, the remaining path to
the goal lies in \(G^{\mathrm{cert}}\).
The sequence is
used only as an existence witness in the proof and is not provided to the
planner. 

    (A2) \textbf{Ideal monotone sensing.}
    Sensor updates are consistent with a fixed ground-truth map and move
    voxels monotonically from \(U\) to either \(F\) or \(O\); no observed
    voxel returns to \(U\). 

The proof uses the following property of the exhaustive planner:

    (P1) \textbf{Complete finite-domain primitive searches.}
    With unlimited planning effort, \(\Guide\), \(\Cand\), and \(\Conn\)
    exhaust their relevant finite search domains. Thus, \(\Guide\) and
    \(\Conn\) return a valid path whenever one exists and report failure
    only after proving that no such path exists. The candidate search
    \(\Cand\) either finds a certified candidate or, when no certified
    candidate exists, exhaustively generates the contaminated candidate
    family associated with the current observation target. 

Property~(P1) follows from the graph-search implementations.
\(\Guide\) is an A* search over a finite subgraph of
\(G^{\mathrm{opt}}\), \(\Cand\) is a labeled graph search over a finite
augmented state space, and \(\Conn\) is a reverse multi-source A* search
over the finite graph \(G^{\mathrm{cert}}\). When allowed to exhaust their
OPEN sets, these searches satisfy Property~(P1).

Runtime budgets may interrupt this exhaustive process, but an interruption
is treated as an incomplete planning tick rather than as a failure
certificate. The theorem therefore concerns the unlimited budget of the planner.


\begin{theorem}[Conditional completeness]
\label{thm:completeness}
Under the fixed-map analysis setting and Assumptions~(A1)--(A2), the
exhaustive search planner of Algorithm~\ref{alg:runtime_pipeline}
satisfying Property~(P1) reaches the goal after finitely many fixed-map
episodes.
\end{theorem}
\begin{proof}
    See Appendix~\ref{app:completeness}.
\end{proof}
The theorem states that, if
a finite feasible sensing sequence satisfying Assumption~(A1) exists, the
planner's rejection mechanism cannot prevent goal reaching by pruning an
action that is required by the feasible sensing sequence.
Under the same conditions, a root-level \texttt{FAIL} certifies
that no finite goal-reaching sequence of sensing actions satisfying
Assumption~(A1) exists under the current belief.

In finite-runtime operation, the planner has three possible outcomes. It may
return a certified search action; it may return
\texttt{FAIL} after the relevant finite searches have been exhausted; or it
may return an incomplete planning tick when the current budget is exhausted.


\subsection{Certified Execution Safety}
\label{sec:theory-safety}

The completeness theorem concerns whether the search planner finds a
certified action. We next consider the geometric question of whether
the motion used to realize that action remains inside the region certified by
the map. The search-level containment chain is direct: a returned voxel path
lies in \(G^{\mathrm{cert}}\), and the continuous extension of
Corollary~\ref{cor:gcert-semantics} states that every robot position in the corresponding continuous
certified region has its complete inflated safety volume inside known-free
space.

\begin{proposition}[Certified motion execution]
\label{prop:soundness}
Let $\mathbf{x}:[0,T]\rightarrow\mathbb{R}^{3}$
be an executed trajectory. If its geometric support is contained in the
continuous certified region induced by the map belief at execution time,
then, at every time along the trajectory, the robot's entire inflated safety
volume lies in known-free space.
\end{proposition}

\begin{proof}
By Corollary~\ref{cor:gcert-semantics} and its continuous extension in
Sec.~\ref{sec:problem}, every robot-center position inside the continuous
certified region has its complete inflated safety volume contained in
known-free space. Applying this property to
\(\mathbf{x}(t)\) for every \(t\in[0,T]\) proves the claim.
\end{proof}

The discrete search path satisfies the premise exactly because it lies in
\(G^{\mathrm{cert}}\). For the continuous backend, safe-flight-corridor
penalties or corridor construction alone do not establish the premise if
they permit a nonzero containment violation.
Proposition~\ref{prop:soundness} therefore applies to the optimized
trajectory only when its realized support remains inside the continuous
certified region, for example after a hard post-optimization containment
check. Without such an execution gate, the backend is
certification-guided rather than formally certified, and
Sec.~\ref{sec:experiments} measures its corridor-realization slack
empirically.



\section{Experiments}
\label{sec:experiments}

The central focus of our evaluation is not
only whether a goal is successfully reached, but also how safely it is reached.
We explicitly investigate whether the spatial path produced by the search planner can consistently avoid non-certified regions throughout
the navigation.
First, we conduct a comprehensive benchmark comparing
representative baselines' search-level results, the search-level result of \PlannerName-search,
and the complete \PlannerNameSpaced system integrated with the SFC-based trajectory optimization backend.
Second, we ablate the preview routine and examine when the recursive subgoal mechanism is activated.
Third, we deploy the complete \PlannerNameSpaced system on a physical quadrotor platform,
demonstrating its feasibility in real-world environments.

\subsection{Experimental Setup}
\label{subsec:exp_setup}

\subsubsection{Simulation environments}
We evaluate four scene families (five variants) shown in
Fig.~\ref{fig:scenes}. Scene~1 contains two vertical-ascent variants
(1A/1B), each formed by two stacked $10\times10\times10$~m$^3$
compartments. Scene~1A uses a $3.0\times3.0$~m$^2$ ceiling aperture,
whereas Scene~1B uses a smaller $1.5\times1.5$~m$^2$ opening. The start lies below the
opening and the goal above it, requiring repeated observations to certify the
vertical safety volume. Scene~2 is a $5\times5$ array of
$5\times5\times5$~m$^3$ rooms connected by randomized
$1.5\times1.5$~m$^2$ wall openings; starts and goals are sampled in
distinct rooms, producing predominantly horizontal navigation. Scene~3 is a
four-story structure with $20\times20\times10$~m$^3$ floors and five
$1.5\times1.5$~m$^2$ apertures per level. Aligned beams with $2.5$~m
spacing occupy the second floor and randomized beams the third; tasks run from
the clear first floor to the clear fourth floor. Scene~4 is a
$30\times30\times4$~m$^3$ hall containing $80$ thin randomized door-frame
obstacles, each with two vertical poles and a $2.5$--$4.0$~m plate. The plate
modes are floor-, ceiling-, mid-air-, and full-height, sampled with ratios
$0.30/0.30/0.25/0.15$. All passable gaps are at least $1.5$~m; start and goal
positions are at least $20$~m apart, with $0.65$~m clearance at both endpoints. This dense,
thin-obstacle scene stresses omnidirectional occlusion handling.

For each scene variant, we evaluate $20$ paired randomized trials using the
same seeds for all methods. All planners share the same maps, robot geometry,
and sensing parameters: $0.1$~m voxels; box inflation
$r_{xy}=0.5$~m and $r_z=0.3$~m; and a forward-facing sensor with
$\pm45^\circ\times\pm45^\circ$ FOV and $5.0$~m maximum range, mounted
$0.05$~m above the robot center without relative rotation.

\begin{figure}[t]
    \centering
    \begin{minipage}[c]{0.42\linewidth}
        \centering
        \begin{subfigure}[t]{\linewidth}
            \centering
            \includegraphics[
                width=\linewidth,
                trim={0px 20px 0px 150px},
                clip
            ]{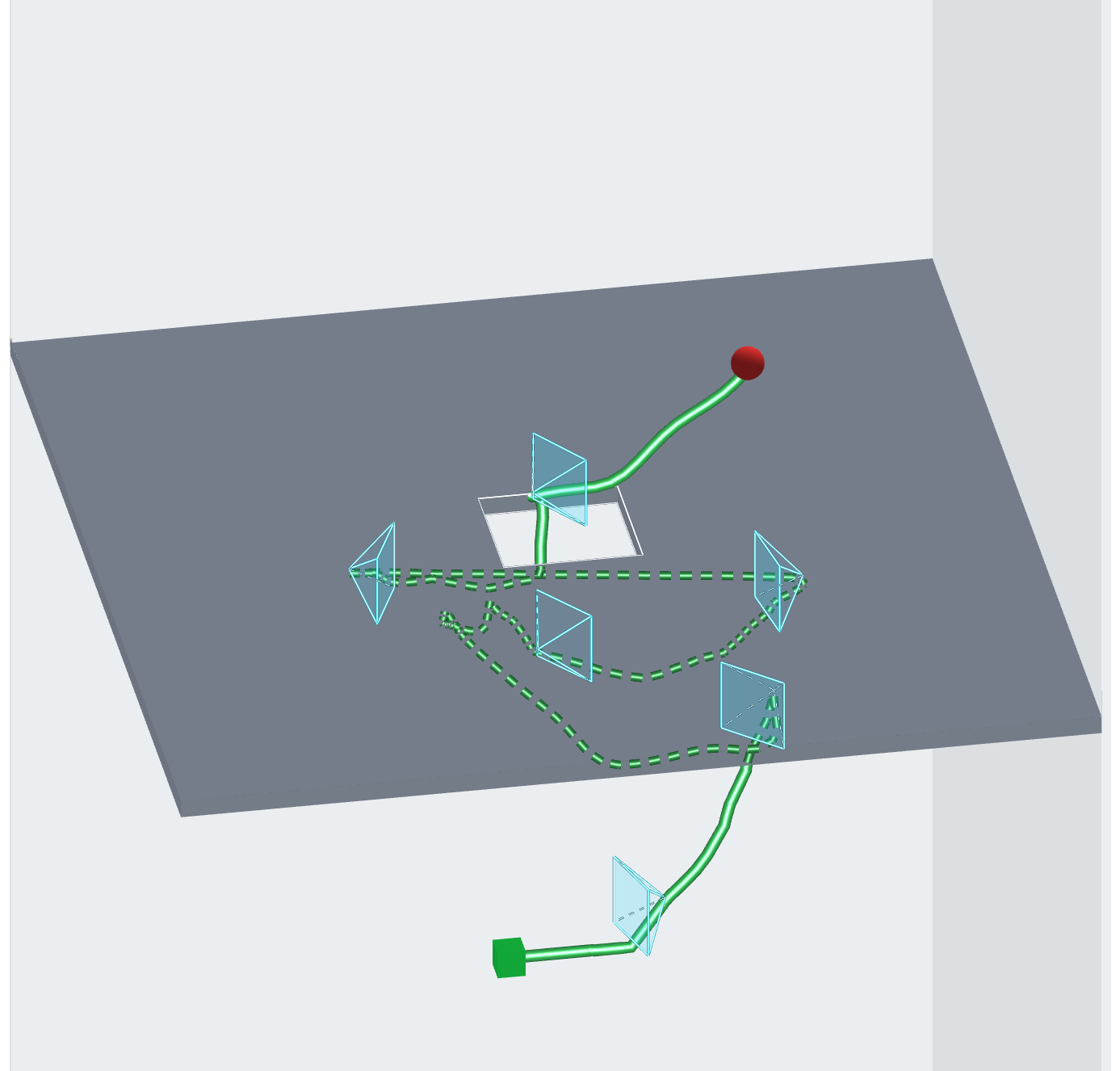}
            \caption{Scene 1B.}
            \label{fig:scene_1}
        \end{subfigure}

        \vspace{0.5em}

        \begin{subfigure}[t]{\linewidth}
            \centering
            \includegraphics[
                width=\linewidth,
                trim={0px 0px 0px 0px},
                clip
            ]{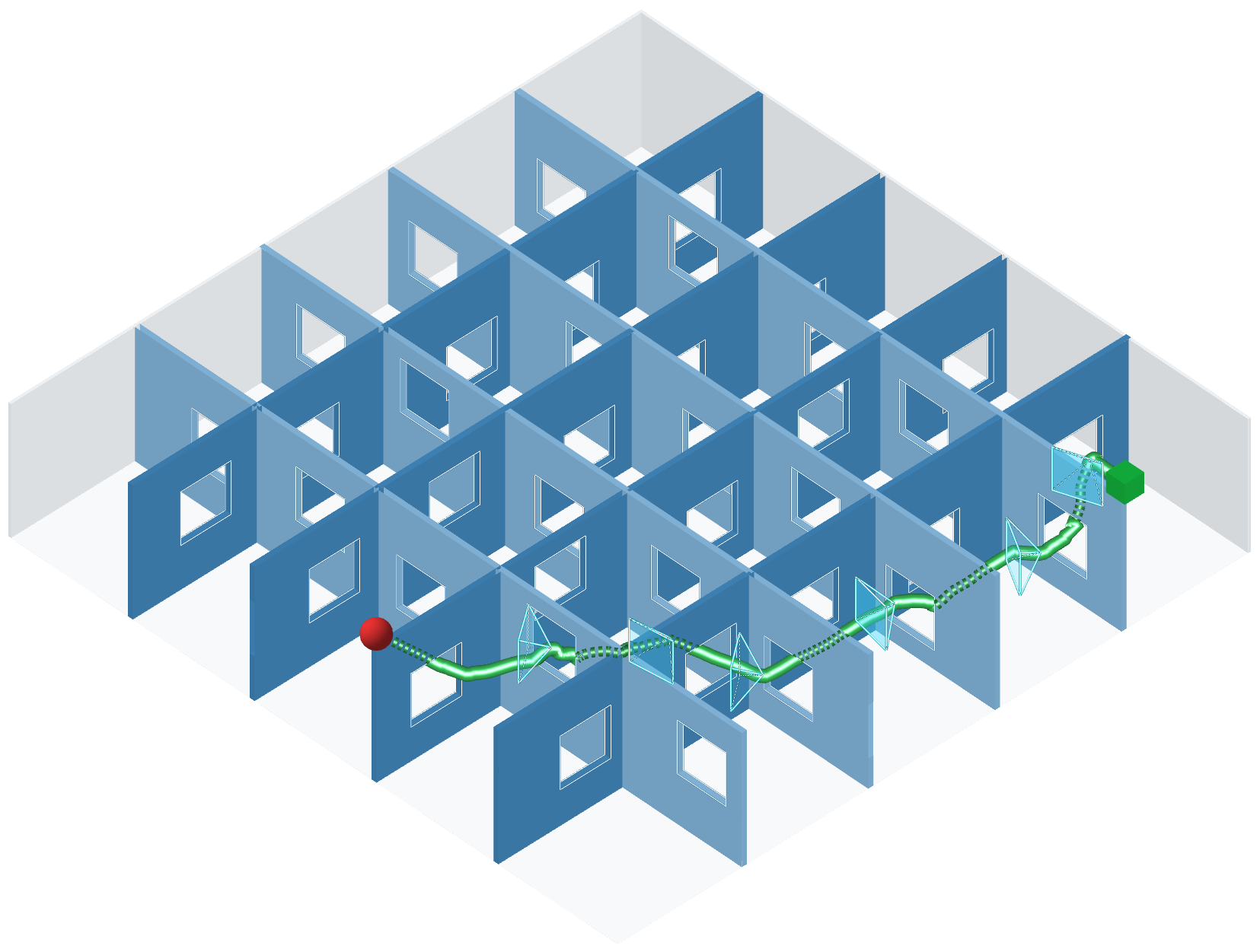}
            \caption{Scene 2.}
            \label{fig:scene_2}
        \end{subfigure}
    \end{minipage}
    \hfill
    \begin{minipage}[c]{0.54\linewidth}
        \centering
        \begin{subfigure}[t]{\linewidth}
            \centering
            \includegraphics[
                width=\linewidth,
                trim={20px 0px 20px 0px},
                clip
            ]{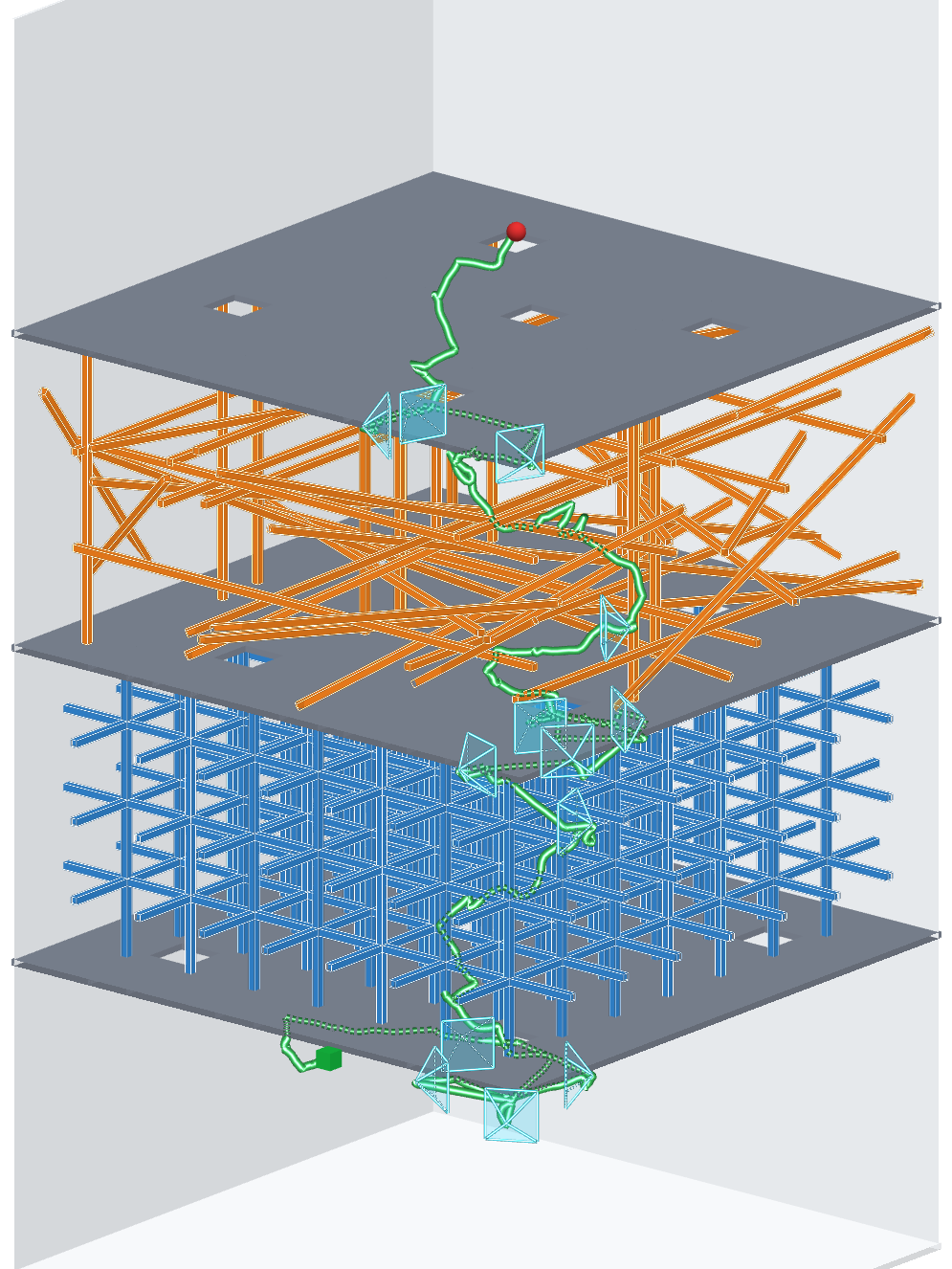}
            \caption{Scene 3.}
            \label{fig:scene_3}
        \end{subfigure}
    \end{minipage}

    \par
    \vspace{0.6em}

    \begin{subfigure}[t]{0.96\linewidth}
        \centering
        \includegraphics[
            width=\linewidth,
            trim={0px 0px 0px 0px},
            clip
        ]{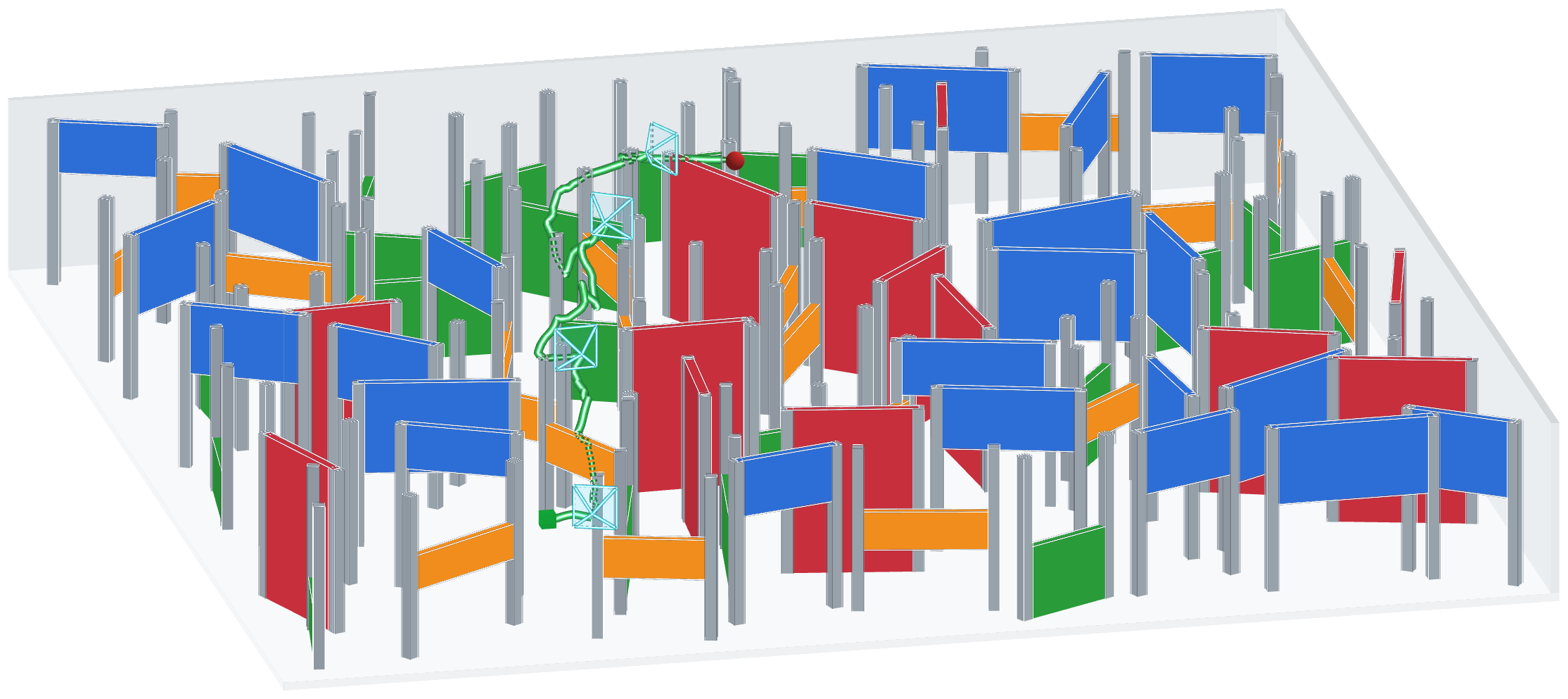}
        \caption{Scene 4.}
        \label{fig:scene_4}
    \end{subfigure}

    \caption{Simulation environments and representative \PlannerName-full
    paths. Green and red markers denote the start and goal. Scene~1A differs
    from the illustrated Scene~1B only by a larger aperture.}
    \label{fig:scenes}
    \vspace{-1.0em}
\end{figure}

\subsubsection{Selected baselines}
The three baselines are \emph{CPA-search} ~\cite{yu2022cpa}, which requires future path states
to be observed in advance and satisfy a safe-distance condition;
\emph{FOV-FMT*-search}~\cite{wang2024multifov}, an FOV-admissible search under the same
forward-facing sensor model; and \emph{OmniPlanner-TR}~\cite{zacharia2026omni}, a local--global
target-reach planner combining global guidance, frontier selection, and local
collision-free motion.

The main benchmark isolates search-level planning, because public
implementations of CPA-Planner, the multi-FOV
planner, and Omni-Planner
are not available and their trajectory backends contain engineering choices
outside the safety-volume certification question. We therefore reimplement
their published search conditions and evaluate all search methods. We additionally include
\PlannerName-full to quantify the effects introduced by SFC-based trajectory
smoothing.
Generic local replanners such as EGO-Planner and Raptor are excluded because
they do not address limited-FOV pre-execution certification.

\subsection{Metrics}
\label{subsec:exp_metrics}

At each replanning cycle, a search-level method receives the current robot and
map states and returns a discrete waypoint path. We densely interpolate this
path piecewise linearly and append only the prefix actually traversed before
the next sensing or replanning event; unexecuted suffixes are not scored. The
resulting executed reference is used for map updates and spatial safety
evaluation. For \PlannerName-full, the optimized trajectory is rasterized
under the same protocol. Because piecewise-linear search paths are not
dynamically feasible at waypoint corners, the main benchmark does not compare
traversal time or control effort; it isolates realized spatial certification.

The benchmark evaluates the spatial safety of the executed closed-loop motion
using $G^{\mathrm{cert}}$ defined in Sec.~\ref{sec:problem}. Although this
criterion is proposed internally by our method, the evaluation itself is a common
safety check applied to all methods: For a robot with nonzero body size,
traversal of $O_t^+ \cup \Phi_t^+$ indicates that the body-inflated volume
swept by the executed path overlaps either known occupied voxels or
frontier-adjacent unknown voxels. We therefore rasterize the executed
interpolated reference and report the number of corresponding voxels that enter
$O_t^+ \cup \Phi_t^+$ at the execution time.

When the interpolated reference is traced through the voxel grid, it can
touch boundary voxels near voxel edges or corners even when it is generated
from certified voxel centers. We classify
such boundary traversals as marginal violations. A rasterized voxel is counted
as marginal if it lies in $O_t^+ \cup \Phi_t^+$ only because of this
discretization-boundary contact. A corresponding voxel is counted as risky if it
lies in $O_t^+ \cup \Phi_t^+$ and is not classified as marginal.
Such grazing contacts are permitted by design: requiring the executed
reference to avoid all inflated cells in the continuous sense would let a
single diagonal pair of inflated voxels seal a passage between certified
voxels. Their penetration depth is below half a voxel and is absorbed by the
physical margin contained in $r_{xy}$ and $r_z$, while a reference that stays
inside the continuous certified region of Sec.~\ref{sec:problem} is covered
by the containment property established there.

For each run, we report the number of distinct marginal voxels
$n_{\mathrm{marg}}$ and risky voxels $n_{\mathrm{risk}}$ entered by the
executed reference; a voxel traversed repeatedly within an episode is counted
once, at the worst severity observed. These counts are reported as
safety-exposure metrics rather than used as hard success thresholds.

A run is considered successful if the robot reaches the goal. We report success rate, executed path
length, $n_{\mathrm{marg}}$, and $n_{\mathrm{risk}}$. Per-scene violation
counts are mean $\pm$ std over all episodes, including failures, because
violations often precede failure. Path length is reported over successful
episodes, and success rate over all episodes. Because risky counts for
certified variants are zero-inflated, we also report \emph{safe reach},
the number of episodes that reach the goal with $n_{\mathrm{risk}}=0$.

\subsection{Main Search-Level Benchmark}
\label{subsec:exp_main}

\begin{table*}[t]
\centering
\footnotesize
\begin{tabular}{l l c c c c c}
\hline
\textbf{Scene} &
\textbf{Method} &
\textbf{Goal reached} $\uparrow$ &
\textbf{Path length (m)} $\downarrow$ &
\textbf{Marginal viol.} $\downarrow$ &
\textbf{Risky viol.} $\downarrow$ &
\textbf{Safe reach} $\uparrow$ \\
\hline
Scene 1A & CPA-search & $19/20$ & $20.00 \pm 5.95$ & $5.85 \pm 4.00$ & $185.35 \pm 69.39$ & $0/20$ \\
Scene 1A & FOV-FMT*-search & $13/20$ & $17.07 \pm 3.53$ & $1.70 \pm 3.28$ & $150.70 \pm 70.24$ & $0/20$ \\
Scene 1A & OmniPlanner-TR & $12/20$ & $38.16 \pm 7.80$ & $1.30 \pm 2.20$ & $\mathbf{0.00 \pm 0.00}$ & $12/20$ \\
Scene 1A & \PlannerName-search & $\mathbf{20/20}$ & $24.60 \pm 6.77$ & $3.70 \pm 2.64$ & $\mathbf{0.00 \pm 0.00}$ & $\mathbf{20/20}$ \\
Scene 1A & \PlannerName-full & $\mathbf{20/20}$ & $27.43 \pm 7.67$ & $3.50 \pm 2.98$ & $0.15 \pm 0.67$ & $19/20$ \\
\hline
Scene 1B & CPA-search & $18/20$ & $34.02 \pm 11.98$ & $4.05 \pm 3.44$ & $256.95 \pm 83.37$ & $0/20$ \\
Scene 1B & FOV-FMT*-search & $0/20$ & -- & $0.75 \pm 1.29$ & $99.45 \pm 46.99$ & $0/20$ \\
Scene 1B & OmniPlanner-TR & $1/20$ & $85.63$ & $1.00 \pm 1.75$ & $1.35 \pm 2.62$ & $1/20$ \\
Scene 1B & \PlannerName-search & $\mathbf{20/20}$ & $36.45 \pm 8.59$ & $7.30 \pm 3.63$ & $\mathbf{0.00 \pm 0.00}$ & $\mathbf{20/20}$ \\
Scene 1B & \PlannerName-full & $\mathbf{20/20}$ & $39.95 \pm 9.45$ & $4.35 \pm 4.17$ & $0.10 \pm 0.45$ & $19/20$ \\
\hline
Scene 2 & CPA-search & $\mathbf{20/20}$ & $23.99 \pm 8.53$ & $4.55 \pm 4.07$ & $57.25 \pm 44.11$ & $1/20$ \\
Scene 2 & FOV-FMT*-search & $19/20$ & $15.53 \pm 4.37$ & $0.85 \pm 1.04$ & $10.10 \pm 14.56$ & $7/20$ \\
Scene 2 & OmniPlanner-TR & $\mathbf{20/20}$ & $23.71 \pm 9.30$ & $0.00 \pm 0.00$ & $\mathbf{0.00 \pm 0.00}$ & $\mathbf{20/20}$ \\
Scene 2 & \PlannerName-search & $\mathbf{20/20}$ & $19.10 \pm 5.32$ & $0.95 \pm 1.23$ & $\mathbf{0.00 \pm 0.00}$ & $\mathbf{20/20}$ \\
Scene 2 & \PlannerName-full & $\mathbf{20/20}$ & $19.41 \pm 5.24$ & $0.75 \pm 1.68$ & $0.25 \pm 0.79$ & $18/20$ \\
\hline
Scene 3 & CPA-search & $11/20$ & $108.15 \pm 26.74$ & $19.10 \pm 17.43$ & $525.25 \pm 256.39$ & $0/20$ \\
Scene 3 & FOV-FMT*-search & $0/20$ & -- & $1.70 \pm 3.21$ & $128.50 \pm 99.78$ & $0/20$ \\
Scene 3 & OmniPlanner-TR & $0/20$ & -- & $0.45 \pm 1.00$ & $\mathbf{0.20 \pm 0.52}$ & $0/20$ \\
Scene 3 & \PlannerName-search & $\mathbf{20/20}$ & $136.21 \pm 14.87$ & $30.50 \pm 9.68$ & $2.65 \pm 4.42$ & $\mathbf{14/20}$ \\
Scene 3 & \PlannerName-full & $\mathbf{20/20}$ & $153.12 \pm 15.53$ & $20.30 \pm 9.79$ & $3.00 \pm 5.17$ & $12/20$ \\
\hline
Scene 4 & CPA-search & $11/20$ & $36.38 \pm 9.71$ & $5.70 \pm 6.55$ & $31.70 \pm 32.23$ & $0/20$ \\
Scene 4 & FOV-FMT*-search & $19/20$ & $26.75 \pm 3.91$ & $2.00 \pm 2.25$ & $18.90 \pm 28.13$ & $4/20$ \\
Scene 4 & OmniPlanner-TR & $18/20$ & $37.22 \pm 6.83$ & $0.25 \pm 0.91$ & $\mathbf{0.00 \pm 0.00}$ & $\mathbf{18/20}$ \\
Scene 4 & \PlannerName-search & $\mathbf{20/20}$ & $32.30 \pm 5.39$ & $3.90 \pm 3.21$ & $1.05 \pm 2.42$ & $16/20$ \\
Scene 4 & \PlannerName-full & $\mathbf{20/20}$ & $32.69 \pm 5.37$ & $2.30 \pm 2.79$ & $0.55 \pm 1.82$ & $17/20$ \\
\hline
\end{tabular}
\caption{Main search-level benchmark. Violation counts are unique inflated
voxels entered by the executed reference, mean\,$\pm$\,std over all $20$
episodes. Bold marks the best value per scene.} 
\label{tab:main_benchmark}
\end{table*}

Table~\ref{tab:main_benchmark} reports the results of the main search-level
benchmark; the comparison separates ordinary goal reaching from
volume-certified goal reaching.

\emph{Reachability.}
Both \PlannerNameSpaced variants reach the goal in all $100$ trials. The baselines remain competitive in
the predominantly horizontal Scene~2 ($19$--$20/20$), but
their success rates decrease sharply when reaching the goal
requires vertical transitions through floor apertures. When the
aperture shrinks from Scene~1A to Scene~1B, the success rates of
FOV-FMT*-search and OmniPlanner-TR drop from $13/20$ and $12/20$ to
$0/20$ and $1/20$, whereas CPA-search is nearly unaffected ($19/20$ to
$18/20$) because it does not wait for the transit volume to be certified.
Scene~3 stacks three such aperture levels with beam lattices in between:
FOV-FMT*-search and OmniPlanner-TR fail in every trial ($0/20$) and
CPA-search reaches only $11/20$, whereas both \PlannerNameSpaced variants
reach $20/20$. The cluttered
Scene~4 largely restores baseline reachability for the two
sensing-limited baselines (FOV-FMT*-search $19/20$,
OmniPlanner-TR $18/20$) --- every obstacle can be observed and
circumvented from many directions --- while CPA-search drops to
$11/20$: the dense occlusion edges continuously trigger its
active-path safety monitor, and the emergency replans it demands
must satisfy the historical-visibility constraint from the
current dynamic state, which repeatedly fails in clutter.
Representative
qualitative behaviors underlying these results are examined
in Sec.~\ref{subsec:exp_qualitative} and
Fig.~\ref{fig:qualitative_results}.

\emph{Safety.}
The violation metrics directly evaluate whether the executed
motion remains within certified free space. In the vertical
scenes, the search-only baselines CPA-search and FOV-FMT*-search
enter on the order of $10^2$ risky voxels per episode on average
($99$--$257$ in Scenes~1A--1B, $129$--$525$ in Scene~3): they advance
toward the goal through non-certified inflated space.
OmniPlanner-TR exhibits the opposite trade-off: its strict
free-space admission confines the graph to already-observed free
space, so it advances only where it holds certified clearance.
Its risky counts are correspondingly near-zero ($1.35$ in
Scene~1B, $0.20$ in Scene~3, and $0.00$ in Scene~1A, the planar
Scene~2, and the cluttered Scene~4),
and the clearance audit below confirms that every one of these
flagged samples is a simulator-flicker artifact retaining
$\geq 0.40$~m clearance rather than genuine proximity. What
OmniPlanner-TR lacks is reachability through the vertical
apertures ($1/20$ in Scene~1B, $0/20$ in Scene~3): unlike the
search-only baselines it will not drive through non-certified
space, but it also cannot actively observe the blocking transit
volume to certify an ascent, so it is pinned below the aperture
and terminates. In the vertical scenes, therefore, only the \PlannerNameSpaced
variants attain both objectives, maintaining near-zero risky
counts ($\leq 0.25$ in Scenes~1--2, $2.65$--$3.00$ in Scene~3,
$0.55$--$1.05$ in Scene~4)
while reaching every goal; \PlannerName-search is risky-free in
every Scene-1 and Scene-2 trial and exhibits only a small
bounded tail in Scenes~3 and~4. Scene~4 makes the two sides of
the comparison explicit: the \PlannerNameSpaced variants match the
near-zero exposure of OmniPlanner-TR --- the baseline whose
sensing model this scene favors most --- while additionally
reaching the goals it forfeits, and they match the reachability
of FOV-FMT*-search while avoiding the ${\sim}19$ risky voxels per
episode that it incurs (it attains safe reach in only
$4/20$ episodes).   

The residual samples classified as risky for the two
search-level planners that restrict execution to body-valid
observed free space, \PlannerName-search and
OmniPlanner-TR, are caused by simulator-side occupancy
flicker rather than an actual loss of clearance. In a separate
static test, voxels along a fixed beam silhouette repeatedly
switched between occupied and free, causing the inflated
obstacle boundary to shift by one voxel. A ground-truth audit
confirmed that every flagged \PlannerName-search sample
retained at least $0.33$~m clearance and every flagged
OmniPlanner-TR sample at least $0.40$~m, whereas risky
samples from CPA-search and FOV-FMT*-search extended to
physical collision. These residual search-level counts therefore
reflect a simulation artifact rather than a planner safety failure.
    
\emph{Path efficiency.}
The additional observation actions required for safety
certification moderately increase the executed path length in
the vertical scenes. For example, the Scene-3 mean path
lengths of the \PlannerNameSpaced variants are $136$--$153$~m, compared
with $108$~m for CPA-search. However, baseline path lengths
are averaged only over successful episodes, which in the harder
scenes correspond to the more favorable trials.

\emph{Search layer vs.\ backend.}
\PlannerName-search publishes the certified search path directly,
with the trajectory backend disabled, and already achieves full
success with zero risky counts in Scenes~1A, 1B, and~2. Its small
residuals in Scenes~3 and~4 are attributed to the simulator-side
occupancy flicker discussed above. The safety advantage is therefore
established at the search layer rather than recovered
downstream by trajectory optimization.

\PlannerName-full realizes the certified search output as a
smooth, dynamically feasible trajectory within constructed
corridors. The small nonzero risky counts in Scenes~1A, 1B,
and~2 ($\leq 0.25$) reflect the slight corridor-realization slack of
the trajectory backend (Sec.~\ref{sec:traj_opt}), which bounds the
smoothed trajectory by corridor faces rather than by voxel centers.
In Scenes~3 and~4, the backend reduces marginal violations from
$30.50$ to $20.30$ and from $3.90$ to $2.30$ voxels per episode,
while the risky counts remain comparable ($2.65$ versus $3.00$ and
$1.05$ versus $0.55$); the latter may contain both simulator-side
occupancy flicker and the same backend slack.

\subsection{Qualitative Analysis}
\label{subsec:exp_qualitative}

Fig.~\ref{fig:qualitative_results} illustrates representative
Scene-1B behaviors underlying the aggregate results above.
\PlannerName-search first moves among viewpoints beneath the
ceiling slab to observe and certify the body-inflated transit
volume around the aperture, and only then executes the ascent
(Fig.~\ref{fig:qual_ffa_success}). The resulting reference
contains only marginal contacts around the aperture rim.

FOV-FMT*-search constrains each local connection by its motion
direction. In Fig.~\ref{fig:qual_fmt_fail}, every sampled
aperture-crossing branch is rejected: red branches near the rim
violate obstacle inflation, whereas purple branches provide
sufficient collision clearance but are too steep to lie within
the forward FOV.

CPA-search guarantees that future path states were observed
earlier, but does not certify the complete swept body volume.
Its successful execution therefore crosses the aperture with an
extended risky section
(red paths in Fig.~\ref{fig:qual_cpa_success}). In the failed execution
(Fig.~\ref{fig:qual_cpa_fail}), no motion-primitive chain can
both approach the aperture center with sufficiently small
lateral velocity and preserve historical visibility for the
subsequent vertical ascent.

OmniPlanner-TR restricts motion to observed free space, but
its target-directed frontier guidance does not explicitly identify
the blocking transit volume as an observation objective. In the
rare successful execution
(Fig.~\ref{fig:qual_omni_success}), general exploratory motion
incidentally reveals enough of the aperture region to form a
safe ascent. In the failed execution
(Fig.~\ref{fig:qual_omni_fail}), the robot continues exploring
below the aperture without clearing a complete body-valid
connection through it, and eventually terminates without further
progress toward the goal.

\begin{figure}[!t]
    \centering
    \begin{subfigure}[t]{0.49\linewidth}
        \centering
        \includegraphics[width=\linewidth, trim={10.0cm 10.0cm 10.0cm 0.0cm}, clip]{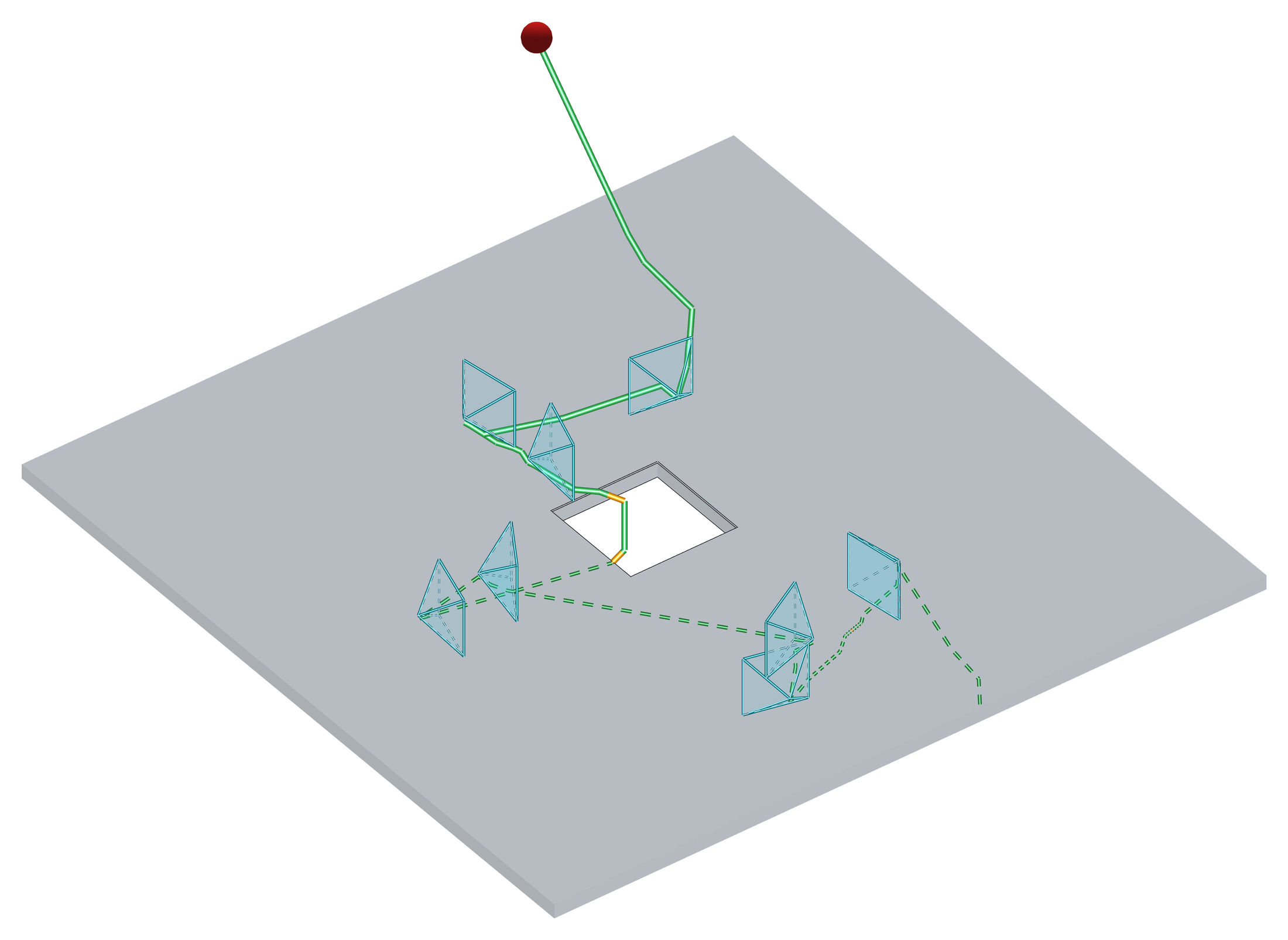}
        \caption{\PlannerName-search success.}
        \label{fig:qual_ffa_success}
    \end{subfigure}
    \hfill
    \begin{subfigure}[t]{0.49\linewidth}
        \centering
        \includegraphics[width=\linewidth, trim={20.0cm 14.0cm 20.0cm 10.0cm}, clip]{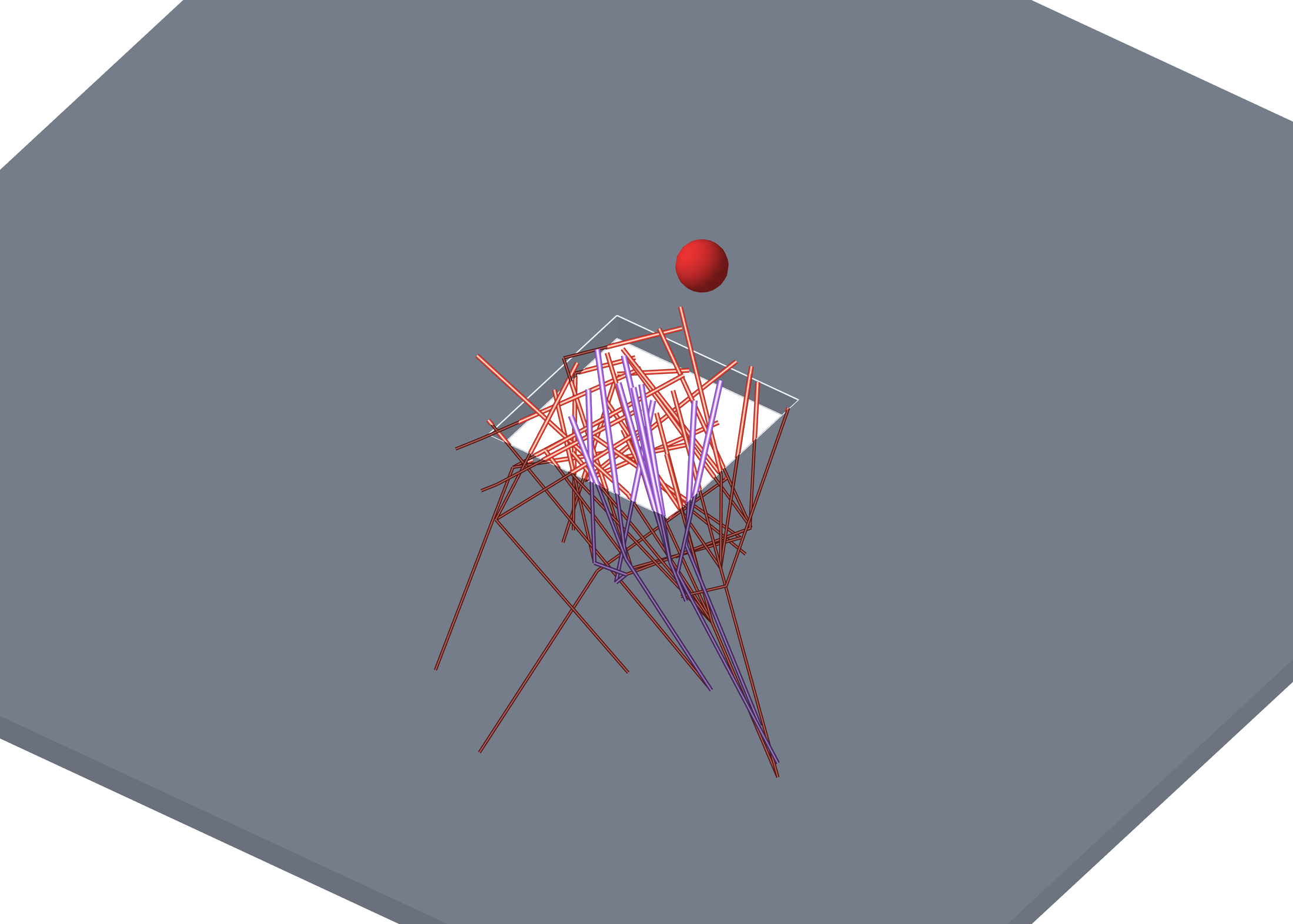}
        \caption{FOV-FMT*-search failure.}
        \label{fig:qual_fmt_fail}
    \end{subfigure}

    \vspace{0.3em}

    \begin{subfigure}[t]{0.49\linewidth}
        \centering
        \includegraphics[width=\linewidth]{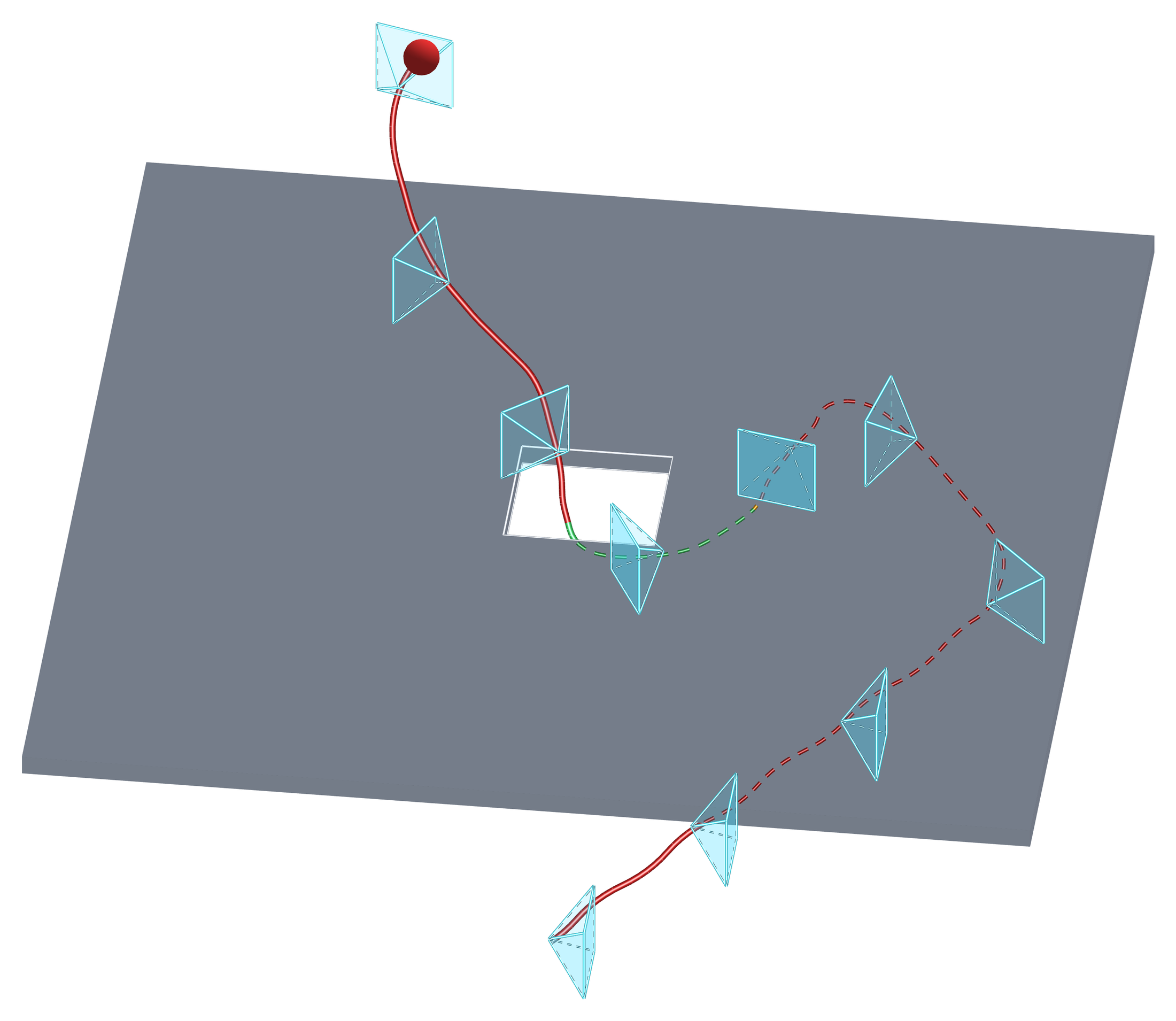}
        \caption{CPA-search success.}
        \label{fig:qual_cpa_success}
    \end{subfigure}
    \hfill
    \begin{subfigure}[t]{0.49\linewidth}
        \centering
        \includegraphics[width=\linewidth, trim={0.0cm 20.0cm 0.0cm 0.0cm}, clip]{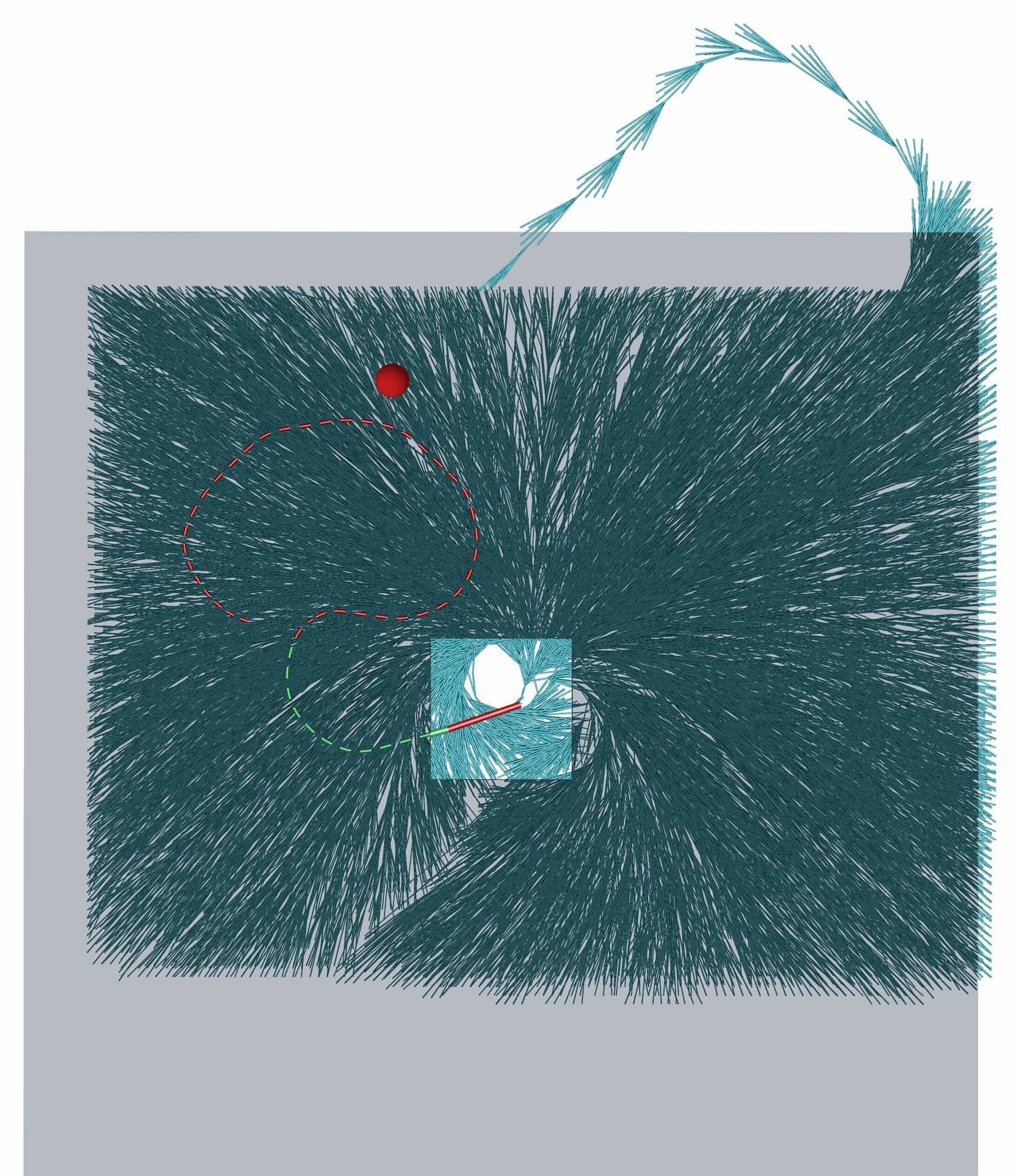}
        \caption{CPA-search failure.}
        \label{fig:qual_cpa_fail}
    \end{subfigure}

    \vspace{0.3em}

    \begin{subfigure}[t]{0.49\linewidth}
        \centering
        \includegraphics[width=\linewidth, trim={1.0cm 0.0cm 0.0cm 1.0cm}]{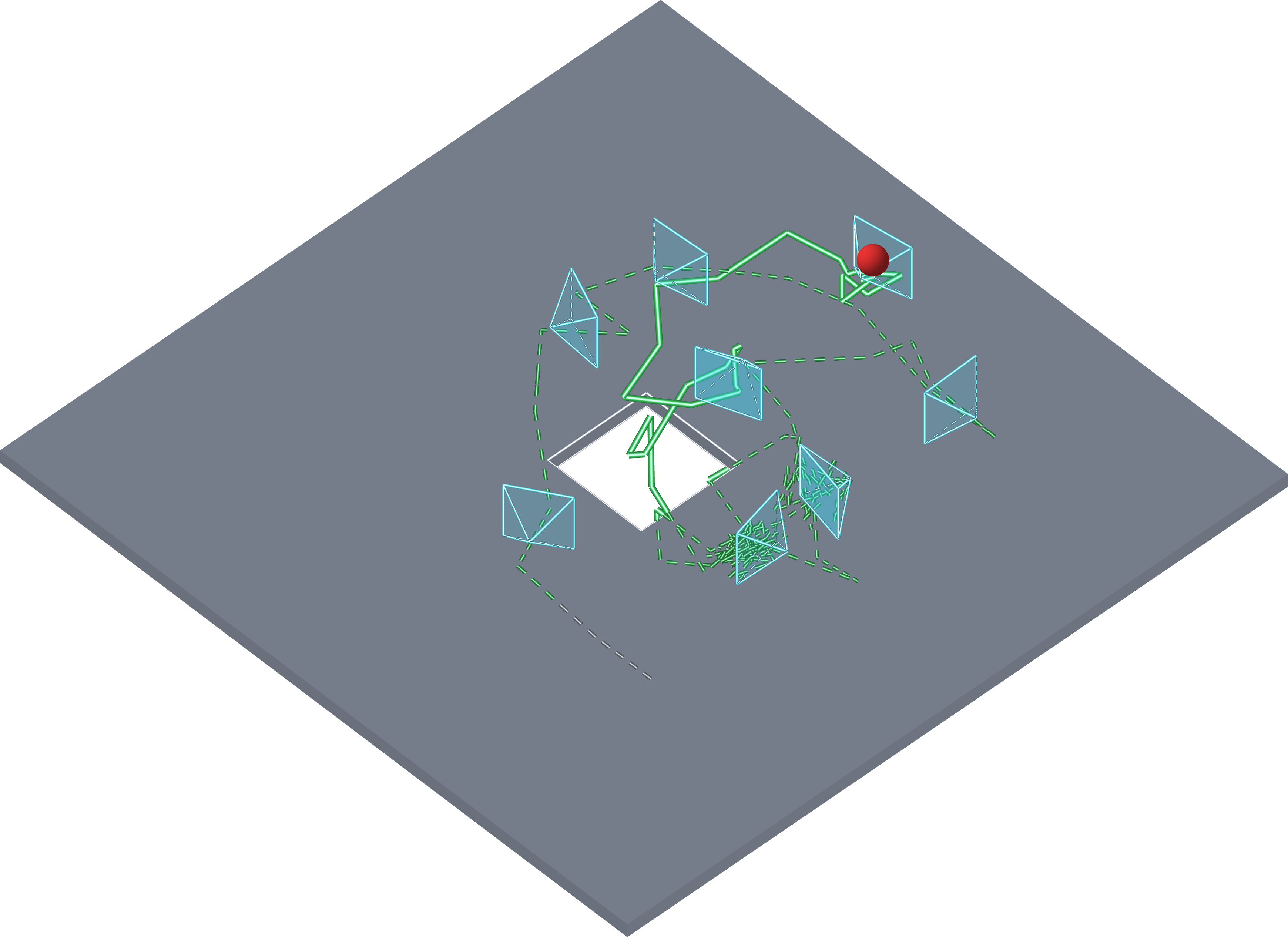}
        \caption{OmniPlanner-TR success.}
        \label{fig:qual_omni_success}
    \end{subfigure}
    \hfill
    \begin{subfigure}[t]{0.49\linewidth}
        \centering
        \includegraphics[width=\linewidth]{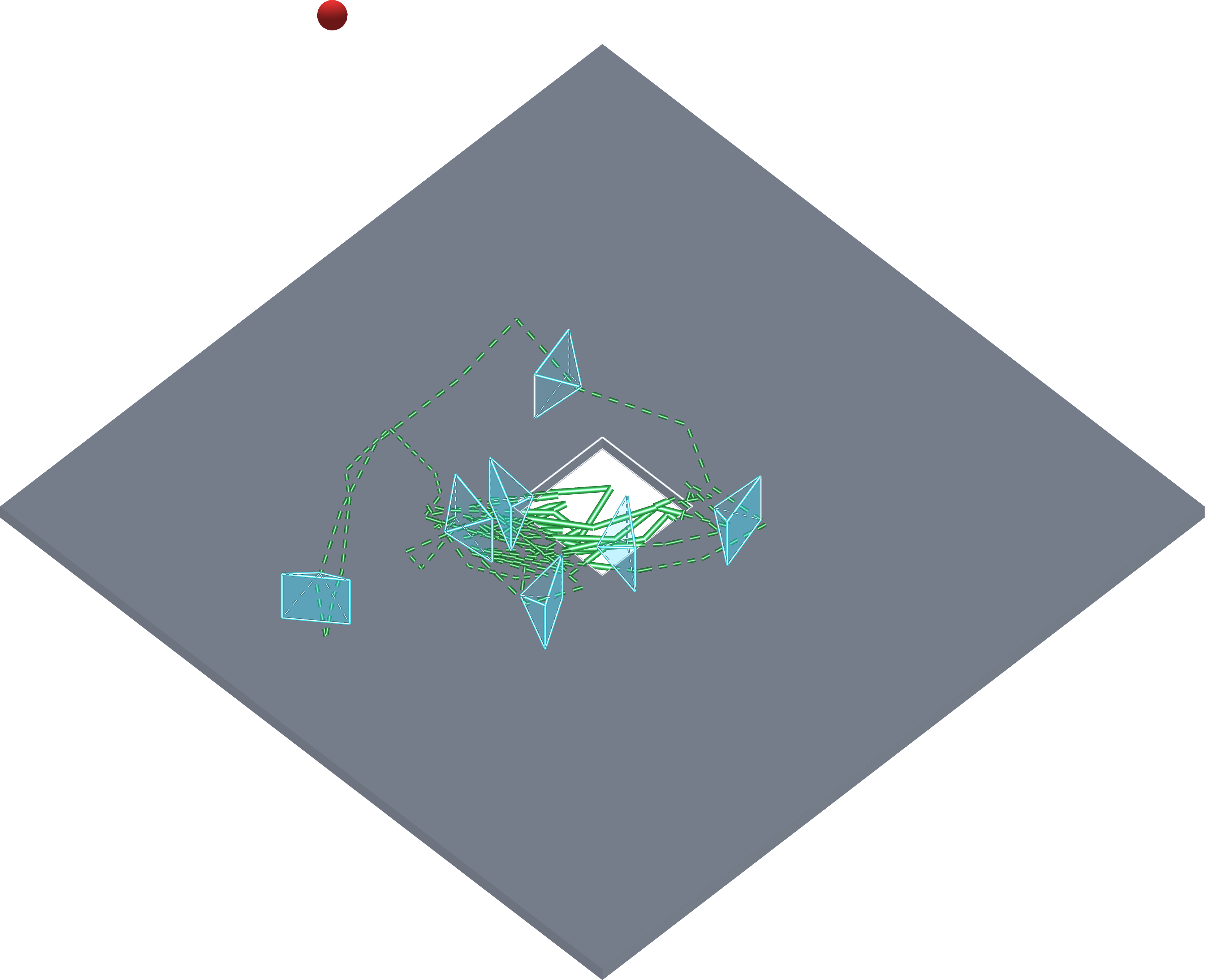}
        \caption{OmniPlanner-TR failure.}
        \label{fig:qual_omni_fail}
    \end{subfigure}

    \caption{Representative examples of planners in Scene 1B. 
    Paths are colored by violation class (green: safe, orange: marginal, red: risky); 
    cyan frusta show the sensor heading; the red sphere is the goal.
    }
    \label{fig:qualitative_results}
    
\end{figure}

\subsection{Ablation Study of Preview Routine}
\label{subsec:exp_ablation}

We isolate the contribution of the certified preview routine
(Sec.~\ref{sec:receding-preview}), the one component that can be removed
without altering the safety-volume certification formulation or the set of
tasks the planner can solve. 
We run the \emph{No preview} variant on the same $20$ Scene~3 tasks used for the
main benchmark (Table~\ref{tab:main_benchmark}). Results are summarized in
Table~\ref{tab:ablation}; the safety column reuses the risky-voxel metric of
Sec.~\ref{subsec:exp_metrics}.

\begin{table}[t]
\centering
\footnotesize
\begin{tabular}{l c c c c}
\hline
\textbf{Variant} &
\textbf{Reached} $\uparrow$ &
\textbf{Risky} $\downarrow$ &
\textbf{Path (m)} &
\textbf{Time (s)} $\downarrow$ \\
\hline
\PlannerName-full & 20/20 & 3.00 & $153.1\pm15.5$ & $218.5\pm24.7$ \\
No preview        & 20/20 & 2.40 & $133.6\pm16.3$ & $276.4\pm35.5$ \\   
\hline
\end{tabular}
\caption{Preview ablation on the $20$ Scene~3 tasks.}
\label{tab:ablation}
\end{table}

Removing the preview routine leaves both goal completion ($20/20$) and safety
exposure ($2.40$ vs.\ $3.00$ risky voxels per episode) essentially unchanged,
confirming that preview is an execution-smoothing layer and not part of the
feasibility or safety guarantees (Sec.~\ref{sec:receding-preview}). The
speed-up does not come from a shorter route: the \emph{No preview} trajectories
are in fact $15\%$ \emph{shorter} on average ($133.6$ vs.\ $153.1$~m). Because
preview optimistically commits to a farther lookahead viewpoint, that viewpoint
occasionally fails to observe its intended target and the planner falls back to
clearing the current hitpoint directly, so preview can trace a slightly longer,
less direct path. Despite covering more distance, \PlannerName-full completes
the missions $27\%$ faster on average ($218.5$ vs.\ $276.4$~s),
because it sustains continuous motion through certified intermediate anchors
instead of decelerating to rest and re-accelerating at every viewpoint. 

\subsection{Recursive Subgoal Activation}
\label{subsec:exp_recursion}

Recursive subgoals are invoked when all useful viewpoints
for the active observation target are reachable only through uncertified
space. In our randomized benchmark tasks this branch was rarely activated:
never in Scenes~1A, 1B, and~2; $10$ times for \PlannerName-full and $2$
times for \PlannerName-search in Scene~3, always at the vertical apertures
where the first target-centric search returned only contaminated clearing
viewpoints; and $3$ times within a single \PlannerName-search episode in
Scene~4, each cleared by the next map update.
Each activation invokes the level-by-level necessity check of
Sec.~\ref{sec:subgoal}, which clears the intermediate obligation before the
parent obligation is reconsidered, and every one resolved so that all runs
reached the goal. 


\subsection{Real-Robot Demonstration}
\label{subsec:exp_real_robot}

We validate the complete \PlannerNameSpaced pipeline through real-robot
experiments in four representative scenarios. In every run, onboard sensing,
online voxel-map update, search replanning, SFC-based trajectory optimization,
and closed-loop control operate concurrently. These experiments provide
qualitative system-level validation rather than a statistical comparison;
the controlled quantitative evaluation is reported in
Table~\ref{tab:main_benchmark}.

The quadrotor platform and sensing configuration are shown in
Fig.~\ref{fig:real_robot_platform}. A MID-360 LiDAR is mounted beneath the
vehicle to provide strong coverage of the surface below the robot. This
configuration reflects the intended construction-site application, in which
the vehicle navigates above an excavation while reconstructing the excavation
surface beneath it. The full LiDAR FOV is used for odometry and online
voxel-map update.
For observation planning, \PlannerNameSpaced uses only the designated
forward planning FOV illustrated in Fig.~\ref{fig:drone_2}. 
In the two demonstrated navigation scenarios, the
safety-critical unknown regions lie primarily above or ahead of the vehicle.
Consequently, the downward-looking measurements assist mapping and state
estimation but do not by themselves certify the safety volumes required for
the subsequent ascent or entrance maneuver.

The first scenario, shown in the upper row of Fig.~\ref{fig:real_robot_demo},
is a multi-level ascent task.
At location~1 (Fig.~\ref{fig:real_robot_snapshot_1}), the planner
selects a clearing viewpoint together with a preview viewpoint, allowing
continued progress while observing farther along the guidance path. At
location~2 (Fig.~\ref{fig:real_robot_snapshot_2}), newly revealed geometry
leaves part of the required safety volume unknown, triggering an
unknown-observation action. The resulting observation expands the certified
region sufficiently for the planner to complete the ascent at location~3
(Fig.~\ref{fig:real_robot_snapshot_3}).

The second scenario, shown in the lower row of Fig.~\ref{fig:real_robot_demo},
involves entering through a low-clearance opening. 
At location~1 (Fig.~\ref{fig:real_robot_gate_snapshot_1}), 
no certified preview viewpoint is available beyond the opening, 
so the planner executes the immediate clearing
action without preview. Because part of the entrance safety volume remains
unknown, at location~2
(Figs.~\ref{fig:real_robot_gate_snapshot_2} and
\ref{fig:real_robot_gate_snapshot_3}) it selects an unknown-observation
viewpoint to resolve the remaining occlusion before entry. 
Thus, \PlannerNameSpaced certifies the full inflated volume before traversal, 
rather than relying on visibility through the opening alone.

The third scenario
(Figs.~\ref{fig:ric_2d_photo}--\ref{fig:ric_2d_rviz})
demonstrates navigation in a cluttered indoor environment, including both
lateral detours around obstacles and traversal over them by exploiting
available vertical free space.
The fourth scenario
(Figs.~\ref{fig:stairwell_down}--\ref{fig:stairwell_rviz})
demonstrates ascent through a narrow stairwell, where the vehicle
must exploit the limited available space while continuously certifying the
route ahead under its onboard planning FOV.
Together, these examples further demonstrate the ability of
\PlannerNameSpaced to navigate constrained 3-D environments beyond the two
representative cases discussed above.



\begin{figure}[t]
    \centering
    \subfloat[UAV platform.\label{fig:drone_1}]{%
        \includegraphics[
            width=0.48\columnwidth,
            trim={0cm 1cm 1.33cm 0cm},
            clip
        ]{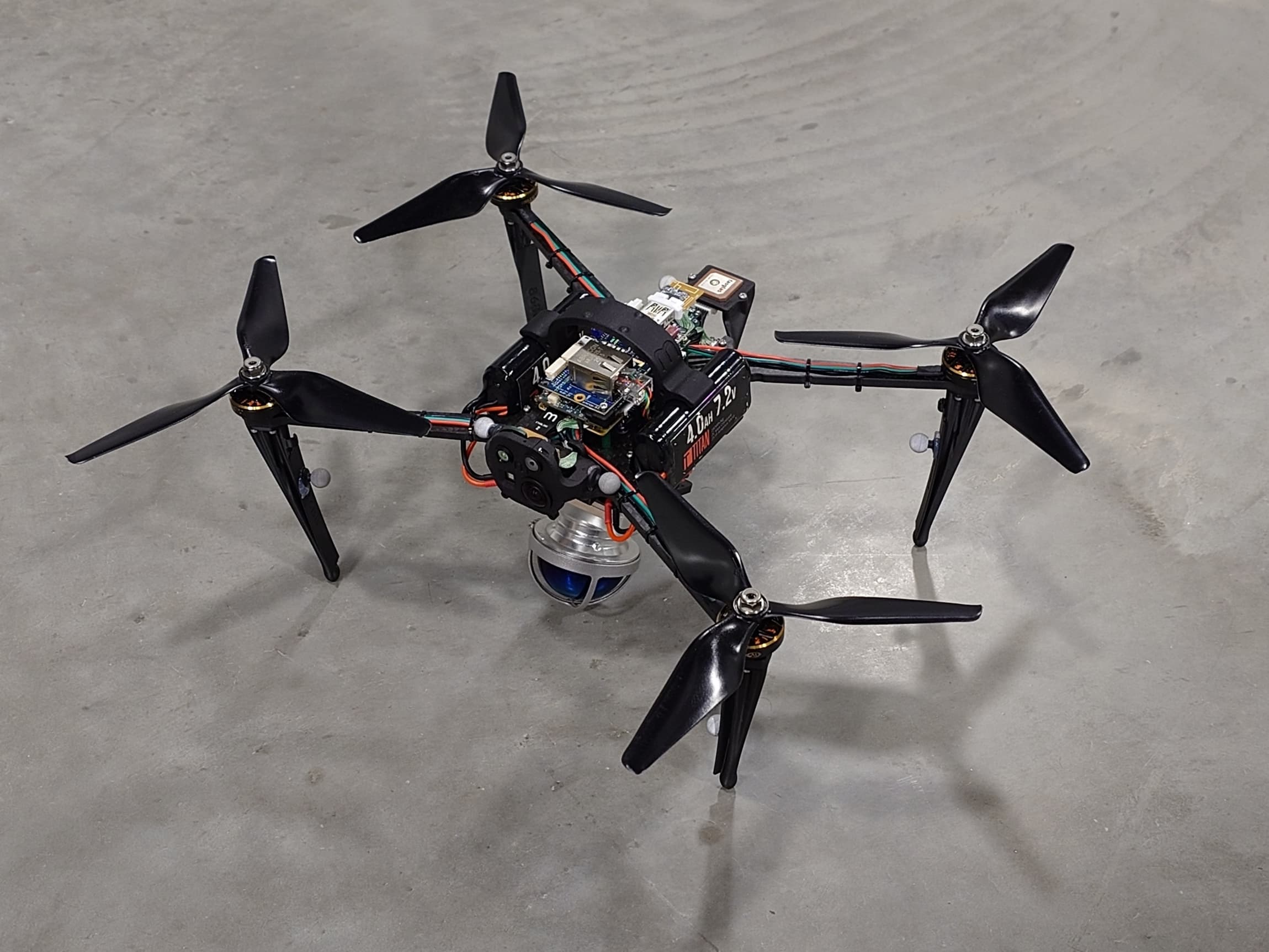}%
    }\hfill%
    \subfloat[FOV configuration.\label{fig:drone_2}]{%
        \includegraphics[
            width=0.48\columnwidth,
            trim={0cm 2.5cm 0cm 0cm},
            clip
        ]{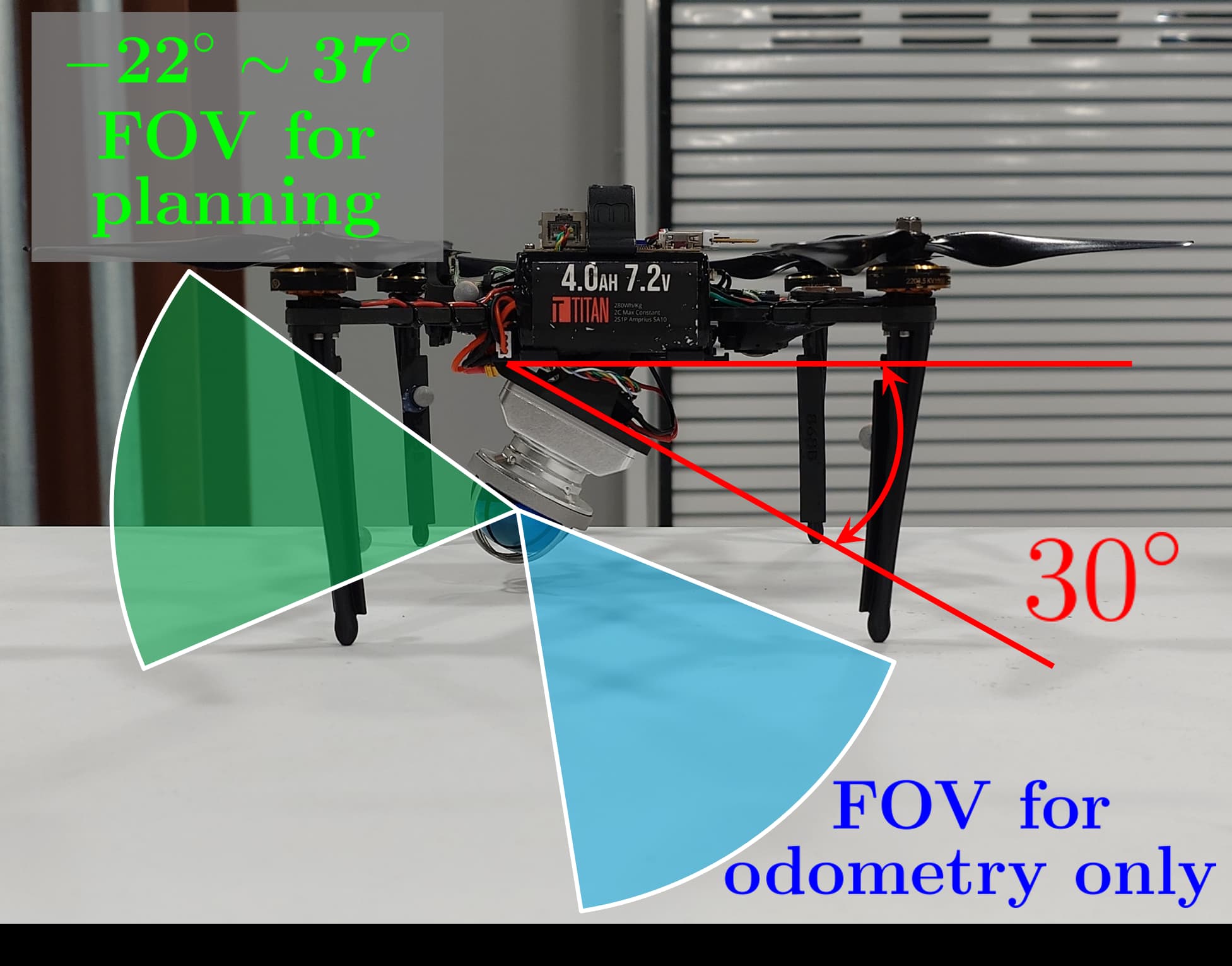}%
    }
    \caption{Real-robot platform and sensing configuration.}
    \label{fig:real_robot_platform}
\end{figure}

\begin{figure*}[t]
    \centering
    \begin{subfigure}[t]{0.24\textwidth}
        \centering
        \includegraphics[
            trim={0cm 0cm 0cm 0cm},
            clip,
            width=\linewidth
        ]{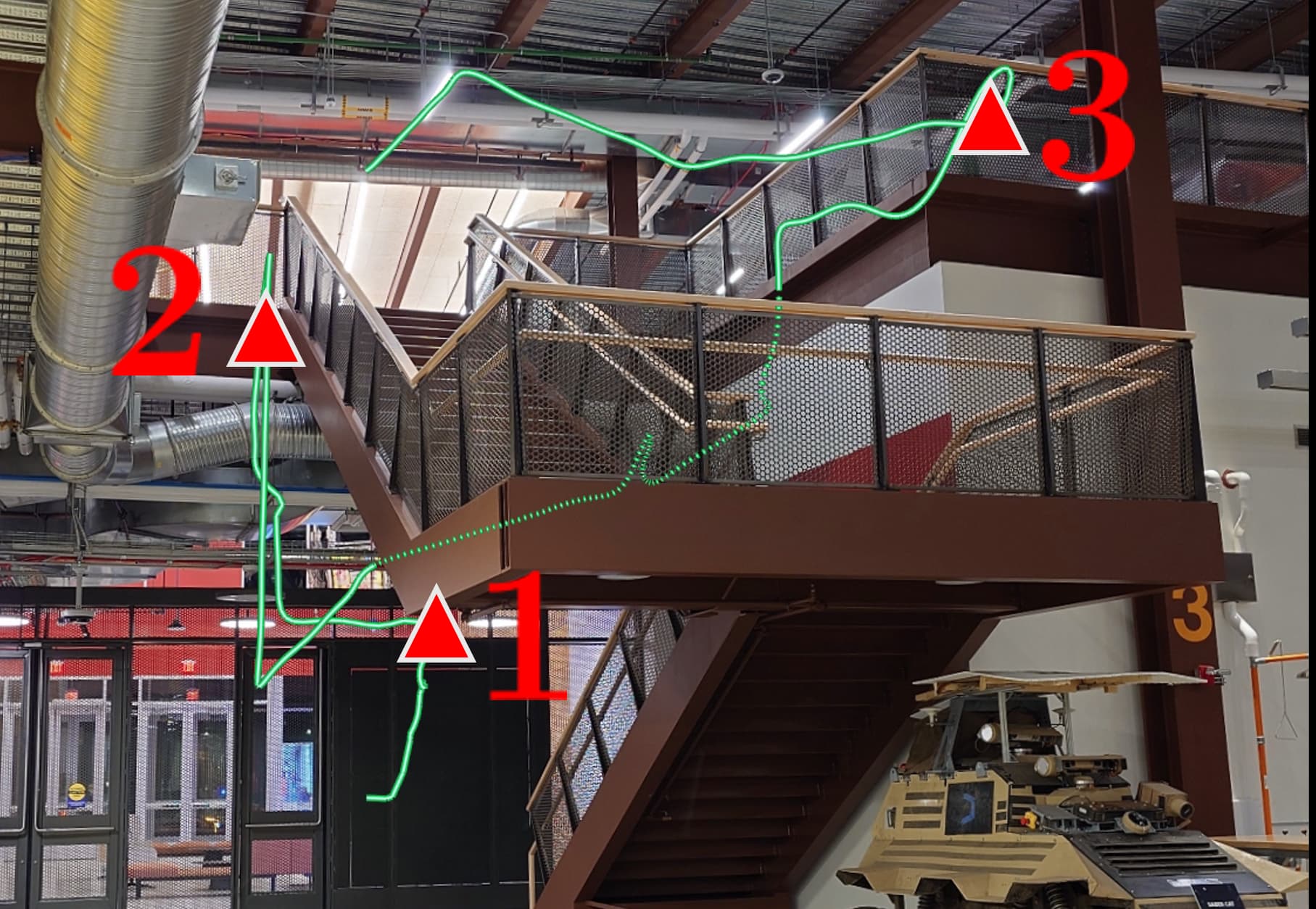}
        \caption{Multi-level ascent.}
        \label{fig:real_robot_overview}
    \end{subfigure}
    \hfill
    \begin{subfigure}[t]{0.24\textwidth}
        \centering
        \includegraphics[
            trim={0cm 0cm 0cm 0cm},
            clip,
            width=\linewidth
        ]{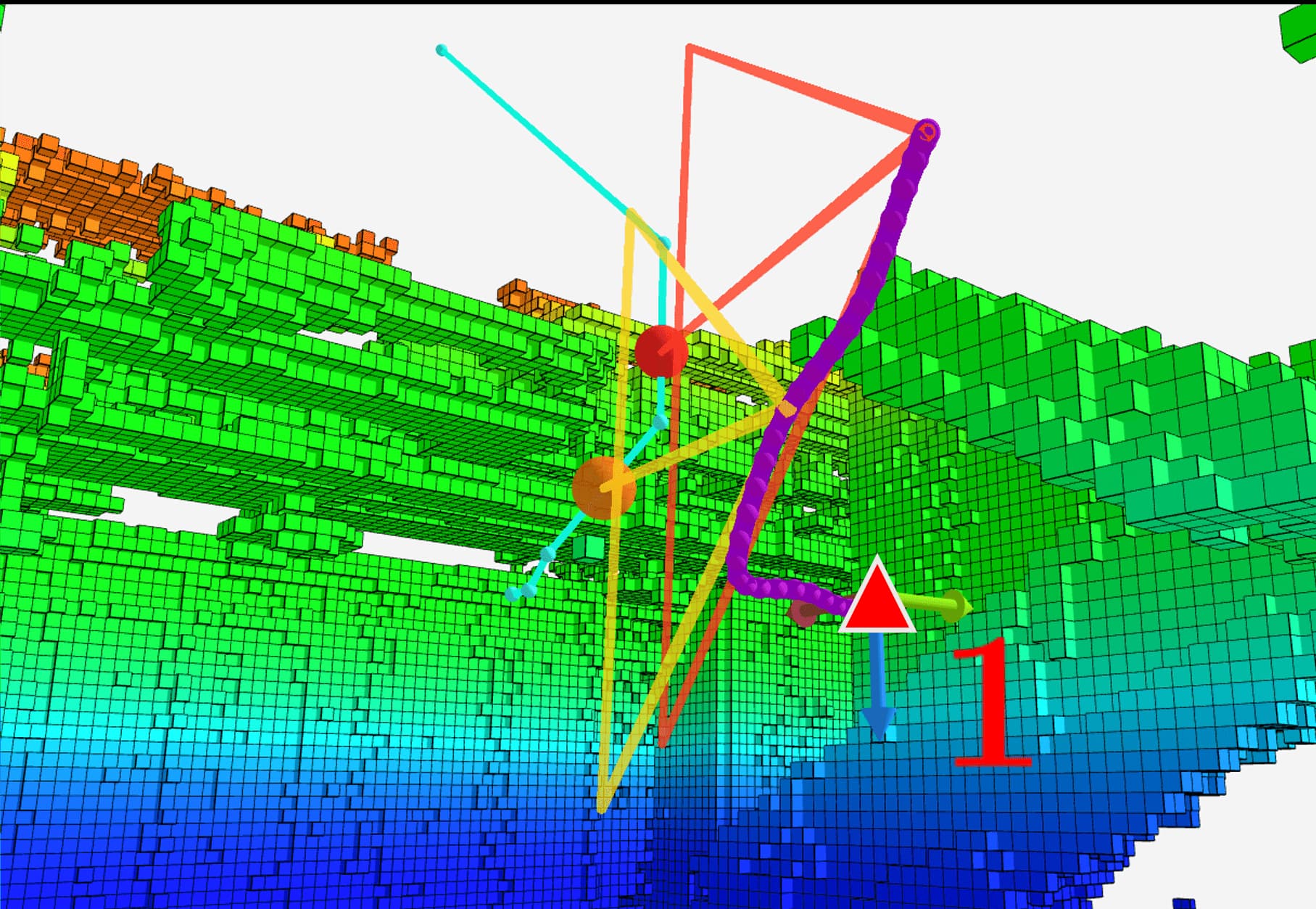}
        \caption{Clearing and preview planning.}
        \label{fig:real_robot_snapshot_1}
    \end{subfigure}
    \hfill
    \begin{subfigure}[t]{0.24\textwidth}
        \centering
        \includegraphics[
            trim={0cm 0cm 0cm 0cm},
            clip,
            width=\linewidth
        ]{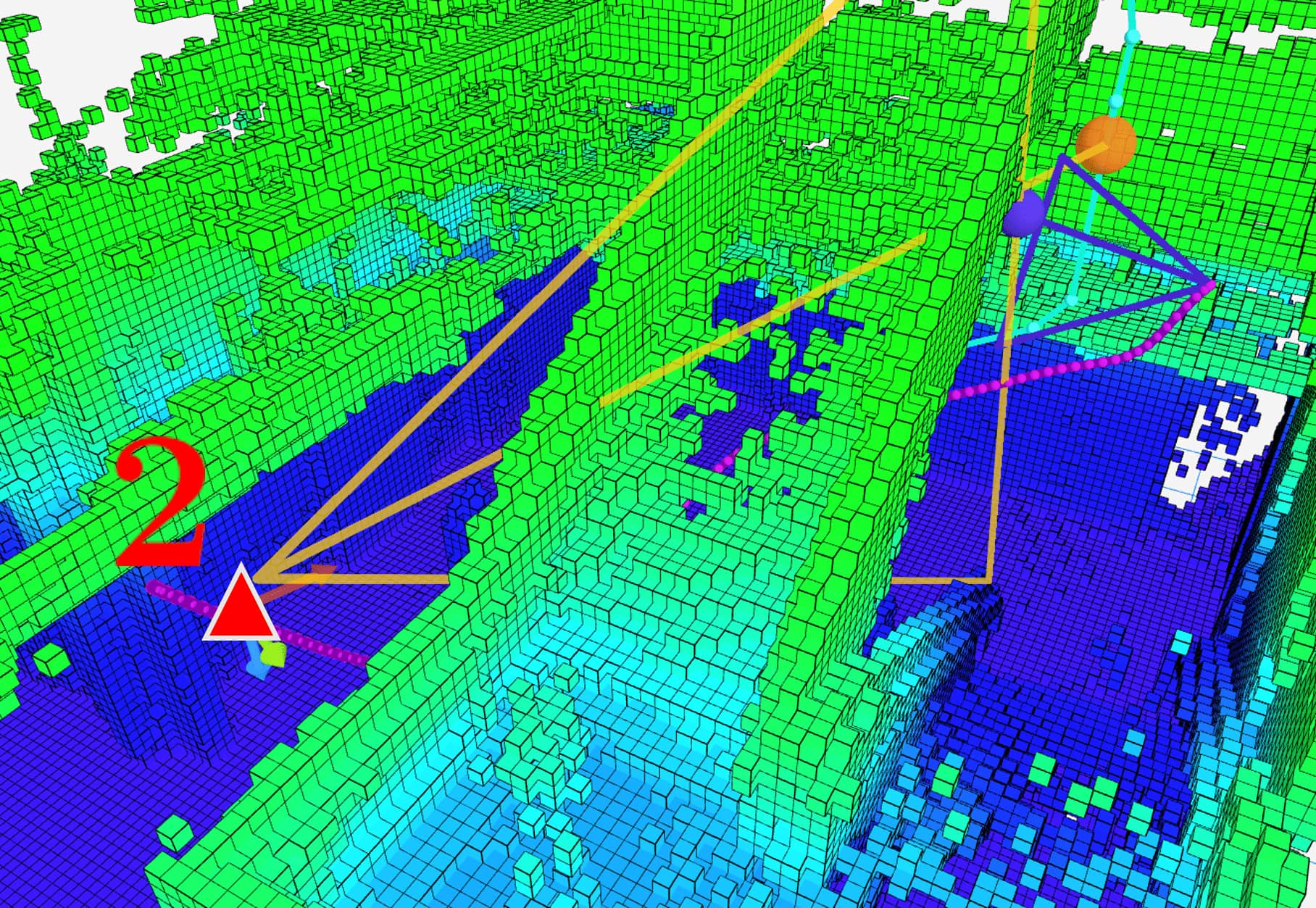}
        \caption{Unknown-observation fallback.}
        \label{fig:real_robot_snapshot_2}
    \end{subfigure}
    \hfill
    \begin{subfigure}[t]{0.24\textwidth}
        \centering
        \includegraphics[
            trim={0cm 0cm 0cm 0cm},
            clip,
            width=\linewidth
        ]{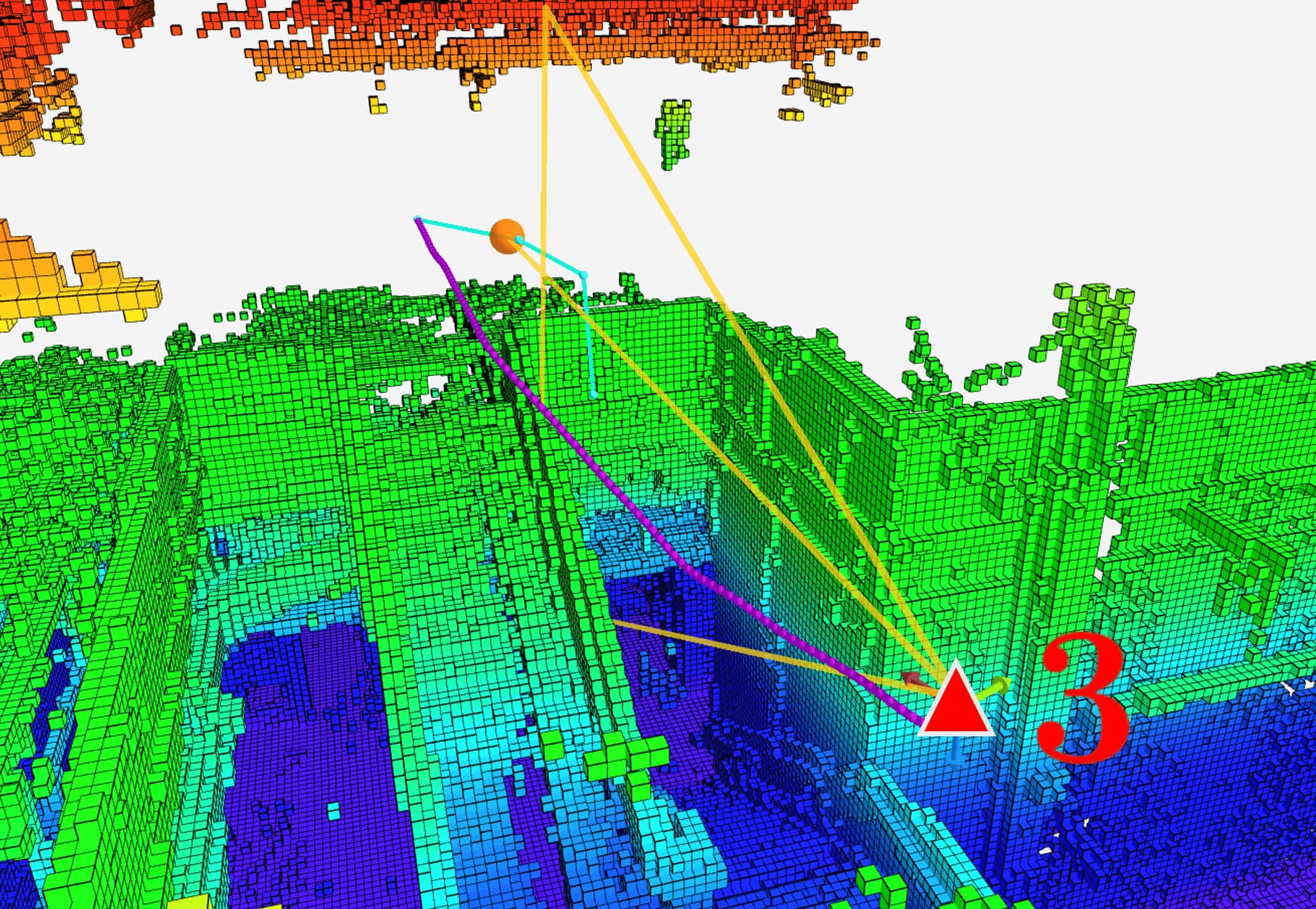}
        \caption{Certified route to the upper goal.}
        \label{fig:real_robot_snapshot_3}
    \end{subfigure}

    \vspace{0.15cm}

    \begin{subfigure}[t]{0.24\textwidth}
        \centering
        \includegraphics[
            trim={0cm 0cm 0cm 0cm},
            clip,
            width=\linewidth
        ]{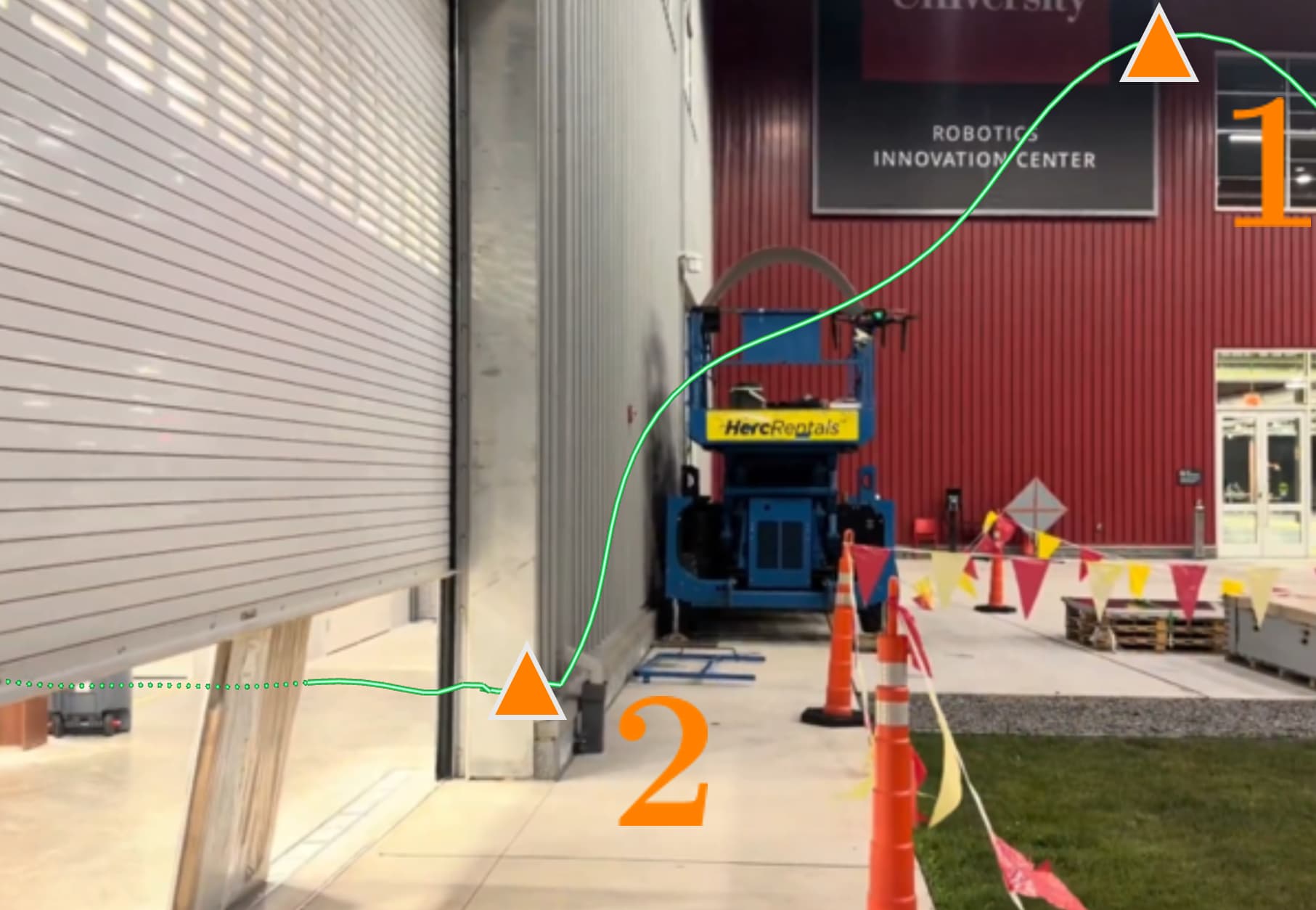}
        \caption{Low-clearance entrance.}
        \label{fig:real_robot_gate_overview}
    \end{subfigure}
    \hfill
    \begin{subfigure}[t]{0.24\textwidth}
        \centering
        \includegraphics[
            trim={0cm 0cm 0cm 0cm},
            clip,
            width=\linewidth
        ]{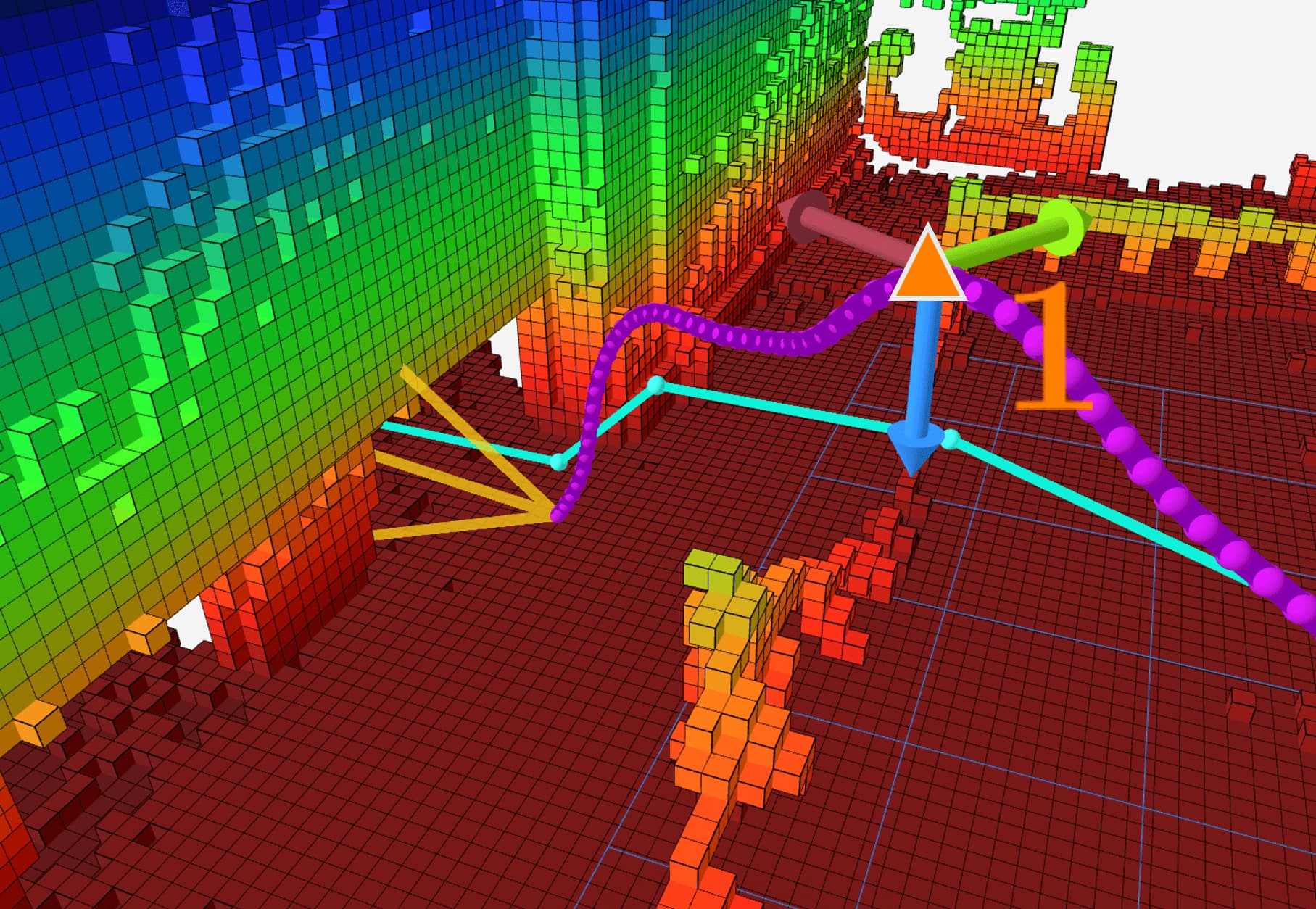}
        \caption{Clearing at the entrance.}
        \label{fig:real_robot_gate_snapshot_1}
    \end{subfigure}
    \hfill
    \begin{subfigure}[t]{0.24\textwidth}
        \centering
        \includegraphics[
            trim={0cm 0cm 0cm 0cm},
            clip,
            width=\linewidth
        ]{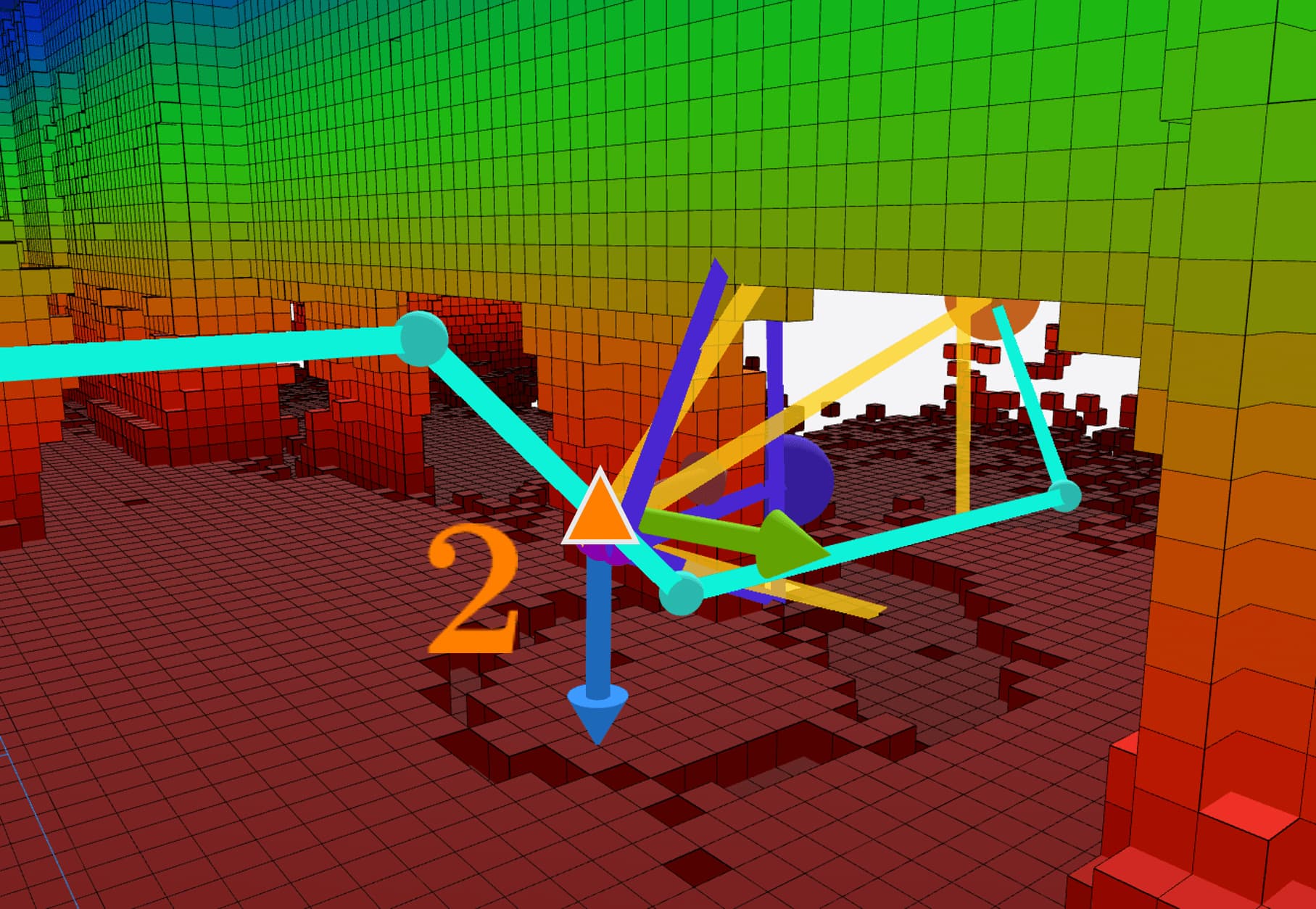}
        \caption{Unknown observation before entry.}
        \label{fig:real_robot_gate_snapshot_2}
    \end{subfigure}
    \hfill
    \begin{subfigure}[t]{0.24\textwidth}
        \centering
        \includegraphics[
            trim={0cm 0cm 0cm 0cm},
            clip,
            width=\linewidth
        ]{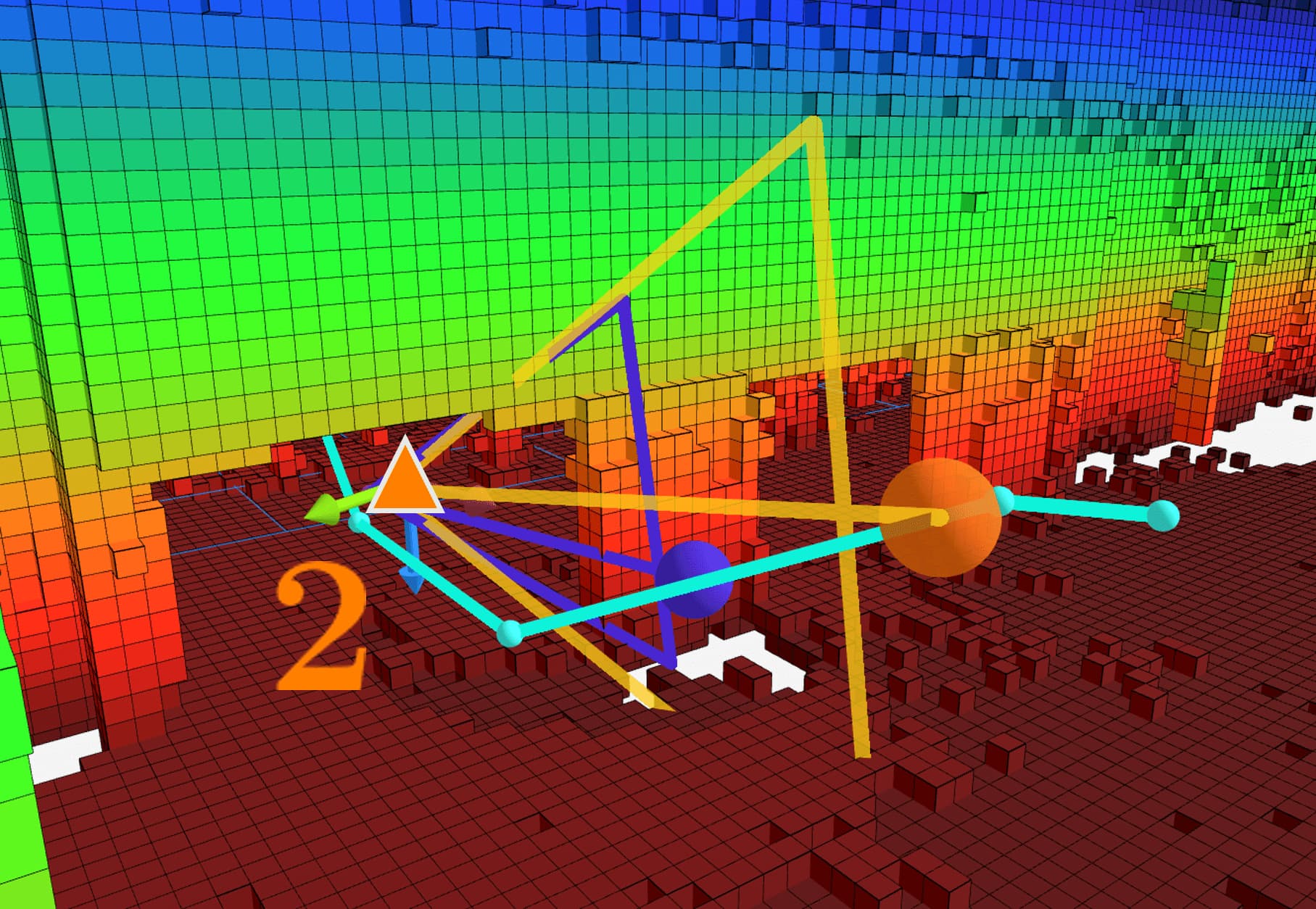}
        \caption{Opposite-side view of the traversal.}
        \label{fig:real_robot_gate_snapshot_3}
    \end{subfigure}

\caption{
Real-robot demonstrations in two representative scenarios.
The upper row shows a multi-level ascent in a factory environment, and the
lower row shows an outdoor-to-indoor transition through a low-clearance
opening. Cyan denotes the optimistic guidance path, deep magenta denotes the
planned trajectory, and yellow, red, and dark-purple markers denote clearing,
preview, and unknown-observation FOVs and targets, respectively. Each row
shows the complete executed path and representative online replanning states.
}
    \label{fig:real_robot_demo}
    \vspace{-0.3cm}
\end{figure*}

\begin{figure}[t]
    \centering

    \begin{subfigure}[t]{0.4814\columnwidth}
        \centering
        \includegraphics[
            trim={0cm 0cm 0cm 0cm},
            clip,
            width=\linewidth
        ]{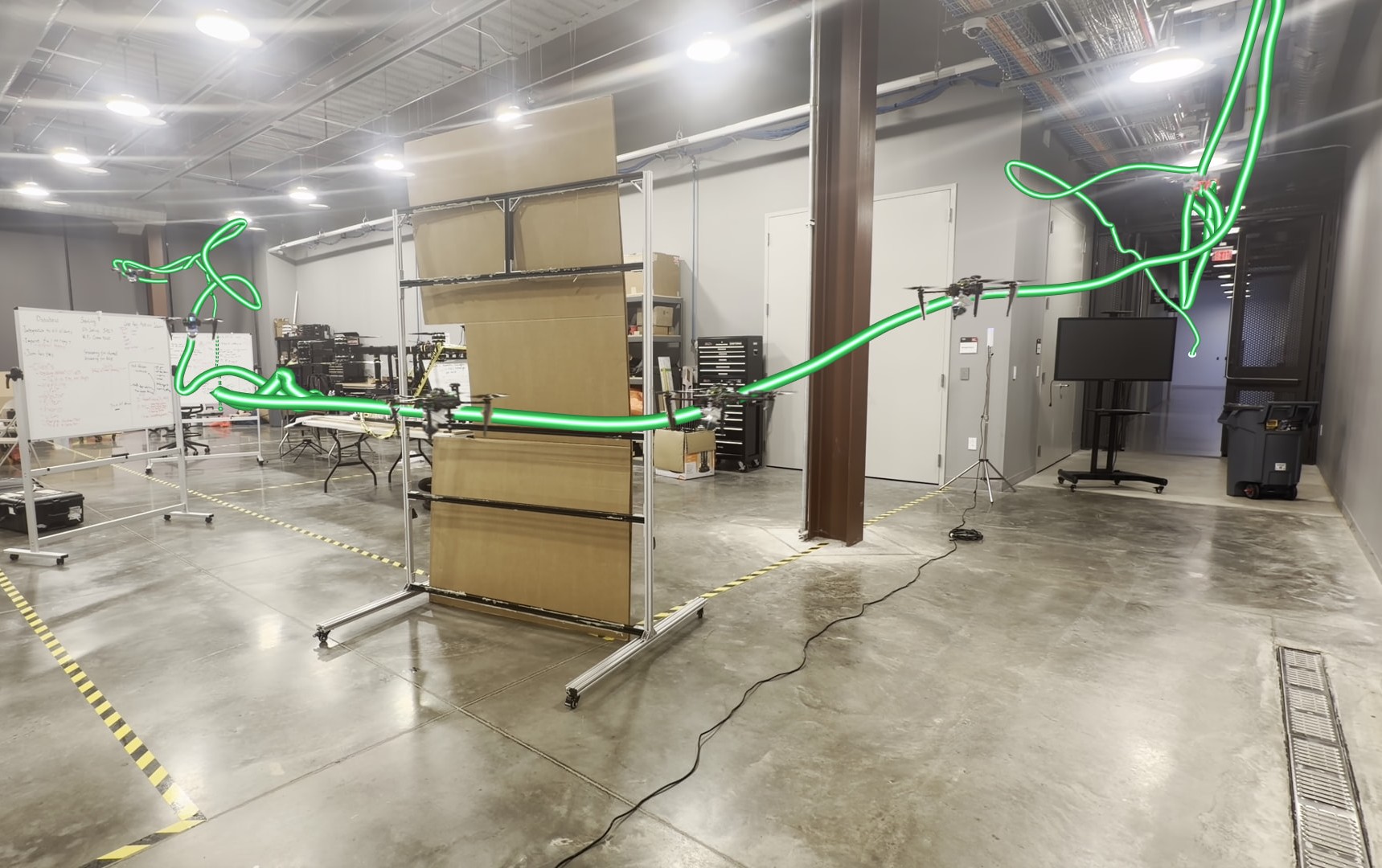}
        \caption{}
        \label{fig:ric_2d_photo}
    \end{subfigure}
    \hfill
    \begin{subfigure}[t]{0.4786\columnwidth}
        \centering
        \includegraphics[
            trim={0cm 0cm 0cm 0cm},
            clip,
            width=\linewidth
        ]{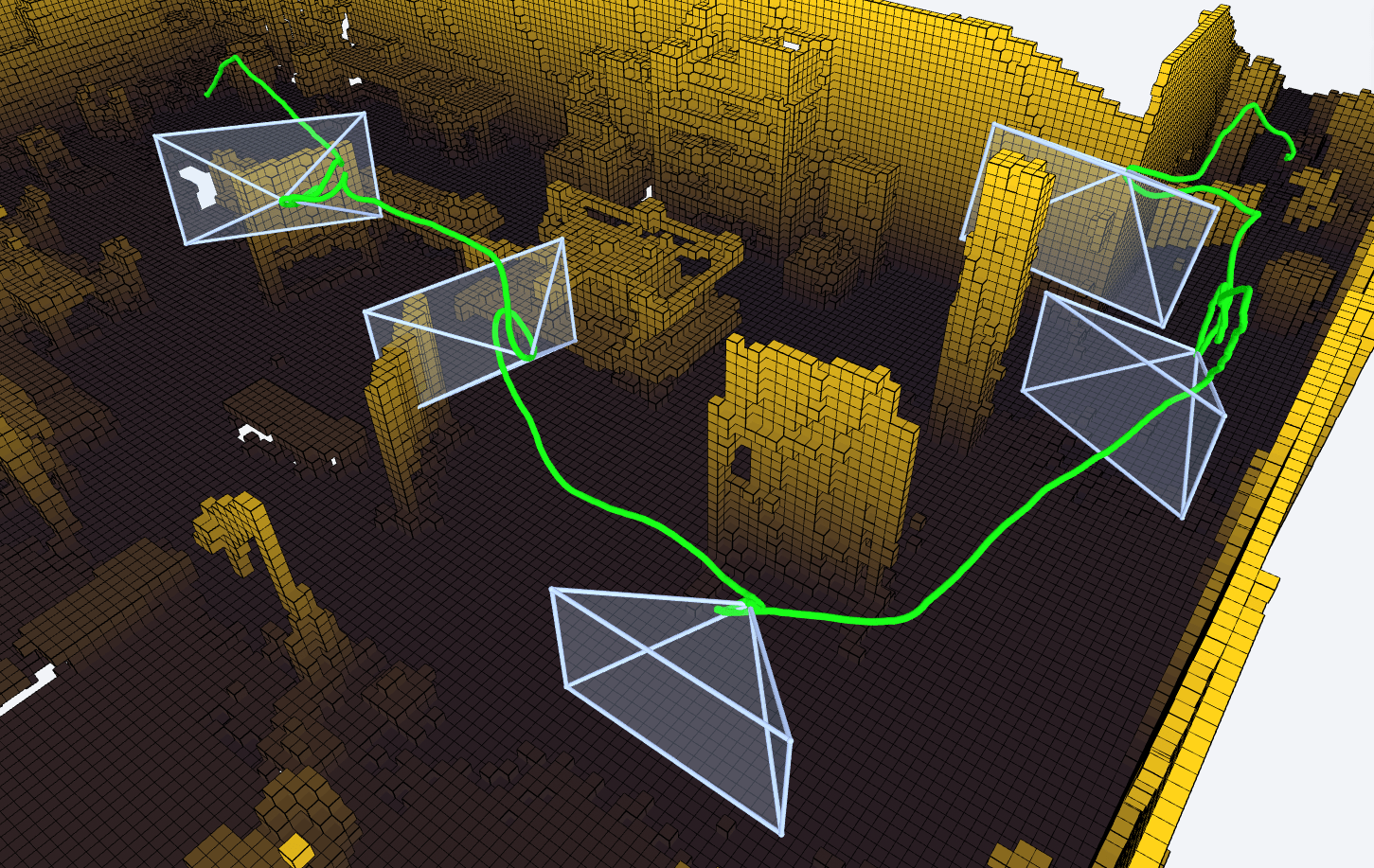}
        \caption{}
        \label{fig:ric_2d_rviz}
    \end{subfigure}

    \vspace{0.5em}

    \begin{subfigure}[t]{0.2249\columnwidth}
        \centering
        \includegraphics[
            trim={0cm 0cm 0cm 0cm},
            clip,
            width=\linewidth
        ]{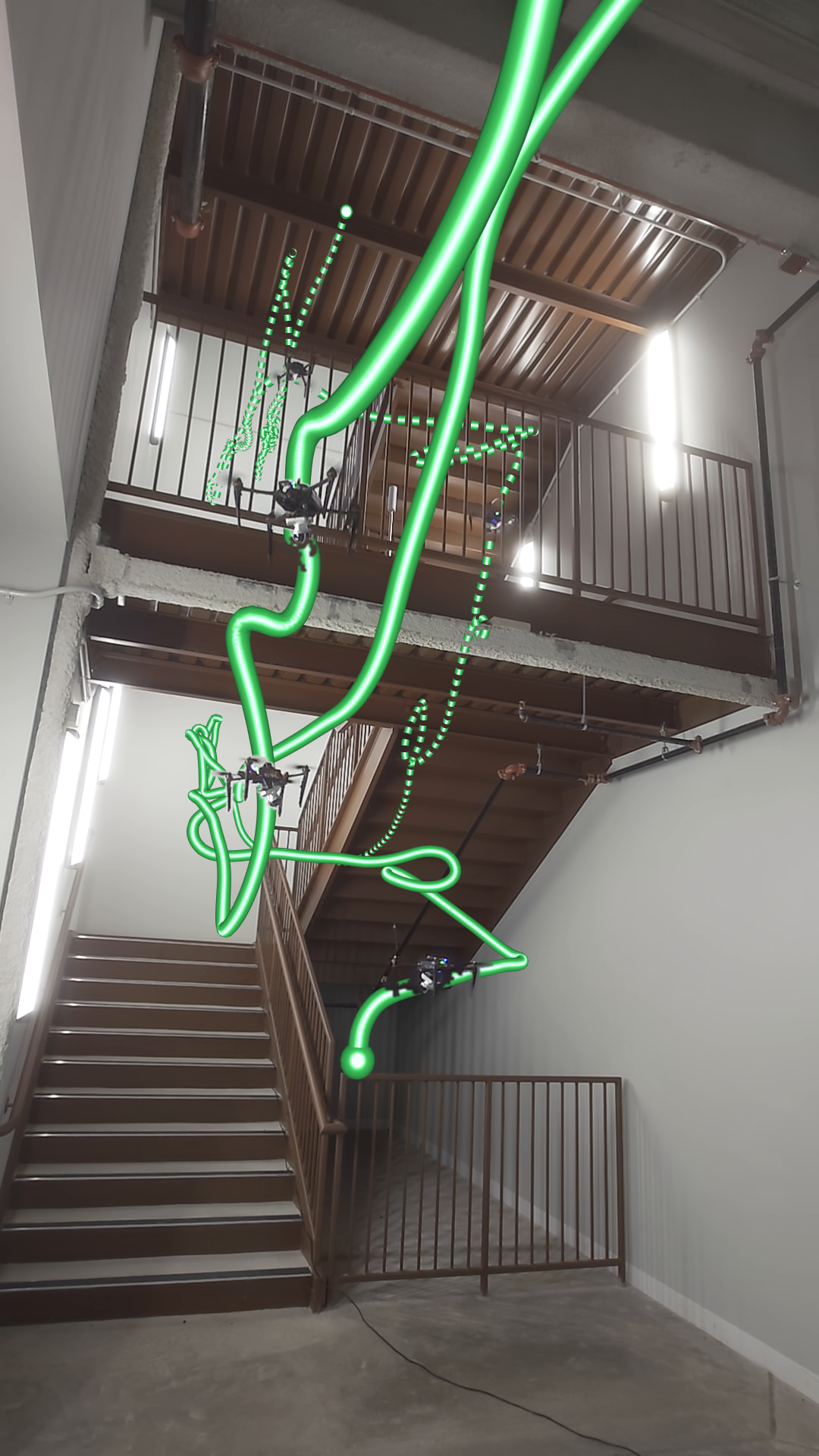}
        \caption{}
        \label{fig:stairwell_down}
    \end{subfigure}
    \hfill
    \begin{subfigure}[t]{0.3651\columnwidth}
        \centering
        \includegraphics[
            trim={0cm 0cm 0cm 0cm},
            clip,
            width=\linewidth
        ]{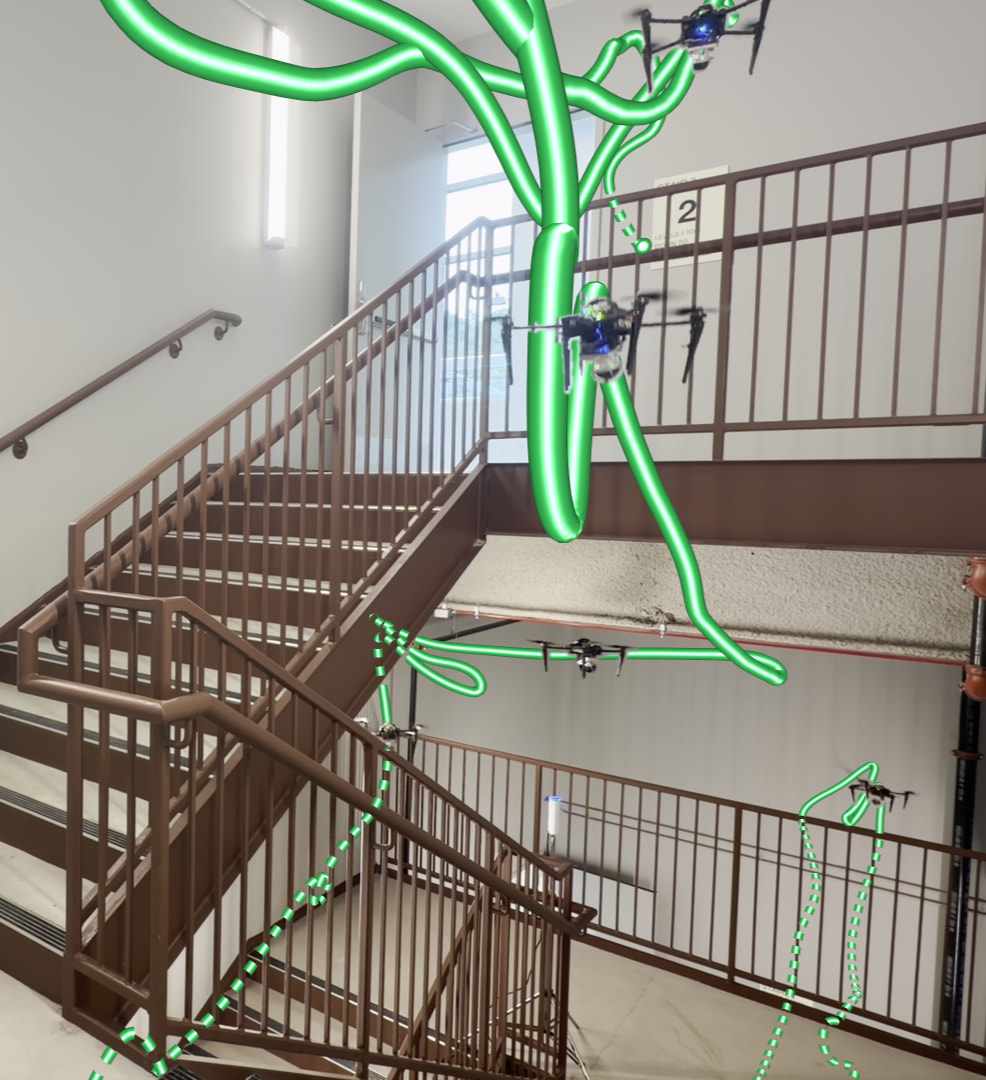}
        \caption{}
        \label{fig:stairwell_up}
    \end{subfigure}
    \hfill
    \begin{subfigure}[t]{0.3700\columnwidth}
        \centering
        \includegraphics[
            trim={0cm 0cm 0cm 0cm},
            clip,
            width=\linewidth
        ]{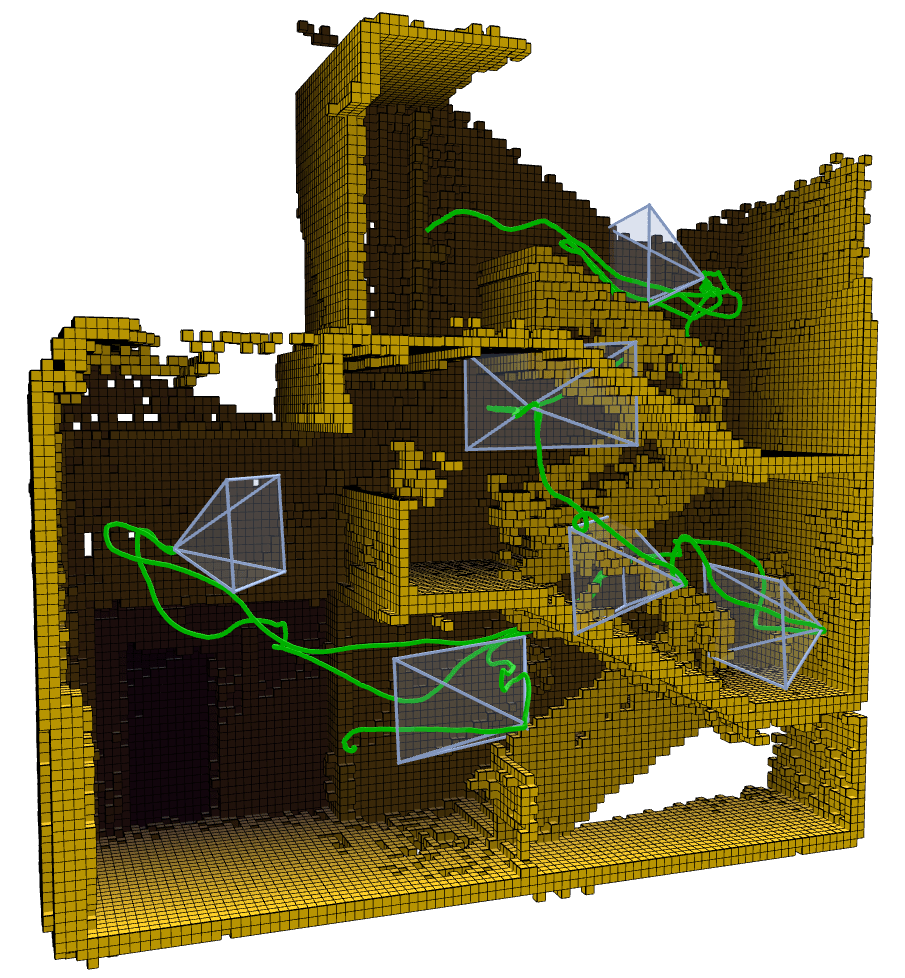}
        \caption{}
        \label{fig:stairwell_rviz}
    \end{subfigure}

    \caption{
            Additional real-robot demonstrations of SCOPE.
            (\subref{fig:ric_2d_photo}--\subref{fig:ric_2d_rviz})
            demonstrates navigation over obstacles in a cluttered indoor environment,
            while (\subref{fig:stairwell_down}--\subref{fig:stairwell_rviz})
            demonstrates multi-level ascent through a narrow stairwell using the limited onboard planning FOV.
            }
    \label{fig:ric_experiments}
    \vspace{-1.0em}
\end{figure}

\section{Conclusion}
\label{sec:conclusion}

This paper presented \PlannerName, a field-of-view-aware framework for
safe goal-directed quadrotor navigation in unknown 3D
environments. By formulating navigation as online
safety-volume certification, \PlannerNameSpaced separates optimistic progress
from certified execution and actively observes the specific
robot-inflated volume that blocks further motion. Target-centric
viewpoint search, recursive subgoals, unknown-voxel observation,
and certified preview together enable safe progress through
vertical and partially observed passages.

In simulation, both \PlannerNameSpaced variants reached all $100$ goals while
maintaining near-zero entry into non-certified inflated space,
and preview reduced mean mission time by 27\%. Four real-robot
demonstrations further validated the complete mapping,
planning, trajectory-generation, and control pipeline. The
formal certification guarantee applies directly to the search
output; future work will further tighten its preservation through
continuous trajectory realization and extend the framework to
more general sensing and dynamic environments.

\appendices
\section{Proof of Completeness}
\label{app:completeness}

This appendix proves Theorem~\ref{thm:completeness}. We first show that the
exclusion memory stores only sound failure certificates and that these
certificates remain valid as the blocked set grows. We then prove finite
recursive progress for a fixed map, and finally combine this result with
monotone sensing progress and Assumption~(A1).

\subsubsection{Fixed-Map Analysis Setting}
\label{sec:completeness-model}

We use the fixed-map setting, Assumptions~(A1)--(A2), and
Property~(P1) from Sec.~\ref{sec:theory-completeness}. During an episode
$e$, the planner searches a finite induced domain $D_e\subseteq V$; let
$H_e$ be the finite set of hitpoints that can arise in this domain. The
context-local sets $Q_{\mathrm{rej}}(S)$ and $H_{\mathrm{exh}}(S)$ start
empty, grow monotonically while the map is fixed, and are invalidated by a
planner-relevant map update. Thus, certificates are reused only under the
belief for which they were derived.

\subsubsection{Soundness of the Exclusion Memory}
\label{sec:nogood-soundness}

Fix an episode. For an observation target $\tau$, let
$Q_{\mathrm{cert}}(\tau)$ and $Q_{\mathrm{cont}}(\tau)$ denote the full
candidate families generated by $\Cand(\mathbf h,\tau;\emptyset)$. These
families are fixed during the episode and independent of the context blocked
set, because candidate expansion uses $G^{\mathrm{opt}}$ with only $O^+$
blocked.

For $B\subseteq V$, let $\mathcal P_B(\mathbf v)$ be the optimistic guidance
paths from the current execution start to $\mathbf v$ that avoid $B$, and
let $\mathcal P_B^{\mathrm c}$ and $\mathcal P_B^{\mathrm u}$ denote its
certified and non-certified subsets. Define the planning-time failure
certificates
\begin{align*}
\Fail_B(\mathbf v)
&:\Leftrightarrow
\mathcal P_B^{\mathrm c}(\mathbf v)=\emptyset\\[-0.2em]
&\qquad\wedge
\forall\Pi\in\mathcal P_B^{\mathrm u}(\mathbf v):
\Exh_B(\Hit(\Pi)),\\
\Exh_B(\mathbf h)
&:\Leftrightarrow
\Refut_B(\mathcal I_{\mathcal R}(\{\mathbf h\});\mathbf h)\\[-0.2em]
&\qquad\wedge
\exists\mathbf u\in\Omega(\mathbf h):
\Refut_B(\{\mathbf u\};\mathbf h),\\
\Refut_B(\tau;\mathbf h)
&:\Leftrightarrow
Q_{\mathrm{cert}}(\tau)=\emptyset\\[-0.2em]
&\qquad\wedge
\forall q\in Q_{\mathrm{cont}}(\tau):
\Fail_{B\cup\{\mathbf h\}}(\mathbf v_q).
\end{align*}
Here $\Fail_B$ means that a subgoal cannot be made certified-reachable,
$\Refut_B$ means that no candidate can realize an observation target, and
$\Exh_B$ means that a hitpoint cannot be cleared. Blocking $\mathbf h$ in
$\Refut_B$ prevents circularly reaching a viewpoint through the hitpoint it
is meant to clear. One refuted voxel in $\Omega(\mathbf h)$ is sufficient
for exhaustion because clearing $\mathbf h$ requires resolving every voxel
in that set. These are search certificates, not failed physical trials.

The predicates $\Exh_B(\mathbf h)$ and
$\Refut_B(\cdot;\mathbf h)$ are invoked only for $\mathbf h\notin B$.
Their recursion is well founded because each descent replaces $B$ by
$B\cup\{\mathbf h\}$ within the finite set $H_e$. By Property~(P1) and the
control flow of Algorithm~\ref{alg:runtime_pipeline}, entries are added to
$Q_{\mathrm{rej}}$ and $H_{\mathrm{exh}}$ only after the corresponding
$\Fail$ and $\Exh$ certificates have been established.

\begin{lemma}[Monotonicity of failure certificates]
\label{lem:nogood-monotone}
Within a fixed-map episode, if $B\subseteq B'$, then
$\Fail_B(\mathbf v)\Rightarrow\Fail_{B'}(\mathbf v)$. For
$\mathbf h\notin B'$,
$\Exh_B(\mathbf h)\Rightarrow\Exh_{B'}(\mathbf h)$ and
$\Refut_B(\tau;\mathbf h)\Rightarrow\Refut_{B'}(\tau;\mathbf h)$.
\end{lemma}

\begin{proof}
Use induction on $|H_e\setminus B|$, in the order $\Refut$, $\Exh$,
$\Fail$. For $\Refut$, the certified-candidate clause is independent of
$B$, and the induction hypothesis applies to every
$\Fail_{B\cup\{\mathbf h\}}$. The $\Exh$ claim follows immediately from
the two $\Refut$ clauses. Finally, every path avoiding $B'$ also avoids
$B$; $\Fail_B$ therefore supplies an exhausted first hitpoint, which remains
exhausted under $B'$ by the preceding case. The hitpoint is outside $B'$
because the path avoids $B'$.
\end{proof}

\begin{corollary}[Sound reuse and transfer of the exclusion memory]
\label{cor:nogood-reuse}
Descendants may inherit $Q_{\mathrm{rej}}(S)$ and
$H_{\mathrm{exh}}(S)$. If
$S_+=\Child(S,q,\mathbf h,\tau)$ returns \texttt{FAIL}, its complete
exclusion memory may also be adopted by any context whose blocked set
contains $H_{\mathrm{anc}}(S)\cup\{\mathbf h\}$, including sibling
children and the parent after $\mathbf h\in H_{\mathrm{exh}}(S)$.
\end{corollary}

\begin{proof}
Associate each certificate $e$ with the blocked set $\Dep(e)$ under which it
was derived. Lemma~\ref{lem:nogood-monotone} permits reuse whenever the
receiving blocked set contains $\Dep(e)$, which immediately proves descendant
inheritance.

For child-to-context transfer, list the child's exhausted hitpoints
$\mathbf h_1,\ldots,\mathbf h_m$ in insertion order. Their dependencies
satisfy
\[
\Dep(\mathbf h_k)\subseteq
H_{\mathrm{anc}}(S_+)\cup\{\mathbf h_1,\ldots,\mathbf h_{k-1}\}.
\]
A receiver containing $H_{\mathrm{anc}}(S_+)=
H_{\mathrm{anc}}(S)\cup\{\mathbf h\}$ may therefore adopt these
certificates successively. Each viewpoint rejection additionally depends on
its active hitpoint; because the child returned \texttt{FAIL}, every such
hitpoint is among the adopted exhausted hitpoints. The receiver then contains
the dependency set of every viewpoint rejection and may adopt
$Q_{\mathrm{rej}}(S_+)$ as well.
\end{proof}

\subsubsection{Finite Fixed-Map Search Progress}

\begin{lemma}[Finite fixed-map search progress]
\label{lem:fixed-map-exhaustion}
Every invocation of Algorithm~\ref{alg:runtime_pipeline} in a fixed-map
episode resolves after finite planning work by returning a sensing action or
certified terminal path, or by propagating \texttt{FAIL}. While an
observation obligation remains active, every generated non-rejected
contaminated candidate is processed unless an earlier candidate succeeds.
\end{lemma}

\begin{proof}
Use induction on
$r(S)=|H_e\setminus H_{\mathrm{anc}}(S)|$. If $r(S)=0$, all possible
hitpoints are blocked, so Property~(P1) makes $\Guide$ return either
\texttt{FAIL} or a certified terminal path.

For $r(S)>0$, Property~(P1) first resolves $\Guide$. A missing path gives
\texttt{FAIL}; a path without a hitpoint is terminal. Otherwise, for the
active hitpoint $\mathbf h$, $\Cand$ produces a finite candidate family. A
certified candidate is connected by $\Conn$ and returned. If none exists,
Algorithm~\ref{alg:runtime_pipeline} processes the non-rejected contaminated
candidates sequentially. Each candidate creates a child with
$H_{\mathrm{anc}}(S_+)=H_{\mathrm{anc}}(S)\cup\{\mathbf h\}$ and hence
$r(S_+)<r(S)$. By induction, the child resolves in finite work; on failure
the candidate enters $Q_{\mathrm{rej}}(S)$ and the loop advances. Thus the
finite candidate family is exhausted unless a child succeeds.

After volume-clearing candidates are exhausted, the planner applies the same
argument once to an unknown-observation target. If that target is also
refuted, $\mathbf h$ is added to $H_{\mathrm{exh}}(S)$ and guidance
restarts. Since $H_{\mathrm{exh}}(S)\subseteq H_e$ grows strictly, only
finitely many such restarts are possible. The context must therefore resolve
as claimed.
\end{proof}

\subsubsection{Sensing and Episode Progress}

\begin{lemma}[Monotone observation progress]
\label{lem:map-progress}
Every planner-relevant map update strictly decreases
$u_t=|U_t|$; hence at most $|V|$ such updates can occur.
\end{lemma}

\begin{proof}
Assumption~(A2) moves at least one voxel from $U_t$ to $F$ or $O$ and never
returns a voxel to $U$. Thus $U_{t+1}\subsetneq U_t$ and
$u_{t+1}<u_t$.
\end{proof}

\begin{lemma}[Executed observation progress]
\label{lem:obs-progress}
An executed unknown-observation action at a pose
$q\models\{\mathbf u\}$ resolves at least one previously unknown voxel.
\end{lemma}

\begin{proof}
The ray to $\mathbf u\in U$ lies in the sensor frustum and contains no known
occupied voxel. If it reaches $\mathbf u$, that voxel is resolved. Otherwise,
the first ground-truth obstacle on the ray cannot be known free or known
occupied and is therefore an unknown voxel newly classified as occupied.
Either outcome changes the map belief.
\end{proof}

\begin{lemma}[Finite executions per fixed-map episode]
\label{lem:finite-exec}
Each fixed-map episode contains finitely many sensing executions.
\end{lemma}

\begin{proof}
An unknown-observation execution changes the map by
Lemma~\ref{lem:obs-progress} and ends the episode. A volume-clearing action
either changes the map or leaves the hitpoint uncleared; in the latter case,
the execution-time rule of Sec.~\ref{sec:planning-loop} switches that
context--hitpoint pair to unknown observation and does not retry volume
clearing on the unchanged map. Hence each pair causes at most one zero-update
execution. The number of contexts and hitpoints in the episode is finite by
$D_e$ and Lemma~\ref{lem:fixed-map-exhaustion}.
\end{proof}

\subsubsection{Proof of Theorem~\ref{thm:completeness}}

\begin{proof}
Lemma~\ref{lem:fixed-map-exhaustion} gives finite planning resolution within
each fixed-map episode, and Lemma~\ref{lem:finite-exec} gives finitely many
executions within that episode. Every map-changing execution strictly reduces
$|U|$ by Lemma~\ref{lem:map-progress}, so at most $|V|$ map-changing
episodes can occur.

Fix the finite feasible sensing sequence from Assumption~(A1), and consider
its next unresolved action. Property~(P1) generates its candidate pose and,
when certified-reachable, its connection. If the pose is initially
contaminated, the finite sequential candidate loop and
Lemma~\ref{lem:fixed-map-exhaustion} ensure that it is eventually processed
unless an earlier candidate already resolves the same obligation.
Lemma~\ref{lem:nogood-monotone} and
Corollary~\ref{cor:nogood-reuse} ensure that exclusion memory removes only
certified failures, so it cannot prune this action while the action remains
feasible.

The returned action either changes the map immediately or is a zero-update
volume-clearing action. In the latter case, the fallback selects an
unknown-observation action, which changes the map by
Lemmas~\ref{lem:obs-progress}--\ref{lem:finite-exec}. Thus each unresolved
obligation in the finite sequence is resolved after finite planning and
execution progress. After the sequence is completed, Assumption~(A1) places
the remaining route to the goal in $G^{\mathrm{cert}}$, so the planner
returns a certified terminal path. Therefore the goal is reached after
finitely many fixed-map episodes.
\end{proof}

\bibliographystyle{IEEEtran}
\bibliography{refs}

@article{lumelsky1986dynamic,
title = {Dynamic Path Planning for a Mobile Automaton with Limited Information on the Environment},
author = {Lumelsky, Vladimir J. and Stepanov, Alexander A.},
journal = {IEEE Transactions on Automatic Control},
volume = {31},
number = {11},
pages = {1058--1063},
year = {1986},
doi = {10.1109/TAC.1986.1104175},
publisher = {IEEE}
}

@article{kamon1997distbug,
title = {Sensory-Based Motion Planning with Global Proofs},
author = {Kamon, Ishay and Rivlin, Ehud},
journal = {IEEE Transactions on Robotics and Automation},
volume = {13},
number = {6},
pages = {814--822},
year = {1997},
publisher = {IEEE}
}

@article{kamon1998tangentbug,
title = {{TangentBug}: A Range-Sensor-Based Navigation Algorithm},
author = {Kamon, Ishay and Rimon, Elon and Rivlin, Ehud},
journal = {The International Journal of Robotics Research},
volume = {17},
number = {9},
pages = {934--953},
year = {1998},
doi = {10.1177/027836499801700903},
publisher = {Sage Publications}
}

@inproceedings{stentz1995dstar,
title = {The Focussed {D*} Algorithm for Real-Time Replanning},
author = {Stentz, Anthony},
booktitle = {Proceedings of the Fourteenth International Joint Conference on Artificial Intelligence},
pages = {1652--1659},
year = {1995}
}

@inproceedings{koenig2002dstar,
title = {{D*} Lite},
author = {Koenig, Sven and Likhachev, Maxim},
booktitle = {Proceedings of the Eighteenth National Conference on Artificial Intelligence},
pages = {476--483},
year = {2002},
publisher = {AAAI Press}
}

@article{zhou2019fastplanner,
title = {Robust and Efficient Quadrotor Trajectory Generation for Fast Autonomous Flight},
author = {Zhou, Boyu and Gao, Fei and Wang, Luqi and Liu, Chuhao and Shen, Shaojie},
journal = {IEEE Robotics and Automation Letters},
volume = {4},
number = {4},
pages = {3529--3536},
year = {2019},
publisher = {IEEE}
}

@inproceedings{zhou2020topotraj,
title = {Robust Real-Time {UAV} Replanning Using Guided Gradient-Based Optimization and Topological Paths},
author = {Zhou, Boyu and Gao, Fei and Pan, Jie and Shen, Shaojie},
booktitle = {2020 IEEE International Conference on Robotics and Automation},
pages = {1208--1214},
year = {2020},
organization = {IEEE},
doi = {10.1109/ICRA40945.2020.9196996}
}

@article{WANG2022GCOPTER,
    title={Geometrically Constrained Trajectory Optimization for Multicopters}, 
    author={Wang, Zhepei and Zhou, Xin and Xu, Chao and Gao, Fei}, 
    journal={IEEE Transactions on Robotics}, 
    year={2022}, 
    volume={38}, 
    number={5}, 
    pages={3259-3278}, 
    doi={10.1109/TRO.2022.3160022}
}

@article{zhou2021ego,
title = {{EGO}-Planner: An {ESDF}-Free Gradient-Based Local Planner for Quadrotors},
author = {Zhou, Xin and Wang, Zhepei and Ye, Hongkai and Xu, Chao and Gao, Fei},
journal = {IEEE Robotics and Automation Letters},
volume = {6},
number = {2},
pages = {478--485},
year = {2021},
publisher = {IEEE}
}

@inproceedings{ren2022bubble,
title = {Bubble Planner: Planning High-Speed Smooth Quadrotor Trajectories Using Receding Corridors},
author = {Ren, Yunfan and Zhu, Fangcheng and Liu, Wenyi and Wang, Zhepei and Lin, Yi and Gao, Fei and Zhang, Fu},
booktitle = {2022 IEEE/RSJ International Conference on Intelligent Robots and Systems},
pages = {6332--6339},
year = {2022},
organization = {IEEE},
doi = {10.1109/IROS47612.2022.9981518}
}

@article{tordesillas2021faster,
title = {{FASTER}: Fast and Safe Trajectory Planner for Navigation in Unknown Environments},
author = {Tordesillas, Jesus and Lopez, Brett T. and Everett, Michael and How, Jonathan P.},
journal = {IEEE Transactions on Robotics},
volume = {38},
number = {2},
pages = {922--938},
year = {2022},
doi = {10.1109/TRO.2021.3100142},
publisher = {IEEE}
}

@inproceedings{yang2021far,
title = {{FAR} Planner: Fast, Attemptable Route Planner Using Dynamic Visibility Update},
author = {Yang, Fan and Cao, Chao and Zhu, Hongbiao and Oh, Jean and Zhang, Ji},
booktitle = {2022 IEEE/RSJ International Conference on Intelligent Robots and Systems},
pages = {9--16},
year = {2022},
organization = {IEEE},
doi = {10.1109/IROS47612.2022.9981574}
}

@article{ren2025super,
title = {Safety-Assured High-Speed Navigation for {MAV}s},
author = {Ren, Yunfan and Zhu, Fangcheng and Lu, Guozheng and Cai, Yixi and Yin, Longji and Kong, Fanze and Lin, Jiarong and Chen, Nan and Zhang, Fu},
journal = {Science Robotics},
volume = {10},
number = {98},
pages = {eado6187},
year = {2025},
doi = {10.1126/scirobotics.ado6187},
publisher = {American Association for the Advancement of Science}
}

@inproceedings{yamauchi1997frontier,
title = {A Frontier-Based Approach for Autonomous Exploration},
author = {Yamauchi, Brian},
booktitle = {Proceedings of the 1997 IEEE International Symposium on Computational Intelligence in Robotics and Automation},
pages = {146--151},
year = {1997},
organization = {IEEE}
}

@inproceedings{connolly1985nbv,
title = {The Determination of Next Best Views},
author = {Connolly, C. Ian},
booktitle = {Proceedings of the 1985 IEEE International Conference on Robotics and Automation},
pages = {432--435},
year = {1985},
organization = {IEEE}
}

@inproceedings{bircher2016nbvp,
title = {Receding Horizon ``Next-Best-View'' Planner for 3D Exploration},
author = {Bircher, Andreas and Kamel, Mina and Alexis, Kostas and Oleynikova, Helen and Siegwart, Roland},
booktitle = {2016 IEEE International Conference on Robotics and Automation},
pages = {1462--1468},
year = {2016},
organization = {IEEE},
doi = {10.1109/ICRA.2016.7487281}
}

@article{zhou2021fuel,
title = {{FUEL}: Fast {UAV} Exploration Using Incremental Frontier Structure and Hierarchical Planning},
author = {Zhou, Boyu and Zhang, Yichen and Chen, Xinyi and Shen, Shaojie},
journal = {IEEE Robotics and Automation Letters},
volume = {6},
number = {2},
pages = {779--786},
year = {2021},
publisher = {IEEE}
}

@inproceedings{cao2021tare,
title = {{TARE}: A Hierarchical Framework for Efficiently Exploring Complex 3D Environments},
author = {Cao, Chao and Zhu, Hongbiao and Choset, Howie and Zhang, Ji},
booktitle = {Robotics: Science and Systems},
year = {2021}
}

@article{dang2019gbplanner,
title = {Graph-Based Subterranean Exploration Path Planning Using Aerial and Legged Robots},
author = {Dang, Tung and Tranzatto, Marco and Khattak, Shehryar and Mascarich, Frank and Alexis, Kostas and Hutter, Marco},
journal = {Journal of Field Robotics},
volume = {37},
number = {8},
pages = {1363--1388},
year = {2020},
publisher = {Wiley}
}

@article{chin1988watchman,
title = {Optimum Watchman Routes},
author = {Chin, Wei-Pang and Ntafos, Simeon C.},
journal = {Information Processing Letters},
volume = {28},
number = {1},
pages = {39--44},
year = {1988},
publisher = {Elsevier}
}

@inproceedings{guibas1997pursuit,
title = {Visibility-Based Pursuit-Evasion in a Polygonal Environment},
author = {Guibas, Leonidas J. and Latombe, Jean-Claude and LaValle, Steven M. and Lin, David and Motwani, Rajeev},
booktitle = {Algorithms and Data Structures},
series = {Lecture Notes in Computer Science},
volume = {1272},
pages = {17--30},
year = {1997},
publisher = {Springer}
}

@inproceedings{lavalle1997visibility,
title = {Motion Strategies for Maintaining Visibility of a Moving Target},
author = {LaValle, Steven M. and Gonzalez-Banos, Hector H. and Becker, Craig and Latombe, Jean-Claude},
booktitle = {Proceedings of the 1997 IEEE International Conference on Robotics and Automation},
volume = {1},
pages = {731--736},
year = {1997},
organization = {IEEE}
}

@incollection{latombe1997visibility,
title = {Motion Planning with Visibility Constraints: Building Autonomous Observers},
author = {Gonzalez-Banos, Hector H. and Guibas, Leonidas J. and Latombe, Jean-Claude and LaValle, Steven M. and Lin, David and Motwani, Rajeev and Tomasi, Carlo},
booktitle = {Robotics Research},
pages = {95--101},
year = {1998},
publisher = {Springer},
note = {Presented at the Eighth International Symposium on Robotics Research}
}

@incollection{goretkin2018look,
title = {Look Before You Sweep: Visibility-Aware Motion Planning},
author = {Goretkin, Gustavo and Kaelbling, Leslie Pack and Lozano-P{'e}rez, Tom{'a}s},
booktitle = {Algorithmic Foundations of Robotics XIII},
series = {Springer Proceedings in Advanced Robotics},
volume = {14},
pages = {373--388},
year = {2020},
publisher = {Springer},
doi = {10.1007/978-3-030-44051-0_22},
note = {Proceedings of WAFR 2018}
}

@phdthesis{goretkin2022vamp,
title = {Visibility-Aware Motion Planning},
author = {Goretkin, Gustavo Nunes},
school = {Massachusetts Institute of Technology},
year = {2022},
url = {https://hdl.handle.net/1721.1/143248}
}

@inproceedings{falanga2018pampc,
title = {{PAMPC}: Perception-Aware Model Predictive Control for Quadrotors},
author = {Falanga, Davide and Foehn, Philipp and Lu, Peng and Scaramuzza, Davide},
booktitle = {2018 IEEE/RSJ International Conference on Intelligent Robots and Systems},
year = {2018},
organization = {IEEE}
}

@inproceedings{bartolomei2020semantic,
title = {Perception-Aware Path Planning for {UAV}s Using Semantic Segmentation},
author = {Bartolomei, Luca and Teixeira, Lucas and Chli, Margarita},
booktitle = {2020 IEEE/RSJ International Conference on Intelligent Robots and Systems},
year = {2020},
organization = {IEEE}
}

@inproceedings{wang2021visibility,
title = {Visibility-Aware Trajectory Optimization with Application to Aerial Tracking},
author = {Wang, Qianhao and Gao, Yuman and Ji, Jialin and Xu, Chao and Gao, Fei},
booktitle = {2021 IEEE/RSJ International Conference on Intelligent Robots and Systems},
pages = {5249--5256},
year = {2021},
organization = {IEEE}
}

@article{tordesillas2022panther,
title = {{PANTHER}: Perception-Aware Trajectory Planner in Dynamic Environments},
author = {Tordesillas, Jesus and How, Jonathan P.},
journal = {IEEE Access},
volume = {10},
pages = {22662--22677},
year = {2022},
publisher = {IEEE}
}

@article{wang2022gpa,
title = {{GPA}-Teleoperation: Gaze Enhanced Perception-Aware Safe Assistive Aerial Teleoperation},
author = {Wang, Qianhao and He, Botao and Xun, Zhiren and Xu, Chao and Gao, Fei},
journal = {IEEE Robotics and Automation Letters},
volume = {7},
number = {2},
pages = {5631--5638},
year = {2022},
publisher = {IEEE}
}

@inproceedings{liu2022starconvex,
title = {Star-Convex Constrained Optimization for Visibility Planning with Application to Aerial Inspection},
author = {Liu, Tianyu and Wang, Qianhao and Zhong, Xingguang and Wang, Zhepei and Xu, Chao and Zhang, Fu and Gao, Fei},
booktitle = {2022 International Conference on Robotics and Automation},
pages = {7861--7867},
year = {2022},
organization = {IEEE}
}

@inproceedings{wu2024globalyaw,
title = {Trajectory Optimization with Global Yaw Parameterization for Field-of-View Constrained Autonomous Flight},
author = {Wu, Yuwei and Tao, Yuezhan and Spasojevic, Igor and Kumar, Vijay},
booktitle = {2024 IEEE/RSJ International Conference on Intelligent Robots and Systems},
pages = {10590--10596},
year = {2024},
organization = {IEEE}
}

@article{zhang2025spot,
title = {{SPOT}: Sensing-Augmented Trajectory Planning via Obstacle Threat Modeling},
author = {Zhang, Chi and Huang, Xian and Dong, Wei},
journal = {arXiv preprint arXiv:2510.16308},
year = {2025},
doi = {10.48550/arXiv.2510.16308}
}

@inproceedings{liu2016limited,
title = {High Speed Navigation for Quadrotors with Limited Onboard Sensing},
author = {Liu, Sikang and Watterson, Michael and Tang, Sarah and Kumar, Vijay},
booktitle = {2016 IEEE International Conference on Robotics and Automation},
pages = {1484--1491},
year = {2016},
organization = {IEEE}
}

@inproceedings{lopez2017aggressive3d,
title = {Aggressive 3-{D} Collision Avoidance for High-Speed Navigation},
author = {Lopez, Brett T. and How, Jonathan P.},
booktitle = {2017 IEEE International Conference on Robotics and Automation},
pages = {5759--5765},
year = {2017},
organization = {IEEE}
}

@inproceedings{lopez2017limitedfov,
title = {Aggressive Collision Avoidance with Limited Field-of-View Sensing},
author = {Lopez, Brett T. and How, Jonathan P.},
booktitle = {2017 IEEE/RSJ International Conference on Intelligent Robots and Systems},
pages = {1358--1365},
year = {2017},
organization = {IEEE}
}

@inproceedings{florence2018nanomap,
title = {{NanoMap}: Fast, Uncertainty-Aware Proximity Queries with Lazy Search over Local 3D Data},
author = {Florence, Peter R. and Carter, John and Ware, Jake and Tedrake, Russ},
booktitle = {2018 IEEE International Conference on Robotics and Automation},
pages = {7631--7638},
year = {2018},
organization = {IEEE}
}

@inproceedings{nieuwenhuisen2019sensorvisibility,
title = {Search-Based 3D Planning and Trajectory Optimization for Safe Micro Aerial Vehicle Flight Under Sensor Visibility Constraints},
author = {Nieuwenhuisen, Matthias and Behnke, Sven},
booktitle = {2019 International Conference on Robotics and Automation},
pages = {9123--9129},
year = {2019},
organization = {IEEE}
}

@article{yu2022cpa,
title = {{CPA}-Planner: Motion Planner with Complete Perception Awareness for Sensing-Limited Quadrotors},
author = {Yu, Qiuyu and Qin, Chao and Luo, Lingkun and Liu, Hugh H.-T. and Hu, Shiqiang},
journal = {IEEE Robotics and Automation Letters},
volume = {8},
number = {2},
pages = {720--727},
year = {2023},
doi = {10.1109/LRA.2022.3231827},
publisher = {IEEE}
}

@inproceedings{wang2024multifov,
title = {Multi-{FOV}-Constrained Trajectory Planning for Multirotor Safe Landing},
author = {Wang, Dong and Wang, Jingping and He, Suqin and Huang, Jinxin and Zhang, Bangyan and Mao, Yinian and Huang, Guoquan and Xu, Chao and Gao, Fei},
booktitle = {2024 IEEE/RSJ International Conference on Intelligent Robots and Systems},
pages = {5356--5363},
year = {2024},
organization = {IEEE},
doi = {10.1109/IROS58592.2024.10802806}
}

@inproceedings{baek2025pipe,
  author={Baek*, Seungjae and Moon, Brady and Kim, Seungchan and Cao, Muqing and Ho, Cherie and Scherer, Sebastian and Jeon, Jeong Hwan},
  booktitle={2025 IEEE/RSJ International Conference on Intelligent Robots and Systems (IROS)}, 
  title={PIPE Planner: Pathwise Information Gain with Map Predictions for Indoor Robot Exploration}, 
  year={2025},
  pages={7684-7691},
  doi={10.1109/IROS60139.2025.11246190}
}

@INPROCEEDINGS{harutyunyan2025mapexrl,
          title={MapExRL: Human-Inspired Indoor Exploration with Predicted Environment Context and Reinforcement Learning},
          author={Harutyunyan, Narek and Moon, Brady and Kim, Seungchan and Ho, Cherie and Hung, Adam and Scherer, Sebastian},
          booktitle = {2025 IEEE International Conference on Advanced Robotics (ICAR)},
          year={2025},
          doi={10.1109/ICAR65334.2025.11338661}
        }

@ARTICLE{FIRI,
  author={Wang, Qianhao and Wang, Zhepei and Wang, Mingyang and Ji, Jialin and Han, Zhichao and Wu, Tianyue and Jin, Rui and Gao, Yuman and Xu, Chao and Gao, Fei},
  journal={IEEE Transactions on Robotics}, 
  title={Fast Iterative Region Inflation for Computing Large 2-D/3-D Convex Regions of Obstacle-Free Space}, 
  year={2025},
  volume={41},
  number={},
  pages={3223-3243},
  doi={10.1109/TRO.2025.3562482}}

@misc{zacharia2026omni,
      title={OmniPlanner: Universal Exploration and Inspection Path Planning across Robot Morphologies}, 
      author={Angelos Zacharia and Mihir Dharmadhikari and Mohit Singh and Kostas Alexis},
      year={2026},
      eprint={2603.04284},
      archivePrefix={arXiv},
      primaryClass={cs.RO},
      url={https://arxiv.org/abs/2603.04284}, 
}

\end{document}